\documentclass[runningheads]{llncs}

\usepackage{amsmath,amssymb}
\usepackage{mathtools}
\usepackage{microtype}
\usepackage{algorithm}
\usepackage{algpseudocode}
\usepackage{longtable}
\usepackage{float}
\usepackage{graphicx}
\usepackage{multirow}
\usepackage{xcolor}
\usepackage{tikz}
\usetikzlibrary{arrows.meta,positioning,calc,fit,backgrounds}
\usepackage{hyperref}
\makeatletter
\let\Moose@oldappendix\appendix
\renewcommand{\appendix}{%
  \Moose@oldappendix
  \renewcommand{\theHsection}{\Alph{section}}%
  \renewcommand{\theHsubsection}{\theHsection.\arabic{subsection}}%
  \renewcommand{\theHsubsubsection}{\theHsubsection.\arabic{subsubsection}}%
}
\makeatother
\usepackage{cleveref}
\usepackage{todonotes}

\newcommand{\EL}{\mathcal{EL}}
\newcommand{\ELpp}{$\EL^{++}$}
\newcommand{\ALC}{$\mathcal{ALC}$}
\newcommand{\ALCI}{$\mathcal{ALCI}$}
\newcommand{\BOT}{\bot}
\newcommand{\TOP}{\top}
\newcommand{\sub}{\sqsubseteq}
\newcommand{\sym}[1]{\mathsf{#1}}
\newcommand{\Sub}{\mathsf{Sub}}
\newcommand{\Onto}{\mathcal{O}}
\newcommand{\sat}{\mathsf{Sat}}
\newcommand{\WMC}{\mathsf{WMC}}
\newcommand{\Sig}{\mathsf{Sig}}
\newcommand{\selAx}{\mathsf{sel}}
\newcommand{\disponte}{\mathsf{P}}
\newcommand{\disponteRat}{\mathsf{P}^{\mathbb{Q}}}
\newcommand{\compileSat}{\mathsf{cmp}}
\newcommand{\wmcRat}{\mathsf{WMC}^{\mathbb{Q}}}
\newcommand{\fragWit}{\mathsf{OWL2ELFragment}}
\newcommand{\NOTE}[1]{}

\begin{document}

\title{Moose: Latent concept learning with reasoning-shortcut awareness in $\EL^{++}$}
\titlerunning{Moose}

\author{Olga Mashkova\inst{1}\orcidID{0000-0002-4916-1660}, Asaad Mohammedsaleh\inst{1}\orcidID{0009-0007-3160-8819}, Fernando Zhapa-Camacho\inst{1}\orcidID{0000-0002-0710-2259} and Robert Hoehndorf\inst{1}\orcidID{0000-0001-8149-5890}}
\authorrunning{Mashkova et al.}

\institute{Computer Science Program, Computer, Electrical, and
  Mathematical Sciences \& Engineering Division, King Abdullah
  University of Science and Technology, Thuwal 23955, Saudi Arabia\\
\email{\{first\_name.last\_name\}@kaust.edu.sa}}

\maketitle

% =========================================================================
\begin{abstract}
  The OWL 2 EL profile is used in some of the largest production
  ontologies, including the Gene Ontology and SNOMED~CT. Existing
  neuro-symbolic (NeSy) learning methods accept propositional theories
  or Datalog, and
  reasoning-shortcut (RS) awareness has not been investigated in
  ontology settings. We present
  \textsc{Moose}, a method that compiles an \ELpp\ TBox and finite
  ABox to a Sentential Decision Diagram (SDD). The SDD acts as a
  differentiable weighted-model-counting layer, and we add closure
  clauses outside the \ELpp\ profile on declared exhaustive families
  to overcome the limited expressivity of \ELpp\ under partial
  supervision. We show termination, soundness, completeness, and
  polynomial intermediate sizes, and validate the proofs in Lean.
  We then define the first formal partial-supervision
  latent-concept-learning task over an OWL EL ontology, i.e.,
  learning per-individual classifiers for latent concepts from
  observed ABox literals, and evaluate \textsc{Moose} on
  MNIST-with-ontology and Pizza\"iolo. \textsc{Moose} improves over
  propositional-NeSy, fuzzy-logic, and ontology embedding baselines,
  and presents the first reasoning-shortcut analysis in an OWL EL
  setting.

\end{abstract}
% =========================================================================

\section{Introduction}\label{sec:intro}

The OWL~2 EL profile underpins some of the largest and most
widely-used ontologies, such as the Gene Ontology with $\sim$50K
classes and $\sim$8M annotations \cite{Ashburner2000,GO2023},
SNOMED~CT \cite{Donnelly2006}, and most of the OBO Foundry ontologies
\cite{Smith2007OBOFoundry,Jackson2021OBOFoundry}, because EL trades
expressivity for polynomial-time reasoning on subsumption, instance
checking, and consistency \cite{Baader2005,Kazakov2014}.
Neuro-symbolic (NeSy) learning combines neural perception with
symbolic constraints; several families coexist. Fuzzy / t-norm methods
replace propositional truth with continuous operators on $[0,1]$
\cite{Badreddine2022LTN}; embedding methods score entities and
relations in a continuous geometry
\cite{Kulmanov2019,Chen2021,Chen2025Survey};
\emph{knowledge-compilation} methods compile the constraint to an
arithmetic circuit whose weighted model count (WMC) is differentiable
in the inputs \cite{Manhaeve2018,Xu2018,Pryor2023,Ahmed2022SPL}. The
knowledge-compilation family retains an exact WMC of the constraint as
its training signal under partial supervision.

The task we address, \emph{ABox-supervised latent concept learning},
takes the following form. Each training instance supplies one
perceptual input per named individual (e.g.\ an image of a digit)
together with a partial set of ABox literals over an \emph{observable}
signature; the truth values of atoms over a separate \emph{latent}
signature are never directly supervised. The objective is to learn
per-individual classifiers for the truth values of latent concept
atoms. A WMC layer over a circuit compiled from the ontology supplies
the training signal: evaluated with the perception's per-atom outputs
and clamped on the observed evidence, the layer scores how compatible
the perception's predictions are with the ontology, and the gradient
steers the perception toward configurations that the constraints
admit. When several latent configurations satisfy the same evidence,
no method can recover the true labels from supervision alone;
per-atom-independent predictors fail to recognize this and collapse to
high confidence on a single arbitrary configuration rather than
spreading mass over the satisfying ones
\cite{Manhaeve2018,Marconato2023,vanKrieken2024Independence}.

Existing knowledge-compilation NeSy frameworks accept propositional
formulas, Datalog programs, or finite-domain rules
\cite{Manhaeve2018,Xu2018,Pryor2023,Ahmed2022SPL}, but not OWL
ontologies directly. Two recent works target description logics: a
domino-style reduction of \ALCI{} to a probabilistic circuit
\cite{Lazzari2026} that treats the circuit as a regulariser under full
label supervision, and a fuzzy approximation of \ELpp\ via
Goguen-implication semantics \cite{Zhao2025DFELpp} that targets
knowledge-base completion without exact WMC or compilation soundness;
neither performs partial-supervision concept learning over OWL EL.
Embedding methods such as ELEmbeddings and OWL2Vec$^*$
\cite{Kulmanov2019,Chen2021} map ontology structure to a continuous
geometry (via geometric model-theoretic constraints or graph walks
respectively), trading exact entailment for differentiability. A
separate body of work on reasoning shortcuts (RS)
\cite{Marconato2023,Marconato2024BEARS,vanKrieken2024Independence,vanKrieken2025RS,vanKrieken2025NeSyDM,Bortolotti2024RSBench}
formally proves that conditional independence among predicted concepts
is incompatible with RS-awareness, and develops mitigation strategies
(BEARS ensembles, neuro-symbolic diffusion); these results are stated
entirely in propositional or Datalog NeSy. OWL 2 EL underpins many
production ontologies, so the RS phenomena documented in propositional
settings translate directly into uncertainty-modeling failures in
deployed ontology-driven systems.  Translations of OWL EL into
other formalisms are well established: it is Datalog-rewritable
\cite{Krotzsch2010ELDatalog,Carral2019Datalog}, and probabilistic
\cite{Ceylan2014,Ceylan2017} and fuzzy
\cite{Zhao2025DFELpp,Bobillo2011} variants have been studied at
length. However, a rewriting yields a reasoning procedure rather than a
learning signal, and the probabilistic and fuzzy variants relax or
re-specify the model-theoretic semantics rather than compiling it.

Here we present \textsc{Moose}, a method that compiles an
\ELpp\ ontology and a finite ABox domain into a Sentential Decision
Diagram \cite{Darwiche2011,Choi2013} whose models coincide with
ABox interpretations satisfying the entailments of the input
ontology, paired with a learning framework that uses the
differentiable circuit as the sole supervision channel for latent
concept atoms and provides the first RS analysis in an OWL EL setting.
Our contributions are:

\begin{enumerate}
\item \textbf{End-to-end OWL~2 EL compilation with full proofs.} We
  compile the EL profile, including role chains
  $R_1 \circ R_2 \sub S$ and role hierarchies (the constructs that
  distinguish SNOMED~CT and the Gene Ontology from a propositional
  or \ALCI{} setting), via ELK-style saturation \cite{Kazakov2014} to
  an SDD supporting WMC. 
The four formal contributions are a
  verified SDD encoding (\Cref{thm:encoding}), the rational
  DISPONTE \cite{Riguzzi2015}
  distribution-semantics correspondence (\Cref{thm:end-to-end}), and
  SCC-compositional factorizations at the Sat and WMC levels
  (\Cref{thm:scc-sat,thm:scc-wmc}).  All four are mechanized in
  Lean~4 (no \textsc{Moose}-specific axioms;
  \Cref{app:proofs,app:lean-index}), along with the supporting Lean
  libraries for ELK saturation, soundness/completeness on \ELpp\
  \cite{Kazakov2014}, and SDD knowledge compilation, which to our
  knowledge are the first such formalizations
  (\Cref{app:elk-mechanization}).  The closure-augmented variant is
  correct (\Cref{thm:closure-correct}) and inference is linear in
  the SDD (\Cref{thm:inference-complexity}).
\item \textbf{ABox-supervised latent concept learning under partial
  axiom observation.} We pose a learning task in which a subset of
  ABox literals over the observable signature is given as evidence
  (with truth values written $\mathsf{true}$ and
  $\mathsf{false}$, so that they are not confused with the
  \ELpp\ concepts $\TOP$ and $\BOT$), and the goal is to learn
  per-individual classifiers for the truth values of latent concept
  atoms (role atoms are observed or marginalized, never classified),
  to our knowledge the first such formulation over an OWL EL
  ontology, generalizing DeepProbLog \cite{Manhaeve2018} and Semantic
  Loss \cite{Xu2018} to description logic constraints with role
  chains and role hierarchies.
\item \textbf{Reasoning-shortcut analysis transposed to OWL EL.} We
  compute family-argmax accuracy, expected calibration error (ECE),
  and the RS-consistency rate under \textsc{Moose} and two plug-in
  mitigations (\textsc{Moose}+BEARS \cite{Marconato2024BEARS} and
  \textsc{Moose}+NeSyDM \cite{vanKrieken2025NeSyDM}), and isolate a
  calibration-vs-accuracy trade-off: BEARS leads on family-argmax
  accuracy via ensemble diversification on RS-suspect inputs, NeSyDM
  leads on calibration (ECE) under symbolic ambiguity.
\item \textbf{A benchmark suite.}  A single MNIST-with-ontology
  benchmark covering three supervision regimes (atomic, relational,
  and role-chain), together with a fourth experiment on the
  Pizza\"iolo synthetic-image dataset \cite{Pizzaiolo2024} that
  transfers the method to a real expert-authored OWL EL ontology.
\end{enumerate}

% =========================================================================
\section{Background and related work}\label{sec:bg}

\subsection{Preliminaries: OWL EL, knowledge compilation,
  partial-supervision NeSy}\label{sec:bg-prelim}

\paragraph{The OWL 2 EL profile.}
\ELpp\ \cite{Baader2005} is the description logic basis of OWL~2~EL
(syntax and semantics in \Cref{tab:elpp}).
Concepts are built from atomic names, $\TOP$, $\BOT$, conjunction
$C \sqcap D$, and existential restriction $\exists R.C$; an ontology
$\Onto$ contains GCIs $C \sub D$, role inclusions $R \sub S$, and role
chains $R_1 \circ R_2 \sub S$, with subsumption, instance checking,
and consistency decidable in time polynomial in $|\Onto|$. ELK
\cite{Kazakov2014} realizes this bound through a consequence-based
saturation closure under ten completion rules
\cite{TenaCucala2019Consequence}. \ELpp\ is moreover
Datalog-rewritable: consequence-based reasoning can be recast as the
evaluation of a Datalog program
\cite{Krotzsch2010ELDatalog,Carral2019Datalog}, and Stage~2 of our
pipeline (\S\ref{sec:pipeline}) uses precisely this view, treating the
ELK saturation as a monotone Datalog program. Such a rewriting yields
a reasoning procedure, not a differentiable one. Probabilistic \ELpp\
\cite{Ceylan2014,Ceylan2017,GutierrezBasulto2011}
attaches weights via Bayesian-network encodings; fuzzy \ELpp\
\cite{Zhao2025DFELpp,Bobillo2011} replaces model-theoretic
semantics with a t-norm-based approximation.

\paragraph{Knowledge compilation, WMC, and SDDs.}
For a propositional formula $\varphi$ and literal weights
$w(\ell) \ge 0$, the weighted model count is
$\WMC(\varphi; w) = \sum_{M \models \varphi}\prod_{\ell\in M}
w(\ell)$ \cite{chavira2008probabilistic}. With arbitrary
non-negative weights $\WMC$ is a generic algebraic quantity; it
acquires a probabilistic reading precisely when the weights are
\emph{per-variable normalized}, i.e.\ $w(X) + w(\neg X) = 1$ for
every variable $X$. Under that condition the weights specify a
product distribution over truth assignments, and
$\WMC(\varphi; w)$ is exactly the probability that a sample from
this distribution satisfies $\varphi$: it lies in $[0,1]$ and
$\WMC(\varphi; w) + \WMC(\neg\varphi; w) = 1$
\cite{chavira2008probabilistic}. \textsc{Moose} enforces this
normalization by construction (\S\ref{sec:methods-task},
\Cref{tab:atom-weight}), and \Cref{thm:end-to-end} establishes the
resulting WMC/distribution-semantics identity for the compiled
circuit. When the per-variable weights are produced by a neural
network, $\WMC$ is differentiable in the network's outputs.
Knowledge compilation
translates $\varphi$ to a circuit whose structural properties
(smoothness, decomposability, determinism) make $\WMC$ linear in
circuit size and reduce probabilistic queries to circuit traversals
\cite{Darwiche2002}. A Sentential Decision Diagram (SDD)
\cite{Darwiche2011,Choi2013} is one such circuit and supports
polynomial-time conjunction, disjunction, conditioning, and weighted
model counting; exact definitions are in \Cref{app:proofs}, and
\Cref{app:worked} shows a small example.

\paragraph{Knowledge-compilation NeSy and partial supervision.}
A predictor composes a neural concept extractor
$p_\theta(\mathbf{c} \mid \mathbf{x})$ with a symbolic constraint
$\varphi$ and trains by maximizing the marginal
\begin{equation}\label{eq:nesy-marginal}
   p_\theta(\mathbf{y} \mid \mathbf{x}) =
   \sum_{\mathbf{c}}
   p_\theta(\mathbf{c} \mid \mathbf{x})\,
   \mathbf{1}[\varphi(\mathbf{c}) = \mathbf{y}],
\end{equation}
the standard objective across this family
\cite{Manhaeve2018,Xu2018,Ahmed2022SPL,Pryor2023,vanKrieken2025NeSyDM}.
DeepProbLog \cite{Manhaeve2018} grounds $\varphi$ as probabilistic
Datalog with neural-network annotated facts; Semantic Loss
\cite{Xu2018} adds $-\log\WMC(\varphi)$ on propositional constraints;
Semantic Probabilistic Layers \cite{Ahmed2022SPL} combine exact
probabilistic inference with logical constraints in a single tractable
circuit; Scallop \cite{Li2023Scallop} compiles differentiable Datalog
with provenance semirings; NeurASP \cite{Yang2020NeurASP} embeds
neural perception inside answer-set programs; and A-NeSI
\cite{vanKrieken2023ANeSI} amortizes the intractable WMC marginal with
a neural surrogate. A \emph{partial-supervision} instance reveals
labels for a subset of $\mathbf{y}$-atoms only; the latent concept
vector $\mathbf{c}$ is recovered through the WMC gradient
(\Cref{def:problem}). None of the existing frameworks accepts an OWL
ontology directly: the user must hand-translate existentials, role
hierarchies, and role chains, losing the soundness, completeness, and
polynomial-time saturation guarantees of consequence-based EL
reasoning \cite{Kazakov2014} in the process.

\subsection{The independence assumption and reasoning
  shortcuts}\label{sec:bg-rs}

Knowledge-compilation NeSy predictors almost universally factorize
$p_\theta(\mathbf{c} \mid \mathbf{x}) = \prod_i p_\theta(c_i \mid
\mathbf{x})$. Marconato et~al.\ \cite{Marconato2023} show that this
can attain optimal log-likelihood while learning unintended concept
semantics, i.e., a \emph{reasoning shortcut} (RS), and identify four
root causes (knowledge structure, ground-truth concept distribution,
objective, extractor architecture); RSBench
\cite{Bortolotti2024RSBench} quantifies RS rates across NeSy
architectures. \cite{vanKrieken2024Independence,vanKrieken2025RS}
prove that the independence factorization is incompatible with
\emph{RS-awareness}: placing non-trivial probability mass on every
constraint-consistent latent assignment. Two mitigations are
available: BEARS \cite{Marconato2024BEARS} keeps the independent
extractor and trains a $K$-encoder ensemble whose members commit to
distinct shortcut explanations; NeSyDM \cite{vanKrieken2025NeSyDM}
replaces independence with a masked discrete-diffusion concept
distribution. Every published RS result is propositional, Datalog, or
finite-domain relational; no OWL fragment has been studied through
this lens. Related work on description logic compilation
\cite{Lazzari2026,Zhao2025DFELpp}, refinement
\cite{Daniele2023ILR,Aspis2022Embed2Sym}, and ontology embeddings
\cite{Kulmanov2019,Chen2021,Chen2025Survey} is reviewed in
\Cref{app:related}.

% =========================================================================
\section{Methods}\label{sec:methods}

\begin{figure}[t]
\centering
\includegraphics[width=\linewidth]{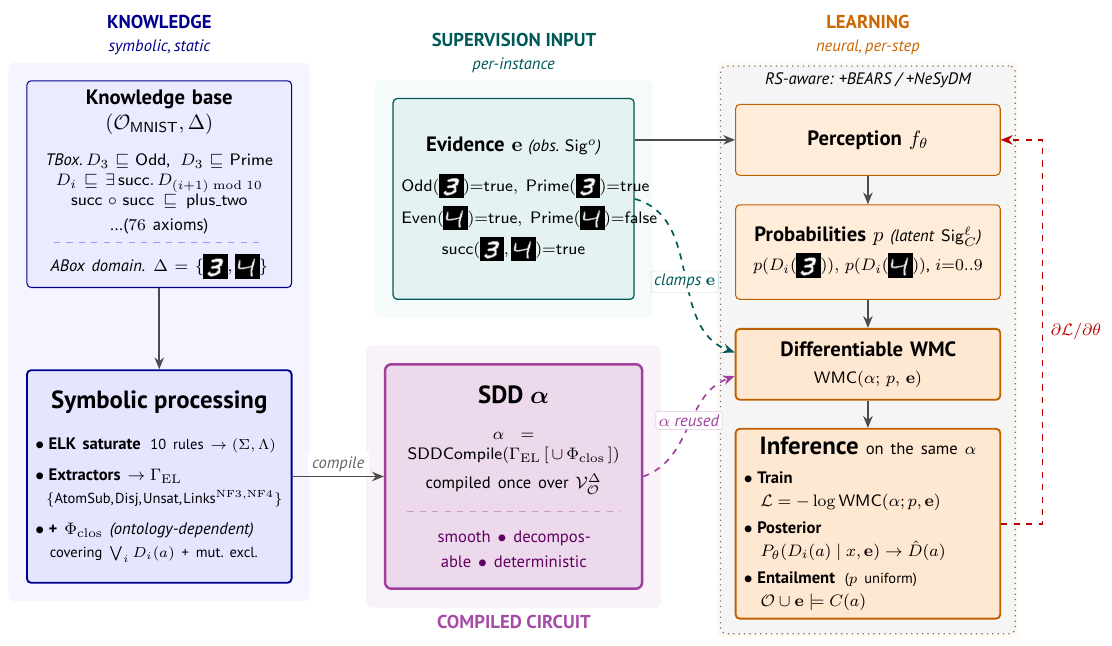}
\caption{End-to-end \textsc{Moose} workflow on
  $\Onto_{\textsf{MNIST}}$ at $\Delta{=}\{a,b\}$. The TBox is compiled
  once via ELK saturation $\to$ shape-aware extractors
  $\to \Gamma_{\mathrm{EL}}$ ($+\Phi_{\mathrm{clos}}$ when supplied)
  $\to$ SDD $\alpha$. The differentiable WMC layer evaluates
  $\WMC(\alpha; p_\theta(\mathbf{x}), \mathbf{e})$ for the training
  loss, the conditional posterior, and the entailment query.}
\label{fig:overview}
\end{figure}

\subsection{Task and ground vocabulary}\label{sec:methods-task}

Given an \ELpp\ ontology $\Onto$ and a finite ABox domain
$\Delta = \{a_1,\dots,a_m\}$, the \emph{ground vocabulary} is
$\mathcal{V}_\Onto^\Delta := \{C(a) : C\in\Sig_C, a\in\Delta\} \cup
\{R(a,b) : R\in\Sig_R, a,b\in\Delta\}$. The modeller partitions each
signature into \emph{observable} and \emph{latent} names,
$\Sig_C = \Sig^o_C \uplus \Sig^\ell_C$ and
$\Sig_R = \Sig^o_R \uplus \Sig^\ell_R$. A training instance
$(\Delta, \mathbf{x}, \mathbf{e})$ comprises a per-individual input
tuple $\mathbf{x}=(x_a)_{a\in\Delta}$ with each $x_a\in\mathcal{X}$
and a partial set of ABox literals $\mathbf{e}$ over the observable
ground vocabulary (Definition~\ref{def:problem},
\Cref{app:atom-weight}). We learn a weight-shared per-individual
classifier $f_\theta : \mathcal{X}\to[0,1]^{|\Sig^\ell_C|}$ that
approximates the marginal of each latent atom under
$\Onto\cup\mathbf{e}$. The differentiable WMC layer evaluates
$\WMC(\alpha; w, \mathbf{e})$ under the literal-weight map $w$ that
assigns $f_\theta(x_a)$ to each latent concept atom $C(a)$, its
observed truth value to each observed atom, $\tfrac12$ to each
unobserved observable and to each latent role atom, and $1-w(X)$ to
each negative literal $\neg X$, where $X$ ranges over the ground
atoms of $\mathcal{V}_\Onto^\Delta$ (\Cref{tab:atom-weight}).
Because $w(X) + w(\neg X) = 1$ for every $X$ by construction, the
map is per-variable normalized: it defines a product distribution
over the ground atoms whose per-atom marginal is $w$, and
$\WMC(\alpha; w, \mathbf{e})$ is the probability under this
distribution that a sample satisfies $\alpha$ and the evidence
$\mathbf{e}$ (\S\ref{sec:bg-prelim}; \Cref{thm:end-to-end}).
This product form encodes the per-atom \emph{independence}
factorization $p_\theta(\mathbf{c}\mid\mathbf{x}) = \prod_i
p_\theta(c_i \mid \mathbf{x})$ shared by knowledge-compilation NeSy
predictors (\S\ref{sec:bg-rs}). We make the assumption explicit
here because it is the object of study rather than an incidental
choice: it is exactly the factorization that induces reasoning
shortcuts, which contribution~3 quantifies via
$\mathrm{RS}_\text{cons}$ and which the BEARS and NeSyDM variants
(\S\ref{sec:methods-train}) are designed to mitigate.

\begin{table}[t]
\centering\footnotesize\setlength{\tabcolsep}{4pt}
\caption{Literal-weight map $w$; the two weights of each variable sum
to $1$.}\label{tab:atom-weight}
\begin{tabular}{@{}llll@{}}
\hline
Literal & Weight & Literal & Weight \\
\hline
$C(a)$, $C \in \Sig^\ell_C$ & $f_\theta(x_a)[C]$
  & $C(a)$, $C \in \Sig^o_C$ unobs. & $\tfrac12$ \\
$C(a) \in \mathbf{e}$ & $1$
  & $R(a,b)$, $R \in \Sig^o_R$ obs. & $1$ or $0$ from $\mathbf{e}$ \\
$\neg C(a) \in \mathbf{e}$ & $0$
  & $R(a,b)$ latent or unobs. & $\tfrac12$ \\
\multicolumn{4}{@{}l@{}}{$\neg X$ for any atom $X$: $1 - {}$weight of $X$} \\
\hline
\end{tabular}
\end{table}

\subsection{Compilation algorithm}\label{sec:pipeline}

The algorithm (full pseudocode in \Cref{alg:pipeline},
\Cref{app:algorithm}) compiles an \ELpp\ ontology and a finite ABox
domain to an SDD over $\mathcal{V}_\Onto^\Delta$ in seven stages; the
SDD is built once at startup and reused across every training step.
Stage~1 runs the ELK saturation, which derives all entailed
subsumptions and role links and records the grounded rule instances
that justify each derivation. Stages~2--5 build a propositional
encoding of the full saturation closure, including deeply nested
existential derivations and cyclic dependency chains, via Clark
completion, strongly-connected-component (SCC) detection,
time-stamped unrolling, and conjunctive normal form (CNF) conversion.
Stages~6--7 ground the saturation onto the finite ABox domain and
compile the resulting clauses into the SDD that the differentiable
WMC layer evaluates.

The two groups of stages operate at different abstraction levels and
serve complementary purposes. Stages~2--5 produce a TBox-level
propositional theory whose variables are concept-level atoms such as
$x_{C \sub D}$ and $x_{E \xrightarrow{R} C}$, including derived
complex concepts (e.g.\ $C \sub \exists R.\exists S.D$); the input
axiom variables $\alpha_i$ in this theory are free, so the theory
supports \emph{probabilistic ontology reasoning}: assigning a weight
$p_i$ to each $\alpha_i$ and computing the WMC gives
$P(\text{query entailed})$. Stages~6--7 operate on the ground
vocabulary $\mathcal{V}_\Onto^\Delta$ whose variables are ABox atoms
$C(a)$ and $R(a,b)$; for \emph{NeSy learning} all ontology axioms are
certain and the WMC literal weights come from a neural network's
per-atom outputs (\Cref{tab:atom-weight}), so the input-axiom tracking
of Stages~2--5 is not consumed.

In the NeSy learning use case, Stages~6--7 return to the Stage~1
saturation output and extract the subset relevant for ABox grounding:
subsumptions, disjointness, and role links between named concepts.
Deeply nested existential derivations that Stages~2--5 process,
including those that create cyclic SCCs requiring time-stamped
unrolling, have no ABox-level counterpart: the complex concepts
produced during saturation (e.g.\ $\exists R.C$ or
$\exists R.\exists S.D$) do not lie in the named signature
$\Sig_C(\Onto)$, and the shape-aware extractors of Stage~6 correctly
filter them out. Stages~2--5 still serve a diagnostic and validation
role: the Clark completion, SCC structure, and CNF size characterize
the TBox complexity and verify that the Stage~6 filter captures every
ABox-relevant consequence.

We summarize each stage below.

\paragraph{Stage 1: ELK saturation.}
\textsf{Saturate}$(\Onto, E)$ runs the ten ELK completion rules
\cite{Kazakov2014} and returns the saturation pair
$(\Sigma, \Lambda)$, where $\Sigma$ is the set of derived subsumptions
$C \sub D$ and $\Lambda$ the set of derived role links
$E \xrightarrow{R} C$, sound and complete with respect to $\Onto$
\cite{Kazakov2014}. We write
$\textsf{Sat}(\Onto) := \Sigma \cup \Lambda$ for the saturation
closure. For each derived atom the saturation records all grounded
rule instances that can derive it (there may be multiple
justifications), and these form the input to Stage~2. When the
ontology contains existential restrictions or role chains, the
saturation may derive non-atomic concepts (e.g.\ $\exists R.C$,
$\exists R.\exists S.D$); these are internal to the completion and are
consumed by Stages~2--5 but filtered out by the shape-aware extractors
of Stage~6.  Algorithmic details are in \Cref{app:algorithm}, and
\Cref{app:worked} works one axiom end-to-end through the
algorithm.

\paragraph{Stages 2--5: TBox-level encoding.}
The saturation of Stage~1 can be viewed as a monotone Datalog program
\cite{Clark1978}. Stages~2--5 turn it into a propositional TBox
theory: Clark completion \cite{Clark1978} replaces each derived
atom's justifications with a biconditional, Tarjan's algorithm
\cite{Tarjan1972} isolates cyclic strongly-connected components,
these are broken by SCC-local time-stamped unrolling, and the
Tseitin transformation \cite{Tseitin1983} yields an equisatisfiable
CNF. With all input-axiom variables true this theory has
$\textsf{Sat}(\Onto)$ as its unique model; when the axioms carry
weights it supports probabilistic ontology reasoning. For NeSy
learning these stages are diagnostic: the deeply nested and cyclic
derivations they handle lie outside the named signature and are
filtered by the Stage~6 extractors, so the ABox grounding of
Stages~6--7 consumes only the Stage-1 saturation. The construction
and its equations are given in \Cref{app:algorithm}.

\paragraph{Stage 6: shape-aware extractors and ABox grounding.}
Stage~6 extractors return to the Stage~1 saturation output
$(\Sigma, \Lambda)$ and select the subset that maps to ground ABox
atoms over $\mathcal{V}_\Onto^\Delta$.  \textsf{AtomSub} emits
$\neg A(a) \vee B(a)$ for every subsumption $A \sqsubseteq B$ in
$\Sigma$ where both $A$ and $B$ are named concepts in $\Sig_C(\Onto)$;
\textsf{Disj} emits $\bigvee_i \neg A_i(a)$ for $n$-ary atomic
disjointness $A_1 \sqcap \dots \sqcap A_n \sqsubseteq \bot$ (flattened
from any parenthesization); \textsf{Unsat} emits $\neg A(a)$ for
atomic $A \sqsubseteq \bot$; \textsf{Links} emits the NF3-forward
clause $\neg E(a) \vee \neg R(a,b) \vee C(b)$ per saturation-derived
atomic link $(E, R, C) \in \Lambda$ (i.e.\ $E \sqsubseteq \exists R.C$
with $E, C \in \Sig_C(\Onto)$), plus an NF4-reverse clause
$\neg R(a,b) \vee \neg C(b) \vee E(a)$ when
$L(E, R) := \{C : (E, R, C) \in \Lambda\}$ is a singleton (so the NF4
direction is deterministic). The union of the clauses produced by these
extractors is the propositional Horn program $\Gamma$ over
$\mathcal{V}_\Onto^\Delta$; the saturation closure is polynomial by
\Cref{thm:circuit-size} and grounding multiplies it by at most
$|\Delta|^2$ (\Cref{app:circuit-size-proof}).

\paragraph{Stage 7: SDD compilation.}
$\Gamma$ is compiled bottom-up into an SDD $\alpha$ that is smooth,
decomposable, and deterministic \cite{Choi2013,Darwiche2002}.
The role-pinning caveat applies to the \textsf{Links} clause:
ELK derives $E \sqsubseteq \exists R.C$ (existential, ``some
witness''); \textsc{Moose} encodes the propositional clause
$E(a) \wedge R(a,b) \to C(b)$, which is universal on the named
pair. Soundness depends on $\Delta$ being closed and on role
atoms over $\Sig^o_R$ being pinned by evidence (a
\textsc{Moose}-specific encoding choice for the
finite-named-domain regime, not an ELK theorem).
The
differentiable WMC layer traverses $\alpha$ and evaluates
$\WMC(\alpha; w)$ by the standard smooth/decomposable/deterministic
recursion (\Cref{eq:sdd-wmc} in \Cref{app:algorithm}),
with literal weights $w$ read from \Cref{tab:atom-weight}.
Decomposability factors the product at each decision node; determinism
makes the prime/sub pairs mutually exclusive, so the sum is plain
rather than an inclusion--exclusion. Conditioning on $\mathbf{e}$
clamps the literal weight of every observed atom to $0$ or $1$.

\subsection{Closure axioms outside the EL profile}\label{sec:closure}

OWL EL is open-world and lacks both right-hand-side disjunction and
cardinality constructors \cite{Baader2005}, so the covering axiom
$\TOP\sqsubseteq D_0\sqcup\dots\sqcup D_{K-1}$ on an exhaustive family
is unstatable in EL. Under partial supervision the WMC gradient on
unobserved members of such a family vanishes at every all-False
extension (e.g.\ on MNIST, observing $\textsf{Even}(a)$ without
closure leaves the digit posterior uninformative). We therefore
optionally append a finite set of \emph{closure axioms}
$\Phi_{\textsf{clos}}$ to $\Gamma_{\textsf{EL}}$ on each
modeller-declared exhaustive family $\mathcal{F}$: pairwise
disjointness, the covering clause $\bigvee_i D_i(a)$, and (when each
$D_i$ carries a distinguishing property profile $\pi_i$) profile-keyed
reverse implications. We write these three jointly manipulated
components as $\Phi_{\textsf{clos}} =
\Phi_{\textsf{mutex}} \cup \Phi_{\textsf{cover}} \cup
\Phi_{\textsf{profile}}$. Two of the three are outside the EL profile but
are standard OWL-DL machinery; \textsc{Moose}'s contribution is the
ELK-via-SDD infrastructure that absorbs $\Phi_{\textsf{clos}}$ at the
propositional level, so the same compilation yields
$\alpha_{\textsf{EL}}$ or $\alpha_{\textsf{clos}}$ with no method
change.  This absorption is free at the CNF level
($\Phi_{\textsf{clos}}$ adds at most $\binom{K}{2}m + (K{+}1)m$
clauses per declared family: $\binom{K}{2}m$ mutex, $m$ covering, and
up to $Km$ profile-keyed reverse implications), but the polynomial
saturation-closure bound of \Cref{thm:circuit-size} applies to the EL
Horn fragment only and does not transfer to
$\Gamma_{\textsf{EL}} \cup \Phi_{\textsf{clos}}$: covering
disjunctions can raise the primal-graph treewidth, and SDD compilation
is worst-case exponential in $|\Gamma|$ with a $O(|\Gamma| \cdot 2^w)$
bound under treewidth $w$ (\Cref{app:treewidth}); inference remains
$O(|\alpha|)$ in the SDD node count regardless
(\Cref{thm:inference-complexity}). Whether closure is needed is
ontology- and regime-dependent: a forward-only EL hierarchy under
factorized perception leaves the all-false latent assignment
consistent with any property evidence, so the WMC gradient is
uninformative and closure restores a non-trivial gradient (an
architectural softmax head encoding mutex+covering, as DeepProbLog
uses via $\textsf{nn}(\cdot)$, achieves the same effect by other
means); when forward subsumption alone pins the latent vector, no
closure is added. The choice is orthogonal to RS-awareness,
which targets the cross-atom factorization at the perception output
\cite{Marconato2023,vanKrieken2024Independence} and is implemented by
the BEARS / NeSyDM wrappers of \S\ref{sec:methods-train}.
\Cref{app:assumptions} summarizes the assumptions
introduced by each pipeline component and what each concedes against
open-world semantics.

\subsection{Inference and RS-aware variants}\label{sec:methods-train}

The differentiable WMC layer supports three inference operations on
the same SDD (training loss, conditional posterior, and entailment
query), varying only literal weights and evidence. The training
loss is
$L(\theta;\mathbf{x},\mathbf{e}) = -\log\WMC(\alpha;
p_\theta(\mathbf{x}),\mathbf{e})$ \cite{Xu2018}, which coincides with
the DeepProbLog marginal loss on this Horn theory by
\Cref{prop:dpl-equiv}. The conditional posterior
$P_\theta(C(a)\mid\mathbf{x},\mathbf{e})$ is the ratio of two WMCs
with $C(a)$ clamped vs.\ left free; the perception-free entailment
query $\Onto\cup\mathbf{e}\models C(a)$ holds iff
$\WMC(\alpha;w_{1/2},\mathbf{e}\cup\{C(a){:=}\mathsf{false}\}) = 0$ (the
$\tfrac12$-weighted SDD acts as a satisfiability oracle). All three
run in $O(|\alpha|)$ (\Cref{thm:inference-complexity},
\Cref{eq:loss,eq:posterior,eq:entail} in \Cref{app:inference-eqs}).
The compilation algorithm is independent of the perception's output
distribution, so RS-aware methods (\S\ref{sec:bg-rs}) plug in by
replacing $p_\theta(\mathbf{x})$ in $w$ without modifying $\alpha$.
\textsc{Moose}+BEARS \cite{Marconato2024BEARS} trains a $K$-encoder
ensemble diversified by Kullback--Leibler (KL) divergence against the
running ensemble average; \textsc{Moose}+NeSyDM
\cite{vanKrieken2025NeSyDM} replaces
$\prod_i p_\theta(c_i\mid\mathbf{x})$ with a masked discrete-diffusion
$q_\theta$, evaluated with both the REINFORCE leave-one-out (RLOO)
estimator \cite{Kool2019RLOO} and an exact-WMC gradient estimator. Wrapper losses with
hyperparameters are in \Cref{app:wrapper-losses}.

% =========================================================================
\section{Correctness and complexity}\label{sec:correctness}

This section states the four theorems that constitute the formal
contribution of \textsc{Moose}: a verified SDD encoding of the
ELK-derived saturation, the unconditional rational DISPONTE
correspondence, and the SCC-compositional factorizations at the Sat
and WMC levels.  Each theorem is mechanized in Lean~4; the
infrastructure on which they rest (ELK soundness/completeness on
\ELpp\ and polynomial-time Sat decidability \cite{Kazakov2014}) is
reused unchanged from the ELK literature and re-mechanized in our Lean
library (\Cref{app:elk-mechanization}).  Proofs, Lean theorem names,
and audit-surface details are deferred to
\Cref{app:proofs,app:lean-index}.

\paragraph{Notation.}
$\Onto$ ranges over \ELpp{} ontologies (OWL~2~EL minus datatype
properties and concrete domains).  $\sat(\Onto, C, D)$ holds iff ELK
derives $C \sub D$ from $\Onto$ \cite{Kazakov2014}.  A
\emph{world} $M : \Onto \to \{0,1\}$ selects the sub-ontology
$\selAx(\Onto, M) := \{\alpha \in \Onto : M(\alpha) = 1\}$
(\emph{sel} short for \emph{selectedAxioms} in the Lean library).
For a weight $w : \Onto \times \{0,1\} \to \mathbb{Q}$, the rational
DISPONTE marginal \cite{Riguzzi2015} is
\[
   \disponteRat(\Onto, C, D, w)
   \;:=\;
   \sum_{M : \Onto \to \{0,1\}}
     \mathbf{1}[\,\sat(\selAx(\Onto, M), C, D)\,]
     \cdot \prod_{\alpha \in \Onto} w(\alpha, M(\alpha));
\]
$\disponte(\Onto, C, D, w)$ denotes the natural-valued analogue when
$w$ takes values in $\mathbb{N}$.  $\compileSat(\Onto, C, D)$
(\emph{cmp} short for \emph{compileSat}) is the verified
Shannon-tree SDD produced by the algorithm of \S\ref{sec:pipeline},
and $\wmcRat(t, w)$ is its rational weighted model count.

\begin{theorem}[Verified SDD encoding]\label{thm:encoding}
  For every \ELpp{} ontology $\Onto$ and concepts $C, D$, there exists
  an SDD tree $t$ such that
  \begin{enumerate}
  \item $\mathrm{model}(t, M) \iff \sat(\selAx(\Onto, M), C, D)$ for
    every world $M : \Onto \to \{0, 1\}$;
  \item for every weight $w$, $\WMC(t, w) = \disponte(\Onto, C, D, w)$;
  \item $|t| \,=\, 2^{|\Onto| + 1} - 1$.
  \end{enumerate}
\end{theorem}

\noindent
\Cref{thm:encoding} is the formal contract of
\textsc{Moose}'s compilation step: the SDD's models are exactly the
worlds whose selected sub-ontology ELK-entails $C \sub D$, the SDD's
WMC is the DISPONTE marginal under arbitrary weights, and the
worst-case size is the explicit $2^{|\Onto|+1}-1$ bound of the Shannon
expansion (\Cref{app:treewidth}).  Lean proof: see \Cref{app:proofs}.

\begin{theorem}[Rational DISPONTE correspondence]\label{thm:end-to-end}
  For every \ELpp{} ontology $\Onto$, every concepts $C, D$, and every
  rational weight $w$,
  \begin{equation}\label{eq:end-to-end}
    \wmcRat\!\bigl(\compileSat(\Onto, C, D),\, w\bigr)
    \;=\;
    \disponteRat(\Onto, C, D, w).
  \end{equation}
\end{theorem}

\noindent
The identity holds unconditionally on rational weights: no
distributional assumption is required.  It is the formal warrant of the
standard entailment-as-WMC-zero phrasing of probabilistic DLs
\cite{Riguzzi2015} (at the uniform prior $w \equiv \tfrac{1}{2}$,
$\wmcRat = 0$ iff no world's sub-ontology entails $C \sub D$), and
lifts that phrasing from the propositional encoding to the verified
compiled circuit.  Lean proof: see \Cref{app:proofs}.

\begin{theorem}[SCC compositional theorem, Sat level]\label{thm:scc-sat}
  Let $\Onto = \Onto_1 \uplus \Onto_2$ with disjoint signatures, both
  components nominal-free and range-chain-safe, $\Onto_2$ consistent
  ($\neg\,\sat(\Onto_2, \TOP, \BOT)$), and $C, D$ nominal-free with $C,
  D$ in the signature of $\Onto_1$.  Then
  \[
    \sat(\Onto_1 \uplus \Onto_2, C, D) \;\iff\; \sat(\Onto_1, C, D).
  \]
\end{theorem}

\noindent
The result formalizes the intuition that an inferentially irrelevant
SCC cannot affect Sat-derivability inside the relevant component,
provided the irrelevant component is itself consistent. SCC-wise
compilation follows: each SCC can be saturated and compiled in
isolation, and the joint posterior recovered by combination.
Range-chain safety rules out interactions between role chains and
range axioms that would block the SCC factorization. It holds
vacuously when neither component has range axioms (the case for both
experimental ontologies); for chain-and-range ontologies, mechanized
syntactic range-elimination \cite{Baader2008ELFurther} reduces to this
case under a side condition that is automatically satisfied by
reflexive role-inclusion propagation. Lean proofs, theorem names, and
side-condition details: \Cref{app:proofs,app:lean-index}.

\begin{theorem}[Per-SCC posterior equivalence]\label{thm:scc-wmc}
Under the hypotheses of \Cref{thm:scc-sat}, for any uniform per-axiom
prior $w \equiv c$ with $c > 0$, the per-SCC and joint DISPONTE
posteriors coincide. In the unnormalized counting case $w \equiv 1$,
\begin{equation}\label{eq:scc-wmc-counting}
   \disponteRat(\Onto_1 \uplus \Onto_2, C, D, w)
   \;=\;
   \disponteRat(\Onto_1, C, D, w) \cdot 2^{|\Onto_2|}.
\end{equation}
\end{theorem}

\noindent
The multiplicative $2^{|\Onto_2|}$ from the irrelevant SCC cancels
under posterior normalization, and the same cancellation extends the
identity to any $c > 0$ after factoring out $c^{|\Onto_1|+|\Onto_2|}$.
This is the WMC-level analogue of \Cref{thm:scc-sat}: per-SCC
compilation is posterior-faithful. Lean proof:
\Cref{app:proofs,app:lean-index}.

\medskip\noindent
The closure-augmented variant of the SDD (which materializes the
per-individual partition constraints on declared exhaustive families)
satisfies an analogous correctness theorem
(\Cref{thm:closure-correct}, appendix).  Inference is linear in
$|\compileSat(\Onto, C, D)|$ \cite{Choi2013,Darwiche2002}
(\Cref{thm:inference-complexity}, appendix).  All proofs and the
DeepProbLog equivalence \Cref{prop:dpl-equiv} are in
\Cref{app:proofs}; the implementation index in
\Cref{app:lean-index} cross-references every paper claim with its
formalized counterpart.

% =========================================================================
\section{Research questions and experiments}\label{sec:emp}

We evaluate \textsc{Moose} on two benchmarks: the MNIST-with-ontology
benchmark we developed for this work (\Cref{ex:mnist}), instantiated
under three supervision regimes (atomic property literals, relational
$\exists R.C$ evidence, and role-chain evidence); and the Pizza\"iolo
dataset \cite{Pizzaiolo2024} of $4{,}800$ synthetic pizza images
generated to conform to the published OWL pizza ontology
\cite{PizzaOntology}, which lets us test transfer to a third-party
ontology and a different image distribution. Within each benchmark, the experiments
share the ontology and differ only in the observable signature, the
domain size $|\Delta|$, and the supplied evidence.

We address five research questions. \textbf{RQ1--RQ3} test latent
digit recovery on $\Onto_{\textsf{MNIST}}$ as the supervision signal
grows from a single individual with unary property literals (RQ1), to
a pair of individuals connected by an observed $r(a,b)$ role literal
under an $\exists R.C$ axiom (RQ2), to a pair connected by observed
role literals only through an ELK-derived role-chain consequence
(RQ3). \textbf{RQ4} tests transfer to the pre-existing third-party
OWL pizza ontology, restricted to its EL fragment (Pizza\"iolo,
\S\ref{exp:pizzaiolo}). \textbf{RQ5}
asks which RS-mitigation regime works under which conditions: with
encoder, ontology, and compiled SDDs fixed, we vary only the
output-distribution treatment: an independent per-atom baseline
(no logical structure), \textsc{Moose} on the EL theory alone (the
DPL surface, \Cref{prop:dpl-equiv}), \textsc{Moose} with closure
$\Phi_{\textsf{clos}}$, and the BEARS / NeSyDM wrappers, across
factorized relational, tied symmetric, and high-arity disjunctive
ambiguity. The independent baseline (binary cross-entropy,
BCE, on observed atoms) is a
lower-bound reference that cannot propagate evidence through
subsumption, disjointness, or role axioms and so cannot beat
chance on latent atoms
\cite{Marconato2023,vanKrieken2024Independence}.

\subsection{The MNIST ontology and supervision regimes}\label{sec:mnist}

A single OWL EL ontology $\Onto_{\textsf{MNIST}}$ ($14$ concepts, $2$
roles, $76$ axioms; full listing in \Cref{app:mnist-onto}) is used in
all three MNIST experiments; the experiments differ only in $|\Delta|$
and the evidence regime, never in the TBox. ELK saturation derives the
chain consequence
$D_i \sqsubseteq \exists\textsf{plus\_two}.\, D_{(i+2)\bmod 10}$ from
the axiom
$\textsf{succ}\circ \textsf{succ}\sqsubseteq\textsf{plus\_two}$ at
compile time. The compiled SDD has $14m + 2m^2$ ground atoms. A small
convolutional neural network (CNN) $f_\theta$ ($2$ conv blocks + $2$
fully-connected layers, weight-shared across individuals;
\Cref{app:perception}) maps each image to per-atom probabilities;
training minimizes
$\mathcal{L}(\theta) = -\log\WMC(\alpha_{\text{clos}};
p_\theta(\mathbf{x}), \mathbf{e})$ with Adam, batch~$32$, $5$ seeds;
per-regime learning rates and epoch counts are in \Cref{app:hp}.
EL-only ablations omit $\Phi_{\textsf{clos}}$ and train against
$\alpha_{\text{EL}}$, recovering the DPL loss
(\Cref{prop:dpl-equiv}). \textbf{Exp.~1 (atomic, RQ1)} uses
$\Delta=\{a\}$ with $n_\text{obs}\in\{1,2,3\}$ unary literals over
$\Sig^o_C=\{\textsf{Even},\textsf{Odd},\textsf{Prime},
\textsf{Composite}\}$. \textbf{Exp.~2 (relational, RQ2)} uses
$\Delta=\{a,b\}$ with the role atom $\textsf{succ}(a,b)$ asserted
(digits satisfy $\textsf{digit}(b)=(\textsf{digit}(a){+}1)\bmod 10$)
plus $n_\text{obs}\in\{1,2,3\}$ primality literals.  \textbf{Exp.~3
  (role chain, RQ3)} flips the asserted role to
$\textsf{plus\_two}(a,c)$ and widens the unary pool to parity and
primality; the only chain from $a$ to $c$ is the NF7-derived
consequence above. The digit family is declared exhaustive at every
individual; $\Phi_{\textsf{clos}}$ supplies pairwise disjointness,
covering, and reverse-implication clauses. Some property profiles
uniquely identify the latent digit (e.g.\
$\textsf{Even}(a) \wedge \textsf{Prime}(a)$ pins $D_2$), while
others leave it ambiguous (e.g.\ $\textsf{Even}(a)$ alone admits
$\{D_0, D_2, D_4, D_6, D_8\}$); RS$_\text{cons}$ on the ambiguous
cases measures confident commitment to a wrong
digit.\label{exp:mnist-atomic}\label{exp:mnist-rel}

\subsection{Experiment 4: Pizza\"iolo (RQ4)}\label{exp:pizzaiolo}

Pizza\"iolo
\cite{Pizzaiolo2024} provides $4{,}800$ synthetic pizza images.  The
canonical pizza ontology contains non-EL constructs (universal
restrictions on \textsf{hasTopping}, complement-based definitions such
as
$\textsf{VegetarianPizza} \equiv \textsf{Pizza} \sqcap \neg \exists
\textsf{hasTopping}.\textsf{MeatTopping} \sqcap \neg \exists
\textsf{hasTopping}.\textsf{FishTopping}$, and cardinality on
\textsf{NumberedPizza}); we use only its EL-expressible fragment,
restating the four-pizza recipes and property classes as conjunctive
subsumption and disjointness axioms, and treating classes whose
canonical definitions fall outside EL (e.g.\ \textsf{VegetarianPizza})
as atomic concepts whose truth values come from the dataset labels
rather than from non-EL closure.  The method targets the four-pizza
subset
$\{\textsf{Mushroom}, \textsf{Cajun}, \textsf{Capricciosa},
\textsf{FourSeasons}\}$ over $16$ topping concepts, with conjunctive
per-pizza axioms (e.g.\
$\textsf{Capricciosa} \sqsubseteq \textsf{Ham} \sqcap \textsf{Anchovy}
\sqcap \textsf{Olive} \sqcap \textsf{Peperonata}$), pairwise pizza
disjointness, and disjointness axioms
$P \sqcap T \sqsubseteq \bot$ for every topping $T$ outside $P$'s
recipe. EL forward subsumption plus these disjointness axioms pin
the topping vector once pizza identity is observed; no
out-of-profile closure is added.  \textbf{Track A} reveals pizza identity; toppings are
latent. We test distribution shift via an OOD split: the four training
pizzas all carry $\textsf{Anchovy}$ and $\textsf{Olive}$ together, so
a network can score well on the in-distribution split by coupling
the two atoms; we hold out five tie-breaker pizzas
($\textsf{Fiorentina}$, $\textsf{Giardiniera}$, $\textsf{LaReine}$,
$\textsf{Soho}$, $\textsf{Veneziana}$) whose recipes break this
co-occurrence.  \textbf{Track B} reveals one of four property-class
atoms ($\textsf{NonVegetarianPizza}$, $\textsf{SpicyPizza}$,
$\textsf{RealFrenchPizza}$, $\textsf{VegetarianPizza}$); revealing
$\textsf{NonVegetarianPizza}$ leaves
$\{\textsf{Cajun},\textsf{Capricciosa},\textsf{FourSeasons}\}$
indistinguishable (a 3-way RS over pizza identity).  \textbf{Track C}
(the $\textsf{is\_spicy}$ task in the codebase) adds four axioms
$T_i \sqsubseteq \textsf{SpicyTopping}$ for
$T_i \in \{\textsf{Jalapeno},\textsf{Peperonata},
\textsf{PeperoniSausage},\textsf{Prawn}\}$ together with
$\textsf{SpicyTopping} \equiv \textsf{SpicyPizza}$ (two GCIs; here
$\textsf{SpicyTopping}$ acts as a label-class denoting ``pizza with
a spicy topping'').  Symbolic supervision $\textsf{SpicyPizza}(a)$
forces $\textsf{SpicyTopping}(a)$ but the EL theory leaves the witness
ambiguous over $2^4{=}16$ topping subsets (the multi-witness regime
BEARS/NeSyDM target). We train the perception CNN from scratch with
batch size~$32$, $5$ seeds, and $60$ epochs on Tracks~A and~B
(Track~C uses $15$ epochs); per-method hyperparameters and the
NeSyDM $(\gamma_c, \gamma_h)$ sweep results are in \Cref{app:hp}.

\subsection{Baselines, metrics, and results}\label{sec:results}

We compare \textsc{Moose} against seven baselines.
\textbf{Independent} (BCE on observed atoms) is the NeSy-free
lower bound. \textbf{Semantic Loss} \cite{Xu2018} shares the
$-\log\WMC$ objective but compiles only the directly-stated
NF1/NF2 atomic axioms, dropping NF3/NF4/NF7 and ELK saturation,
isolating the EL-aware method of \S\ref{sec:pipeline}.
\textbf{DeepProbLog} \cite{Manhaeve2018} hand-codes each MNIST
regime as a ProbLog program with one $\textsf{nn}(\cdot)$ digit
directive; by \Cref{prop:dpl-equiv} DPL and \textsc{Moose} compute
the same loss on the same Horn theory, though that directive
is an annotated disjunction, so DPL is not closure-free.
\textbf{\textsc{Moose}+BEARS} \cite{Marconato2024BEARS} keeps the
factorized perception and $\alpha_{\text{clos}}$ but replaces the
single encoder with $K{=}5$ diversified encoders.
\textbf{\textsc{Moose}+NeSyDM} \cite{vanKrieken2025NeSyDM}
replaces the factorized extractor with a masked-diffusion concept
distribution; we evaluate both RLOO and exact-WMC gradient
estimators. \textbf{LTN} \cite{Badreddine2022LTN} substitutes a
fuzzy product T-norm; \textbf{ELEmbeddings}
\cite{Kulmanov2019,mOWL2023} embeds the EL ontology into ball
geometries with a classifier head; both are detailed in
\Cref{app:non-wmc-baselines}.

\textbf{Metrics.} The headline tables
\Cref{tab:headline-acc,tab:headline-ece} report
the regime's principal accuracy and ECE, the expected
calibration error binned by predicted confidence
\cite{Guo2017Calibration}. The two ontologies do not admit the same
accuracy metric. On MNIST the digit family is declared exhaustive and
we report $\text{Acc}_F$, \emph{family-argmax} accuracy: per
individual, the argmax of the WMC posterior over that family. On
Pizza\"iolo we report $\text{Acc}_C$, per-atom accuracy on the latent
slice, since Track~C declares no exhaustive family and the decode
degenerates on the other two (\Cref{app:pizza-argmax}). The full per-regime
breakdown (\Cref{tab:headline-full},
\Cref{app:full-results}) additionally reports the negative
log-likelihood (NLL), the mean per-atom Bernoulli loss on the latent
signature; the RS-consistency rate RS$_\text{cons} =
\mathrm{Conf}\cdot(1{-}\mathrm{Acc}_F)$ on the latent family-argmax
(higher = more confident commitment to a wrong shortcut;
on Pizza\"iolo the complement uses $\text{Acc}_C$, and on Track~C it
is not applicable);
per-atom
accuracy $\text{Acc}_\text{atom}$ on the latent signature; and the
macro-averaged per-concept F1, $\text{F1}_\text{macro}$.
A paired significance analysis of the headline comparisons
over $20$ seeds is given in \Cref{app:significance}, and a
learning-rate sensitivity analysis in \Cref{app:lr-sensitivity}.

\providecommand{\revision}[1]{#1}
\begin{table}[t]
\caption{\revision{Principal accuracy} (\%, mean $\pm$ s.d.\
  over $5$ seeds. MNIST Experiments~1--3 (\Cref{sec:mnist}),
  Pizza\"iolo Tracks~A--C (\Cref{exp:pizzaiolo}); Track~A is the
  out-of-distribution split.
  \revision{Metric: $\text{Acc}_F$ (family-argmax) on MNIST,
  $\text{Acc}_C$ (per-atom, latent slice) on Pizza\"iolo; compare only
  within a column (\Cref{app:pizza-argmax}). Operative theory:
  $\Gamma_{\textsf{EL}}{+}\Phi_{\textsf{clos}}$ for the \textsc{Moose}
  rows on MNIST, $\Gamma_{\textsf{EL}}$ elsewhere, none for
  Independent (\S\ref{sec:closure}).}
  Bold = column-best; \revision{$^\dagger$ = highest but not
  significant (\Cref{app:significance}); DeepProbLog on Pizza\"iolo is
  not a separate run (\Cref{app:full-results}).}
  Full $\text{Acc}_\text{atom}$, NLL, RS$_\text{cons}$, and
  $\text{F1}_\text{macro}$ in
  \Cref{tab:headline-full}.}\label{tab:headline-acc}
\centering\scriptsize\setlength{\tabcolsep}{3pt}
\begin{tabular*}{\linewidth}{@{\extracolsep{\fill}}lcccccc}
\hline
       & \multicolumn{3}{c}{MNIST} & \multicolumn{3}{c}{Pizza\"iolo} \\
Method & Exp\,1 & Exp\,2 & Exp\,3 & Pz\,A & Pz\,B & Pz\,C \\
\hline
Independent & 13.2{\scriptsize\,$\pm$\,7.9} & \phantom{0}8.8{\scriptsize\,$\pm$\,4.5} & 12.2{\scriptsize\,$\pm$\,7.3} & 21.7{\scriptsize\,$\pm$\,9.7} & 25.0{\scriptsize\,$\pm$\,0.0} & 48.9{\scriptsize\,$\pm$\,5.6} \\
DeepProbLog & 42.1{\scriptsize\,$\pm$\,1.3} & 38.9{\scriptsize\,$\pm$\,3.2} & 59.6{\scriptsize\,$\pm$\,13.6} & 68.1{\scriptsize\,$\pm$\,0.3} & 84.5{\scriptsize\,$\pm$\,2.1} & 81.0{\scriptsize\,$\pm$\,3.0} \\
LTN & 25.9{\scriptsize\,$\pm$\,6.5} & 14.5{\scriptsize\,$\pm$\,4.9} & 11.0{\scriptsize\,$\pm$\,7.9} & 22.1{\scriptsize\,$\pm$\,5.2} & 32.5{\scriptsize\,$\pm$\,20.9} & 84.8{\scriptsize\,$\pm$\,0.4}\revision{$^\dagger$} \\
ELEm/mOWL & 34.9{\scriptsize\,$\pm$\,4.2} & 34.8{\scriptsize\,$\pm$\,3.6} & 27.4{\scriptsize\,$\pm$\,7.1} & 25.7{\scriptsize\,$\pm$\,7.3} & 25.0{\scriptsize\,$\pm$\,0.0} & 75.4{\scriptsize\,$\pm$\,6.6} \\
\hline
\textsc{Moose} & 48.1{\scriptsize\,$\pm$\,10.1} & 74.6{\scriptsize\,$\pm$\,0.3} & \textbf{96.1{\scriptsize\,$\pm$\,7.2}} & 68.1{\scriptsize\,$\pm$\,0.3} & 84.5{\scriptsize\,$\pm$\,2.1} & 81.0{\scriptsize\,$\pm$\,3.0} \\
\quad +BEARS & 50.1{\scriptsize\,$\pm$\,5.2}\revision{$^\dagger$} & 76.5{\scriptsize\,$\pm$\,5.9}\revision{$^\dagger$} & 89.2{\scriptsize\,$\pm$\,7.2} & 68.6{\scriptsize\,$\pm$\,0.3} & \textbf{87.9{\scriptsize\,$\pm$\,0.9}} & 84.2{\scriptsize\,$\pm$\,3.6} \\
\quad +NeSyDM (RLOO) & 46.6{\scriptsize\,$\pm$\,2.1} & 58.2{\scriptsize\,$\pm$\,1.4} & 62.7{\scriptsize\,$\pm$\,2.9} & \textbf{72.2{\scriptsize\,$\pm$\,1.3}} & 78.5{\scriptsize\,$\pm$\,0.8} & 63.9{\scriptsize\,$\pm$\,3.7} \\
\quad +NeSyDM (exact) & 49.4{\scriptsize\,$\pm$\,9.7} & 62.6{\scriptsize\,$\pm$\,6.1} & 60.1{\scriptsize\,$\pm$\,5.7} & 67.4{\scriptsize\,$\pm$\,1.4} & 84.7{\scriptsize\,$\pm$\,2.2} & 80.1{\scriptsize\,$\pm$\,1.5} \\
\hline
\end{tabular*}
\end{table}

\begin{table}[t]
\caption{Expected calibration error ECE (\%, mean $\pm$ s.d.).
  Lower is better. Same baselines, seed counts\revision{, operative
  theories,} and metric conventions as
  \Cref{tab:headline-acc}; bold = column-best per regime\revision{,
  not significance-tested}. Per-regime
  $\text{Acc}_\text{atom}$, NLL, RS$_\text{cons}$, and
  $\text{F1}_\text{macro}$ alongside ECE are in
  \Cref{tab:headline-full}
  (\Cref{app:full-results}).}\label{tab:headline-ece}
\centering\scriptsize\setlength{\tabcolsep}{3pt}
\begin{tabular*}{\linewidth}{@{\extracolsep{\fill}}lcccccc}
\hline
       & \multicolumn{3}{c}{MNIST} & \multicolumn{3}{c}{Pizza\"iolo} \\
Method & Exp\,1 & Exp\,2 & Exp\,3 & Pz\,A & Pz\,B & Pz\,C \\
\hline
Independent & \phantom{0}9.3{\scriptsize\,$\pm$\,4.5} & 13.4{\scriptsize\,$\pm$\,3.6} & \phantom{0}9.0{\scriptsize\,$\pm$\,8.2} & 58.3{\scriptsize\,$\pm$\,15.4} & 26.4{\scriptsize\,$\pm$\,17.1} & 22.1{\scriptsize\,$\pm$\,9.2} \\
DeepProbLog & 10.3{\scriptsize\,$\pm$\,0.9} & 12.3{\scriptsize\,$\pm$\,0.9} & \phantom{0}8.0{\scriptsize\,$\pm$\,2.9} & 30.8{\scriptsize\,$\pm$\,0.7} & 15.5{\scriptsize\,$\pm$\,2.1} & 15.3{\scriptsize\,$\pm$\,1.4} \\
LTN & 16.1{\scriptsize\,$\pm$\,6.7} & 10.0{\scriptsize\,$\pm$\,0.0} & 10.0{\scriptsize\,$\pm$\,0.0} & 52.7{\scriptsize\,$\pm$\,21.8} & 39.3{\scriptsize\,$\pm$\,17.4} & 14.1{\scriptsize\,$\pm$\,0.5} \\
ELEm/mOWL & \phantom{0}\textbf{3.9{\scriptsize\,$\pm$\,0.4}} & \phantom{0}7.2{\scriptsize\,$\pm$\,0.1} & \phantom{0}6.3{\scriptsize\,$\pm$\,0.1} & 49.7{\scriptsize\,$\pm$\,10.2} & 16.1{\scriptsize\,$\pm$\,0.2} & 11.8{\scriptsize\,$\pm$\,3.8} \\
\hline
\textsc{Moose} & 10.6{\scriptsize\,$\pm$\,1.8} & \phantom{0}4.9{\scriptsize\,$\pm$\,0.3} & \phantom{0}\textbf{1.1{\scriptsize\,$\pm$\,1.5}} & 30.8{\scriptsize\,$\pm$\,0.7} & 15.5{\scriptsize\,$\pm$\,2.1} & 15.3{\scriptsize\,$\pm$\,1.4} \\
\quad +BEARS & \phantom{0}\textbf{3.9{\scriptsize\,$\pm$\,1.0}} & \phantom{0}\textbf{4.0{\scriptsize\,$\pm$\,1.2}} & \phantom{0}5.6{\scriptsize\,$\pm$\,1.1} & 28.2{\scriptsize\,$\pm$\,1.5} & 12.3{\scriptsize\,$\pm$\,1.0} & 11.3{\scriptsize\,$\pm$\,3.4} \\
\quad +NeSyDM (RLOO) & \phantom{0}6.0{\scriptsize\,$\pm$\,0.6} & \phantom{0}7.3{\scriptsize\,$\pm$\,0.7} & \phantom{0}4.6{\scriptsize\,$\pm$\,0.8} & \phantom{0}\textbf{7.5{\scriptsize\,$\pm$\,1.2}} & \phantom{0}\textbf{2.6{\scriptsize\,$\pm$\,0.9}} & \phantom{0}8.9{\scriptsize\,$\pm$\,1.5} \\
\quad +NeSyDM (exact) & \phantom{0}5.1{\scriptsize\,$\pm$\,1.5} & \phantom{0}7.4{\scriptsize\,$\pm$\,1.2} & \phantom{0}7.9{\scriptsize\,$\pm$\,1.1} & 30.8{\scriptsize\,$\pm$\,3.2} & 15.3{\scriptsize\,$\pm$\,2.2} & \phantom{0}\textbf{6.9{\scriptsize\,$\pm$\,2.2}} \\
\hline
\end{tabular*}
\end{table}

\paragraph{Findings on Experiments~1--3 (RQ1--3, RQ5).}
At $|\Delta|{=}1$ (Exp.~1) \textsc{Moose} leads DeepProbLog on
$\text{Acc}_F$ ($48.1$ vs.\ $42.1$).
Independent stays near
random ($\text{Acc}_F=13.2$). BEARS cuts ECE from $10.6$ to $3.9$ at
no cost to accuracy; at $|\Delta|{=}1$ there is little diversifiable
structure left for the ensemble to exploit. NeSyDM matches
\textsc{Moose} on $\text{Acc}_F$ ($46.6$ and $49.4$ for the RLOO and
exact estimators vs.\ $48.1$). On the relational and role-chain
regimes (Exps.~2--3) \textsc{Moose} dominates DeepProbLog by tens of
points on $\text{Acc}_F$ ($74.6$ vs.\ $38.9$ on Exp.~2; $96.1$ vs.\
$59.6$ on Exp.~3): the EL-aware Links extractor propagates
evidence across the $|\Delta|{=}2$ ground individuals, but the
closure clauses $\Phi_{\textsf{clos}}$ are what make this signal
learnable. The ablation in \Cref{app:closure-ablation} shows the
base WMC objective falling to $17.4$ and $9.4$ (near chance) once
$\Phi_{\textsf{clos}}$ is removed, so the EL compilation supplies the
structure and $\Phi_{\textsf{clos}}$ the identifying constraint, with
reasoning-shortcut mitigations partially substituting for the latter
(\Cref{tab:closure-ablation}).
BEARS reduces RS$_\text{cons}$ (\Cref{tab:headline-full},
\Cref{app:full-results}) but gives up a few points of accuracy on
Exp.~3 ($89.2$ vs.\ $96.1$); NeSyDM lags even at tuned
$(\gamma_c, \gamma_h)$. LTN and
ELEmbeddings under-perform on $\text{Acc}_F$ ($14.5$ and $34.8$ on
Exp.~2): LTN's ECE saturates near the uniform prior ($10.0$,
information-vacuous), while ELEmbeddings' lower ECE on Exps.~2--3
($7.2$, $6.3$) does not translate into accuracy.
An inductive held-out-edge split, in which query
individuals appear only in relational configurations never
supervised, leaves \textsc{Moose}'s accuracy essentially unchanged
while NeSyDM collapses to near-chance (\Cref{app:inductive}).

\paragraph{Findings on Experiment~4 (RQ4--5).}
Track~A is an OOD generalization test on five held-out tie-breaker
pizzas whose recipes break the $\{\textsf{Anchovy}, \textsf{Olive}\}$
co-occurrence admissible on the four training pizzas; on the
in-distribution split most methods reach near-perfect topping
accuracy, but on the OOD split the \textsc{Moose} family clusters
around $67$--$72$\% on $\text{Acc}_C$
(\Cref{tab:headline-acc}). Track~B is the canonical RS test (3-way
symbolic ambiguity: a property-class observation leaves three pizzas
indistinguishable). Plain \textsc{Moose}-WMC commits to a single
constraint-consistent latent (``Mechanism~A'' of
\cite{Marconato2023}); BEARS lifts $\text{Acc}_C$ from $84.5$ to
$87.9$ via ensemble diversification on RS-suspect pizzas, and NeSyDM
(RLOO) trades a few points of accuracy ($78.5$) for the lowest ECE
of any row ($2.6$). Track~C ($\textsf{is\_spicy}$ disjunction) is
the multi-witness regime NeSyDM was designed for: among the
\textsc{Moose} variants BEARS leads on $\text{Acc}_C$ ($84.2$ vs.\
$81.0$ for plain \textsc{Moose}-WMC) and NeSyDM (exact) leads on ECE
($6.9$, the column best); LTN matches BEARS on accuracy
($84.8$, the highest value in the column but not
significantly so, \Cref{app:significance}) at higher ECE ($14.1$). On both
Track~B and Track~C the \textsc{Moose}-family results separate
along an argmax-vs-calibration axis (BEARS leads accuracy,
NeSyDM leads ECE). These mitigation effects are sizeable
rather than marginal: paired over $20$ seeds, BEARS significantly
exceeds plain \textsc{Moose}-WMC on both Track~B and Track~C, and
NeSyDM significantly exceeds it on the Track~A OOD split
(\Cref{app:significance}). In the symbolically ambiguous regimes it
is therefore the RS-mitigation wrapper, not the exact base layer,
that is the stronger configuration; the mitigation is doing
substantive work, not fine-tuning. This reverses the relational
regime (Exps.~2--3), where forward subsumption already pins the
latent vector, no residual reasoning shortcut remains to remove, and
plain \textsc{Moose} is best; mitigation helps precisely where
symbolic ambiguity leaves an RS to exploit. Whether this generalizes
beyond the regimes tested here is left to future work.

% =========================================================================
\enlargethispage{3\baselineskip}
\section{Conclusion and outlook}\label{sec:concl}

\textsc{Moose} compiles OWL~2~EL into a Lean~4-verified
weighted-model-counting layer with native role chains and role
hierarchies, and supplies the first end-to-end formally verified
circuit for an OWL profile together with the first reasoning-shortcut
analysis in OWL~EL. Empirically, the closure-augmented variant
dominates propositional NeSy baselines by tens of points on relational
and role-chain regimes, where the EL-aware grounding propagates
evidence that propositional encodings cannot and the
closure clauses render the resulting signal learnable
(\Cref{app:closure-ablation}); under symbolic
ambiguity BEARS and NeSyDM separate along an argmax-vs-calibration
axis.

\paragraph{Limitations.}
Three assumptions bound the present results. \emph{(i) Scale.} Our
experiments use $|\Delta| = 2$; the compiled SDD grows empirically as
$\approx|\Delta|^{2.9}$ with a $40\times$ compile-time jump at
$|\Delta| = 4$ (\Cref{app:treewidth}). Saturation and grounding are
polynomial and the compilation is exact, but we do not demonstrate
ontology-scale ABoxes such as SNOMED~CT or the Gene Ontology; that
regime requires lifted WMC and is the principal open problem.
\emph{(ii) Finite named domain.} \textsc{Moose} learns over a given
finite ABox of named individuals (\Cref{def:problem}); this defines
the ABox-supervised task rather than weakening it, but it does not by
itself perform open-domain inference over unnamed individuals.
\emph{(iii) Role-pinning.} Soundness of the \textsf{Links} encoding
treats an existential as universal on the named pair, assuming role
atoms over $\Sig^o_R$ are pinned by evidence
(\S\ref{sec:pipeline}, Stage~7), a \textsc{Moose}-specific choice
for the finite-named-domain regime, not an ELK theorem. The
theoretical contribution, a Lean-verified exact compilation and the
RS analysis, stands independently of the scaling outcome: it is the
learned-perception pipeline at ontology scale, not the correctness
result, that the scalability question concerns.

Open directions include lifted WMC \cite{Lazzari2026} for
larger ABox domains, \ALC{} rewriting that would extend the same
encoding to existential-on-the-left axioms, and evaluation on
biomedical ontologies such as SNOMED~CT and the Gene Ontology,
and link prediction over latent role assertions: the role atoms
$R(a,b)$ that the weight map (\Cref{tab:atom-weight}) currently fixes
at $\tfrac12$ when unobserved would instead be predicted from
perception, extending the same WMC layer from latent-concept learning
to latent-role learning.

\paragraph*{\ackname}
This work was supported by funding from King Abdullah University of
Science and Technology (KAUST), through the KAUST Center of Excellence
for Smart Health (KCSH), under award number 5932, and the KAUST Center
of Excellence for Generative AI, under award number 5940.

\paragraph*{Supplemental Material Statement:}
Source code, $\Onto_{\textsf{MNIST}}$, the Pizza\"iolo subset,
seeds, per-seed JSON, the Lean~4 proof library, and pseudocode are
at \url{https://github.com/bio-ontology-research-group/moose-iswc/}; full
proofs are given in \Cref{app:proofs}.

\paragraph*{Declaration of use of Generative AI.}
A large language model assisted with copy-editing and with Lean~4
transcription of the proofs in \Cref{app:proofs}. All technical
content (definitions, theorem statements, proof structure, algorithms,
experimental design and results) was authored, verified, and approved
by the authors, and every mechanized statement was checked by the Lean
type checker.

% =========================================================================
\bibliographystyle{splncs04}
\bibliography{moose}

% =========================================================================
\clearpage
\appendix
\section{\ELpp\ syntax, semantics, and ELK calculus}\label{app:elpp}

This appendix collects the formal foundations on which
\textsc{Moose} is built: the syntax and semantics of the \ELpp\
profile of OWL~2~EL, the Baader--Brandt--Lutz (BBL) normal-form rules
NF1--NF7 \cite{Baader2005}, and the ELK saturation calculus
\cite{Kazakov2014} that drives Stage~1 of the compilation pipeline.
Every construct, normal form, and rule listed here is also formalized
in the Lean~4 library introduced below; \Cref{app:proofs} contains
the full proofs of \S\ref{sec:correctness}.

\paragraph{The Lean library.}  All formalized statements cited in this
paper live in the Lean~4 library \texttt{ELKSDD}, distributed with the
supplementary release at
\url{https://github.com/bio-ontology-research-group/moose-iswc/}. Theorems are
organized into namespaces: \texttt{ELKSDD.ELpp} (the bulk of the
OWL~2~EL stack (syntax, semantics, the ELK saturation calculus,
canonical-model completeness, the SCC compositional results, the
DISPONTE correspondence, and the paper-citation theorems),
\texttt{ELKSDD.SDD} (sentential decision diagrams and weighted model
counting), \texttt{ELKSDD.RangeNorm} (the syntactic range-elimination
of \cite{Baader2008ELFurther}), and the \texttt{ELKSDD.EL} and
\texttt{ELKSDD.MiniEL} fragments for testing purposes.  The
DeepProbLog equivalence layer lives in a companion library under
namespace \texttt{Moose}.  Inline references of the form
\texttt{<namespace>.<theorem>} point to a specific theorem; the index
in \Cref{app:lean-index} cross-references every paper claim with its
mechanized counterpart.  Every theorem listed in \Cref{app:lean-index}
is \emph{audit-clean}: its proof depends only on the standard Lean
foundation
$\{\texttt{propext},\, \texttt{Classical.choice},\,
\texttt{Quot.sound}\}$ and the formalization contains zero
\texttt{sorry} or \texttt{admit}.  An audit pass in each library runs
\texttt{\#print axioms} on every paper-cited theorem at build time and
reports its dependency set.

\paragraph{Syntax.}  Fix countably infinite, pairwise-disjoint
vocabularies of \emph{atomic concept names}~$\mathsf{N_C}$, \emph{role
  names}~$\mathsf{N_R}$, and \emph{individual names}~$\mathsf{N_I}$.
\ELpp\ \emph{concepts} are built by the grammar
\[
   C, D \;::=\; \TOP \mid \BOT \mid A \mid \{a\}
            \mid C \sqcap D \mid \exists R.C,
\]
where $A \in \mathsf{N_C}$, $a \in \mathsf{N_I}$, and
$R \in \mathsf{N_R}$.  An \emph{ontology} $\Onto$ is a finite set of
axioms of one of the following shapes (\Cref{tab:elpp}): general
concept inclusions $C \sub D$, role inclusions $R \sub S$, role chains
$R_1 \circ R_2 \sub S$, range restrictions $\mathrm{range}(R) \sub C$,
reflexive-role declarations $\mathrm{Ref}(R)$, local-reflexivity
$\exists R.\mathsf{Self}$, and key axioms
$\mathrm{HasKey}(C, R_1, \dots, R_n)$.  The \emph{ground vocabulary}
of $\Onto$ over a finite individual set $\Delta$ is
$\mathcal{V}_\Onto^\Delta := \{A(a) : A \in \Sig_C(\Onto),\, a \in
\Delta\} \cup \{R(a, b) : R \in \Sig_R(\Onto),\, a, b \in \Delta\}$.

\paragraph{Semantics.}  An \emph{interpretation}
$\mathcal{I} = (\Delta^\mathcal{I}, \cdot^\mathcal{I})$ comprises a
non-empty domain $\Delta^\mathcal{I}$ and a mapping
$\cdot^\mathcal{I}$ that sends each $A \in \mathsf{N_C}$ to
$A^\mathcal{I} \subseteq \Delta^\mathcal{I}$, each
$R \in \mathsf{N_R}$ to
$R^\mathcal{I} \subseteq \Delta^\mathcal{I} \times
\Delta^\mathcal{I}$, and each $a \in \mathsf{N_I}$ to
$a^\mathcal{I} \in \Delta^\mathcal{I}$.  Concept and axiom semantics
extend $\cdot^\mathcal{I}$ to compound concepts and constraints as in
\Cref{tab:elpp}.  $\mathcal{I}$ is a \emph{model} of $\Onto$, written
$\mathcal{I} \models \Onto$, if it satisfies every axiom in $\Onto$;
$\Onto \models C \sub D$ means $C^\mathcal{I} \subseteq D^\mathcal{I}$
in every model $\mathcal{I}$ of $\Onto$.

\begin{table}[h]
\centering\small
\caption{\ELpp\ syntax and semantics.  Together with role inclusion,
  role chain, and range axioms, plus the ABox-style constructs
  (nominals, $\mathsf{Self}$, $\mathrm{HasKey}$), this is the OWL~2~EL
  profile minus datatype properties.}\label{tab:elpp}
\begin{tabular}{lll}
\hline
Construct & Syntax & Semantics \\
\hline
top                  & $\TOP$              & $\Delta^\mathcal{I}$ \\
bottom               & $\BOT$              & $\emptyset$ \\
atomic concept       & $A$                 & $A^\mathcal{I} \subseteq \Delta^\mathcal{I}$ \\
nominal              & $\{a\}$             & $\{a^\mathcal{I}\}$ \\
conjunction          & $C \sqcap D$        & $C^\mathcal{I} \cap D^\mathcal{I}$ \\
existential          & $\exists R.C$       & $\{a \mid \exists b:\, (a,b) \in R^\mathcal{I},\, b \in C^\mathcal{I}\}$ \\
local reflexivity    & $\exists R.\mathsf{Self}$ & $\{a \mid (a, a) \in R^\mathcal{I}\}$ \\
GCI                  & $C \sub D$          & $C^\mathcal{I} \subseteq D^\mathcal{I}$ \\
role inclusion       & $R \sub S$          & $R^\mathcal{I} \subseteq S^\mathcal{I}$ \\
role chain           & $R_1 \circ R_2 \sub S$ & $R_1^\mathcal{I} \circ R_2^\mathcal{I} \subseteq S^\mathcal{I}$ \\
range restriction    & $\mathrm{range}(R) \sub C$ & $\forall (a,b) \in R^\mathcal{I}:\, b \in C^\mathcal{I}$ \\
reflexive role       & $\mathrm{Ref}(R)$   & $(a,a) \in R^\mathcal{I}$ for all $a \in \Delta^\mathcal{I}$ \\
key                  & $\mathrm{HasKey}(C, R_1, \dots, R_n)$ & shared $R_i$-fillers force individual identity \\
\hline
\end{tabular}
\end{table}

\paragraph{BBL normal forms NF1--NF7.}  Baader, Brandt, and Lutz
\cite{Baader2005} show every \ELpp{} ontology can be transformed into
one whose GCIs have one of seven shapes by introducing fresh atomic
names, with the transformation conservative for atomic subsumptions:
NF1, $A_1 \sqcap \dots \sqcap A_n \sub B$ (atomic conjunctive LHS);
NF2, $\exists R.A \sub B$ (atomic existential LHS); NF3,
$A \sub \exists R.B$ (atomic existential RHS); NF4, $A \sub \BOT$
(atomic unsat); NF5, $A \sub B$ (atomic-atomic); NF6, $\TOP \sub A$
(top-LHS); NF7, $\BOT \sub A$ (vacuous).  Our Lean formalization
(namespace \texttt{ELKSDD.Normalize}) covers all seven rules via
the canonical-extension construction of BBL~2005, \S3.1.

\paragraph{ELK saturation calculus.}  Kazakov, Kr\"otzsch, and
Siman\v{c}\'ik \cite{Kazakov2014} give a saturation procedure that
decides Sat-derivability $\sat(\Onto, C, D)$ for the full OWL~2~EL
profile (minus datatypes) in time $O(|\Onto|^4)$.  The calculus
operates on a finite collection of \emph{atomic} subsumptions and
\emph{links} $E \xrightarrow{R} D$, closed under a fixed rule set:
$\mathrm{R_0}$ ($A \sub A$), $\mathrm{R_\TOP}$ ($A \sub \TOP$),
$\mathrm{R_\sub}$ (transitivity along axioms), $\mathrm{R^-_\sqcap}$
and $\mathrm{R^+_\sqcap}$ (conjunction), $\mathrm{R_\BOT}$
(unsatisfiability propagation), $\mathrm{R^+_\exists}$ and
$\mathrm{R_{\BOT\text{-}\exists}}$ (existential constructors),
role-side rules $\mathrm{R_{rinc}}$, $\mathrm{R_{rchain}}$ for role
inclusions and chains, plus specialized rules for ranges, reflexive
roles, $\mathsf{Self}$, the nominal-handling rules of
\cite[\S3]{Kazakov2014}, and the merging canonical-model construction
of \cite[\S6]{Kazakov2014} for shapes 1--4 of nominals together with
$\mathrm{HasKey}$.  Saturation terminates because all derivable atoms
are over the finite vocabulary of
$\Sub(\Onto) \cup \mathsf{Roles}(\Onto)$; the resulting closure has
size $O(n^2 r)$ for $n = |\Sub(\Onto)|$ and
$r = |\mathsf{Roles}(\Onto)|$ \cite[Thm.~1]{Kazakov2014}.  Our
formalization (namespace \texttt{ELKSDD.ELpp}) mechanizes the same
calculus, sound and complete on the nominal-free fragment (with
ranges, reflexive roles, $\mathsf{Self}$, role inclusions, and role
chains) and on the LHS-nominal and shallow-exist-RHS-nominal
extensions; shapes 2/4 (RHS-nominal GCIs $C \sub \{a\}$,
$\{a\} \sub \{b\}$) and $\mathrm{HasKey}$ are handled by the merging
canonical-model construction of \cite[\S6]{Kazakov2014}, fully
mechanized as \texttt{ELKSDD.ELpp.complete\_via\_mergedCanon\_regular}
and \texttt{ELKSDD.ELpp.complete\_via\_mergedCanon\_nom} (both
audit-clean), under the shallow concept restriction of
\cite[\S6]{Kazakov2014}: concepts may not contain deep existentials
with non-nominal targets.  The unified completeness theorem
\texttt{ELKSDD.ELpp.complete\_owl2el} cases on a fragment witness to
invoke the appropriate canonical-model construction; together with
\texttt{ELKSDD.ELpp.sound\_owl2el}, it gives Sat $\iff$ Entails on
every ontology meeting at least one of its five branch preconditions.
Every individual \ELpp\ construct (nominals, ranges, role chains,
reflexive, $\mathsf{Self}$, $\mathrm{HasKey}$) is supported in some
branch, but the branches are not jointly compatible: shapes 2/4
(RHS-nominal GCIs $C \sub \{a\}$, $\{a\} \sub \{b\}$) combined with
deep non-nominal existentials, and the merging fragment combined with
role chains or ranges, hit the case-disabling preconditions of
\cite[\S6]{Kazakov2014} (an ontology may have a model where $C$ is
empty and another where $C$ is forced into the $\{a\}$-class; a single
canonical construction must pick one); the supported shapes follow
the consequence-based reasoning literature
\cite{TenaCucala2019Consequence,Alrabbaa2025ELProofs}.
The concurrent verified-reasoner effort VEL \cite{MishraTahar2024VEL}
mechanizes BBL\,2005 \ELpp\ in Coq but covers strictly less than our
unified completeness theorem (no nominals, no RHS-nominal GCIs, no
ranges, no role chains, no HasKey).  The exact closure-size bound $|L| \le 5n^3(r+1)$ that
includes ranges, reflexive, and $\mathsf{Self}$ is
\texttt{ELKSDD.ELpp.sat\_closure\_total\_polynomial\_bound}.

\section{Task formalism and MNIST instantiation}\label{app:atom-weight}

\begin{definition}[ABox-supervised latent concept learning]\label{def:problem}
  A training instance is a triple $(\Delta, \mathbf{x}, \mathbf{e})$
  where $\mathbf{x} = (x_a)_{a\in\Delta}$ with each
  $x_a \in \mathcal{X}$ and $\mathbf{e}$ is a set of literals
  (truth-tagged ground atoms) over the observable ground
  vocabulary; ground atoms not in $\mathbf{e}$ are unobserved. The
  task is to learn a weight-shared per-individual classifier
  $f_\theta : \mathcal{X} \to [0,1]^{|\Sig^\ell_C|}$ such that
  $f_\theta(x_a)[C]$ approximates the marginal probability of $C(a)$
  given $\Onto$ and $\mathbf{e}$, for every $C \in \Sig^\ell_C$ and
  $a \in \Delta$. A negative literal $\neg C(a) \in \mathbf{e}$
  encodes the polarity of a ground variable in
  $\mathcal{V}_\Onto^\Delta$ (\ELpp\ itself has no concept
  negation); role atoms are not classified by $f_\theta$, and latent
  role atoms are marginalized at the uniform prior.
\end{definition}

A boolean assignment $M : \mathcal{V}_\Onto^\Delta \to \{0,1\}$ induces an ABox
interpretation $\mathcal{I}_M = (\Delta, \cdot^M)$ via
$C^M := \{a : M(C(a)) = 1\}$ and $R^M := \{(a,b) : M(R(a,b)) = 1\}$.
The differentiable WMC layer evaluates $\WMC(\alpha; w, \mathbf{e})$
where the literal-weight map $w$ is the unique extension of
$f_\theta(\mathbf{x})$ and $\mathbf{e}$ that makes the training-loss,
posterior, and entailment-query operations of \S\ref{sec:methods-train}
three settings of the same SDD traversal:

The literal-weight map itself is now inlined in
\S\ref{sec:methods-task} as \Cref{tab:atom-weight}.

\begin{example}[MNIST instantiation]\label{ex:mnist}
  \Cref{def:problem} instantiates on MNIST as follows. The TBox is
  $\Onto_{\textsf{MNIST}}$ over
  $\Sig_C = \{D_0, \dots, D_9,\, \textsf{Even}, \textsf{Odd},
  \textsf{Prime}, \textsf{Composite}\}$ and
  $\Sig_R = \{\textsf{succ}, \textsf{plus\_two}\}$ (full listing in
  \Cref{app:mnist-onto}). The observable signature is
  $\Sig^o_C = \{\textsf{Even}, \textsf{Odd}, \textsf{Prime},
  \textsf{Composite}\}$; the latent signature is
  $\Sig^\ell_C = \{D_0, \dots, D_9\}$ (the digit class is never
  revealed). A training instance places one MNIST image $x_a$ at each
  named individual and supplies one or two property literals as
  evidence, e.g.\
  $\mathbf{e} = \{\textsf{Even}(a),\, \neg\textsf{Prime}(a)\}$ for an
  image of digit~$4$. An independent per-atom baseline matching its
  predictions to the observed literals provably cannot disambiguate
  the latent atoms when several digits share the property profile
  (e.g.\ $D_3, D_5, D_7$ are all $\textsf{Odd} \sqcap \textsf{Prime}$)
  \cite{Manhaeve2018,Marconato2023,vanKrieken2024Independence}.
\end{example}

\subsection{Summary of pipeline assumptions}\label{app:assumptions}
\Cref{tab:assumptions} collects the assumptions that the pipeline
introduces beyond the source ontology, one per component, and states
for each what it concedes relative to the open-world semantics that
OWL~EL is designed for. The assumptions are not defects of the
implementation but scoping choices that define the ABox-supervised
learning task (\Cref{def:problem}); the right-hand column makes
explicit what a user gives up by adopting them.

\begin{table}[t]
\caption{Assumptions introduced by each component of the
\textsc{Moose} pipeline and their departure from open-world
semantics.}\label{tab:assumptions}
\centering\small
\begin{tabular}{p{0.20\linewidth}p{0.36\linewidth}p{0.36\linewidth}}
\hline
Component & Assumption & Departure from open-world semantics \\
\hline
EL fragment &
Background axioms lie in the \ELpp\ profile: conjunction and
existentials, no right-hand-side disjunction or cardinality. &
Exhaustiveness and counting are not natively expressible; they must
be supplied as closure axioms rather than read from the ontology. \\
Closed domain &
Learning and inference range over a given finite set of named
individuals. &
No certain answers are produced over anonymous or unnamed
individuals; entailment over the anonymous domain is not performed. \\
Role-pinning &
An existential $\exists R.C$ is scored on the observed role pair,
with role atoms over $\Sig^o_R$ fixed by the evidence. &
Unobserved role witnesses are not hypothesized; existential import
over the anonymous domain is dropped. \\
Declared families &
Mutual exclusion and covering hold only on families the modeller
declares exhaustive. &
The closed-world reading is local and opt-in; outside declared
families the open-world reading is retained. \\
Factorized perception &
The predictor factorizes $p_\theta(\mathbf{c}\mid\mathbf{x}) =
\prod_i p_\theta(c_i\mid\mathbf{x})$: atoms independent given the
image. &
Joint perceptual correlations among atoms are not modeled; this is
the reasoning-shortcut source, mitigated but not removed by BEARS and
NeSyDM. \\
\hline
\end{tabular}
\end{table}

\section{Compilation algorithm and worked example}\label{app:algorithm}

This appendix gives the seven-stage compilation pseudocode
(\Cref{alg:pipeline}) and works a concrete two-axiom ontology
through it (\Cref{app:worked}).

\subsection{Pseudocode}\label{app:algorithm-pseudocode}

\Cref{alg:pipeline} gives the full pseudocode of the seven-stage
compilation algorithm summarized in \S\ref{sec:pipeline}. Stages 2--5
(Clark completion of the saturation output, SCC detection via Tarjan's
algorithm, SCC-local time-stamped unrolling, and Tseitin CNF
conversion) operate on the TBox-level derived atoms and produce a
propositional CNF that exactly characterizes the ELK canonical
model. For acyclic ontologies (all SCCs are trivial), Stages~3--4 are
no-ops and the Clark biconditionals compile directly to CNF.  Two
implementation choices are essential for the soundness, completeness,
and termination arguments. The first is the \emph{signature side condition} on the
existential-introduction rule $R^+_\exists$: the rule fires only on
existentials already in $\Sub(\Onto)$ (line~1 of the algorithm). This
is the standard ELK rule \cite[\S5]{Kazakov2014} and is what keeps
saturation finite on ontologies with cyclic role inclusions. The
second is \emph{uniform super-role materialization} under the
role-composition rule $R_\circ$: the implementation routes every link
insertion through a helper that records the link for the base role
and, recursively, for every strict super-role. Without this step, an
ontology with $R_1 \circ R_2 \sub S$ and $S \sub T$ produces the
$S$-link from the chain rule but silently loses the $T$-link, breaking
ELK completeness (\Cref{thm:el-coverage}) on chain heads with a
super-role.

\begin{algorithm}[!htbp]
\caption{\textsc{Moose}: \ELpp\ ontology to SDD over ground atoms}\label{alg:pipeline}
\begin{algorithmic}[1]
\Require Ontology $\Onto$ as a set of GCIs $C \sub D$, role inclusions $R \sub S$, role chains $R_1 \circ R_2 \sub S$. Finite ABox domain $\Delta$.
\Ensure  SDD $\alpha$ over $\mathcal{V}_\Onto^\Delta$ (\Cref{thm:encoding}).
\Statex
\Statex \textbf{[Stage 1: ELK saturation]}
\State $\mathcal{E}_\exists \gets \{\exists R.D \mid \exists R.D \in \Sub(\Onto)\}$
       \Comment{signature side condition on $R^+_\exists$}
\State $(\Sigma, \Lambda) \gets \textsf{Saturate}(\Onto, \mathcal{E}_\exists)$
       \Comment{$\Sigma$: subsumptions $C \sub D$; $\Lambda$: links $E \xrightarrow{R} C$}
\Statex
\Statex \textbf{[Stages 2--5: Clark completion, SCC, unrolling, CNF]}
\For{each derived atom $h \in \Sigma \cup \Lambda$}
    \State $\textsf{Clark}(h) \gets h \leftrightarrow \bigvee_j (\bigwedge_{b \in \mathit{body}_j(h)} b)$
           \Comment{Stage 2: Clark completion}
\EndFor
\State $\textsf{SCCs} \gets \textsf{Tarjan}(\textsf{DepGraph}(\textsf{Clark}))$
       \Comment{Stage 3: SCC detection}
\For{each non-trivial SCC $S$ of size $T{=}|S|$ ($|S|>1$ or self-loop)}
    \Comment{Stage 4: time-stamped unrolling}
    \For{each $h \in S$}
        \State emit $x_h^{(0)} \leftrightarrow \bot$;\;
               $x_h \leftrightarrow x_h^{(T)}$
    \EndFor
    \For{each $h \in S$, $t = 0, \dots, T{-}1$}
        \State emit $x_h^{(t+1)} \leftrightarrow x_h^{(t)} \lor
               \bigvee_j(\bigwedge_b x_b^{(t)})$
    \EndFor
\EndFor
\State $\textsf{CNF} \gets \textsf{Tseitin}(\text{all formulas})$
       \Comment{Stage 5: CNF conversion}
\Statex
\Statex \textbf{[Stages 6--7: ABox grounding and SDD compilation]}
\State $\textsf{AtomSub} \gets \{(X, Y) \in \Sigma : X, Y \in \Sig_C(\Onto),\ X \ne Y\}$
\State $\textsf{Disj} \gets \{\textsf{flatten}(C) : (C \sub \BOT) \in \Sigma,\ C\text{ atomic conjunction}\}$
\State $\textsf{Unsat} \gets \{X \in \Sigma : X \sub \BOT,\ X \in \Sig_C(\Onto)\}$
\State $\textsf{Links} \gets \{(X, R, Y) \in \Lambda : X, Y \in \Sig_C(\Onto)\}$
\State $\Gamma \gets \emptyset$
\For{$(X, Y) \in \textsf{AtomSub},\ a \in \Delta$} $\Gamma {\;\cup\!=\;} \{\neg X(a) \vee Y(a)\}$ \EndFor
\For{$(X_1, \dots, X_n) \in \textsf{Disj},\ a \in \Delta$} $\Gamma {\;\cup\!=\;} \{\bigvee_i \neg X_i(a)\}$ \EndFor
\For{$X \in \textsf{Unsat},\ a \in \Delta$} $\Gamma {\;\cup\!=\;} \{\neg X(a)\}$ \EndFor
\For{$(X, R, Y) \in \textsf{Links},\ (a, b) \in \Delta^2,\ a \ne b$}
    \State $\Gamma {\;\cup\!=\;} \{\neg X(a) \vee \neg R(a, b) \vee Y(b)\}$ \Comment{NF3 forward}
    \If{$|\{Y' : (X, R, Y') \in \textsf{Links}\}| = 1$}
        \State $\Gamma {\;\cup\!=\;} \{\neg R(a, b) \vee \neg Y(b) \vee X(a)\}$ \Comment{NF4 reverse}
    \EndIf
\EndFor
\State \Return $\textsf{SDDCompile}(\Gamma)$ \Comment{PySDD bottom-up}
\end{algorithmic}
\end{algorithm}

\subsection{Stages 2--5: the TBox-level encoding}\label{app:tbox-stages}

This subsection expands the Stage 2--5 block of \Cref{alg:pipeline},
which turns the Stage-1 saturation into a propositional TBox theory.
For NeSy learning these stages are diagnostic (\S\ref{sec:pipeline});
they are the compilation path used for probabilistic ontology
reasoning, where the input-axiom variables carry weights.

\paragraph{Stage 2: Clark completion.}
The saturation of Stage~1 can be viewed as a monotone Datalog program
\cite{Clark1978}: each derived atom $h$ has one or more grounded rule
body $\mathit{body}_1, \dots, \mathit{body}_k$ that can derive it.
The \emph{Clark completion} replaces these defining clauses with a
biconditional
\begin{equation}\label{eq:clark}
   h \;\longleftrightarrow\;
   \mathit{body}_1 \;\lor\; \mathit{body}_2 \;\lor\; \cdots
   \;\lor\; \mathit{body}_k,
\end{equation}
asserting that $h$ holds if and only if at least one of its
justifications holds. Each $\mathit{body}_j$ is a conjunction of
previously derived atoms and \emph{input axiom variables} $\alpha_i$
representing ontology axioms. The input variables are free (not
defined by any biconditional) and represent whether each axiom is
active: in standard reasoning all $\alpha_i$ are true; in
probabilistic reasoning each carries a weight. When the
dependency graph is acyclic, Clark completion has the Datalog least
fixed point as its only model for each assignment to the input
variables \cite{Clark1978}. For recursive monotone Datalog, however, completion alone
can also admit supported non-least fixed points. Stages~3--4 therefore
identify recursive strongly connected components and replace their
circular definitions by time-stamped unrolling from the all-false
interpretation. The completed and unrolled theory, rather than Clark
completion alone, has exactly one model; with all $\alpha_i$ true, that
model is exactly the set of atoms ELK derives from $\Onto$, i.e.\ the
saturation closure $\textsf{Sat}(\Onto)$ \cite{Kazakov2014}.

When the ontology contains axioms with existential restrictions,
derived atoms can have heads that contain nested existentials, e.g.\
$x_{D_i \sub \exists R.\exists R.D_k}$. For NeSy training, the Clark
completion is restricted to atoms whose head is an atomic concept, a
conjunction of atomic concepts, $\top$, or $\bot$ (the
\emph{groundable} atoms); all other atoms are filtered out. This
keeps Stages~3--5 tractable without losing any ABox-relevant
consequence, because the filtered atoms have no corresponding ground
variable in $\mathcal{V}_\Onto^\Delta$.

\paragraph{Stage 3: SCC detection.}
The dependency graph among derived atoms may contain cycles. For
example, a cyclic successor chain
$D_0 \sub \exists\textsf{succ}.D_1 \sub
\exists\textsf{succ}.\exists\textsf{succ}.D_2 \sub \cdots \sub D_0$
(through the existential derivations) produces circular dependencies
among the derived concept-level atoms.  Plain biconditionals have
circular definitions in this case. We apply Tarjan's algorithm
\cite{Tarjan1972} to find strongly connected components (SCCs) of the
dependency graph.  Acyclic atoms (trivial SCCs) can use the plain
biconditional of \Cref{eq:clark} directly. Non-trivial SCCs (size
$> 1$, or self-loops) require time-stamped unrolling in Stage~4.

\paragraph{Stage 4: propositional compilation with time-stamped
unrolling.}
For acyclic atoms the biconditional of \Cref{eq:clark} is
directly expressible in propositional logic. For cyclic SCCs we
break the circularity by introducing time-stamped copies
$x_h^{(0)}, x_h^{(1)}, \dots, x_h^{(T)}$ of each atom $h$ within
the SCC and defining:
\begin{align}
   x_h^{(0)} &\;\longleftrightarrow\; \bot,
      \label{eq:unroll-base}\\
   x_h^{(t+1)} &\;\longleftrightarrow\;
      x_h^{(t)} \;\lor\;
      \bigvee_{j}\Bigl(\bigwedge_{b \in \mathit{body}_j(h)}
         x_b^{(t)}\Bigr),
      \label{eq:unroll-step}\\
   x_h &\;\longleftrightarrow\; x_h^{(T)},
      \label{eq:unroll-final}
\end{align}
where the local bound $T = |S|$ (the SCC size) suffices because
each iteration of the monotone operator must derive at least one
new atom or the fixed point is already reached. Only atoms
within the \emph{same} non-trivial SCC receive time-stamped
copies; atoms from other SCCs (already computed in topological
order) use their plain variable. This gives total cost
$\sum_{S} |S|^2$ instead of $|\mathcal{V}|^2$ for global unrolling.

\paragraph{Stage 5: CNF conversion.}
The Tseitin transformation \cite{Tseitin1983} converts the nested
propositional formulas from Stage~4 into equisatisfiable conjunctive
normal form (CNF) by introducing an auxiliary variable for each
sub-formula $\varphi$ and adding clauses that enforce the auxiliary to
equal $\varphi$. The resulting CNF is linear in the size of the
Stage~4 output and encodes the full TBox-level propositional
theory. In the probabilistic ontology reasoning use case, this CNF is
the input to an SDD compiler that produces a circuit over the
input-axiom variables; in the NeSy learning use case, the CNF serves a
diagnostic role (characterizing the TBox complexity) and the ABox
grounding of Stage~6 produces its own clause set over
$\mathcal{V}_\Onto^\Delta$.

\paragraph{WMC recursion.}
The differentiable WMC layer of Stage~7 traverses the compiled SDD
$\alpha$ and evaluates
\begin{equation}\label{eq:sdd-wmc}
   \WMC(\alpha; w) \;=\;
   \begin{cases}
      w(\ell)
        & \text{if } \alpha = \ell \text{ is a literal node,} \\
      0 & \text{if } \alpha = \bot, \\
      1 & \text{if } \alpha = \top, \\
      \displaystyle\sum_{(p_i, s_i) \in \alpha}
         \WMC(p_i; w)\cdot\WMC(s_i; w)
        & \text{if } \alpha \text{ is a decision node,}
   \end{cases}
\end{equation}
with literal weights $w$ read from \Cref{tab:atom-weight}.

\subsection{Worked example}\label{app:worked}

Consider the two-axiom fragment
$\Onto_5 := \{D_5 \sqsubseteq \textsf{Odd},\, D_5 \sqsubseteq
\textsf{Prime}\}$ on $\Delta = \{a\}$, drawn from
$\Onto_{\textsf{MNIST}}$ (\S\ref{sec:mnist}). \textbf{Stage 1} runs
the ELK rule $R_\sub$ over $\Onto_5$ and adds the two input axioms to
$\Sigma$ as atomic subsumptions; no role axioms, so
$\Lambda = \emptyset$.  Stages~2--5 (Clark completion, SCC detection,
propositional compilation, CNF conversion) produce a propositional
theory over the derived atoms; for this acyclic two-axiom fragment no
time-stamped unrolling is needed (both atoms are in trivial SCCs).
\textbf{Stage 6} applies the four extractors to $(\Sigma, \Lambda)$:
both subsumptions match the \textsf{AtomSub} shape (atomic-atomic,
$X \neq Y$), so
$\textsf{AtomSub} = \{(D_5, \textsf{Odd}),\, (D_5, \textsf{Prime})\}$;
the other extractors produce nothing.  The extractors ground each
$(X, Y) \in \textsf{AtomSub}$ on every $a \in \Delta$, emitting
$\Gamma = \{\,\neg D_5(a) \vee \textsf{Odd}(a),\; \neg D_5(a) \vee
\textsf{Prime}(a)\,\}$ over
$\mathcal{V}_{\Onto_5}^{\{a\}}$. \textbf{Stage 7} compiles $\Gamma$
into an SDD whose models are exactly the assignments in which
$D_5(a) = 1$ implies $\textsf{Odd}(a) = 1$ and
$\textsf{Prime}(a) = 1$.  The full $\Onto_{\textsf{MNIST}}$ extends
$\Onto_5$ with the analogous property axioms for the remaining digits,
the $45$ pairwise digit disjointness axioms (handled by
\textsf{Disj}), the $10$ NF3 successor existentials, and the NF7 role
chain (\Cref{app:mnist-onto}); the SDD is $155$ nodes over
$14$ atoms and compiles in under a quarter of a second
(\Cref{app:repro}).

\section{Full proofs}\label{app:proofs}

This part of the appendix contains the full proofs of the
\S\ref{sec:correctness} theorems
(\Cref{thm:encoding,thm:end-to-end,thm:scc-sat,thm:scc-wmc}), the
ELK mechanization of the reused infrastructure
(\Cref{thm:el-coverage,thm:circuit-size},
\Cref{app:elk-mechanization}), and the closure-augmented and
inference-time variants
(\Cref{thm:closure-correct,thm:inference-complexity}).  The final
subsection (\Cref{app:lean-index}) cross-references each paper claim
with its Lean module and theorem identifier.  We use the notation
introduced in \S\ref{sec:correctness}.

\paragraph{Mechanized statements.}  The Lean library
\texttt{ELKSDD} (\Cref{app:elpp}) proves every prior-work result it
relies on rather than admitting it as an axiom: ELK
soundness/completeness \cite{Kazakov2014} is proved end-to-end in
namespaces \texttt{ELKSDD.EL} and \texttt{ELKSDD.ELpp}; SDD
compilation correctness \cite{Darwiche2002,Choi2013} and the
WMC-correctness theorem underlying
\Cref{thm:inference-complexity}(c) are proved structurally in
namespace \texttt{ELKSDD.SDD}; the four ELK extractors are concrete
\texttt{noncomputable def}s with proven membership and size
theorems; the Manhaeve 2018 distribution-semantics step of the
DPL-equivalence proof is closed both qualitatively
(\texttt{Iff.rfl}) and quantitatively (\cite{Darwiche2002} SDD-WMC
correctness, by structural induction).  Running
\texttt{\#print axioms} on every paper-cited theorem at build time
confirms that the dependency set is exactly
$\{\texttt{propext}, \texttt{Classical.choice},
\texttt{Quot.sound}\}$.

\subsection{Reused ELK infrastructure: soundness, completeness, and polynomial Sat decision}\label{app:elk-mechanization}

The two results below are restatements of the standard ELK
soundness/completeness theorem and the polynomial-time Sat decision
procedure of Kazakov, Kr\"otzsch, and Siman\v{c}\'ik
\cite{Kazakov2014}.  They are reused unchanged in
\S\ref{sec:correctness} but re-mechanized in our Lean library so that
\Cref{thm:encoding,thm:end-to-end,thm:scc-sat,thm:scc-wmc} can be
proved within the same audit-clean foundation.  $\fragWit$ abbreviates
the syntactic fragment witness selecting the canonical-model
construction used for completeness on $(\Onto, C, D)$:
(a)~nominal-free with $C, D$ in the signature of $\Onto$ and
range-chain-safe; (b)~the ABox-style LHS-nominal extension; (c)~the
shallow-exist RHS-nominal extension; (d)~the strict fragment without
ranges or chains; (e)~the merging canonical model
\cite[\S6]{Kazakov2014} for shapes 1--4 of nominals and
$\mathrm{HasKey}$, restricted to Shallow concepts (no deep
existentials with non-nominal targets).  The disjunction is the
inductive predicate \texttt{ELKSDD.ELpp.OWL2ELFragment}; together with
$\fragWit$ the unified completeness theorem
\texttt{ELKSDD.ELpp.complete\_owl2el} covers any ontology meeting at
least one branch precondition (a)--(e).
Every individual \ELpp\ construct is supported in some branch; the
residual unmechanized cases (shapes 2/4 (RHS-nominal GCIs) combined
with deep non-nominal existentials, and the merging fragment combined
with ranges or chains) correspond to the case-disabling preconditions
of \cite[\S6]{Kazakov2014} and constitute open work in the ELK
literature itself rather than gaps in the Lean formalization (see
\Cref{app:elpp} for discussion of a 2019 survey, a 2025
proof-theoretic study, and the concurrent VEL effort, none of which
extend the canonical-model coverage).  The atomic-shape vocabulary
$\mathcal{V}_\Onto^\Delta$ used by the SDD encoding
(\S\ref{sec:pipeline}) is a separate, deliberate scope choice: nested
existentials produce ELK-internal derivations but have no ground
counterpart in the SDD; ELK Sat-completeness via the unified
completeness theorem holds over all concepts, the SDD-encoding
completeness holds over the atomic-shape ABox slice
(\Cref{thm:encoding}).

\begin{theorem}[ELK soundness and completeness on \ELpp]\label{thm:el-coverage}
  For every \ELpp{} ontology $\Onto$ and concepts $C, D$, if
  $\fragWit(\Onto, C, D)$ holds then
\begin{equation}\label{eq:correctness}
   \sat(\Onto, C, D) \;\iff\; \Onto \models C \sub D.
\end{equation}
\end{theorem}

\begin{proof}
Soundness (left to right) is unconditional in $\fragWit$ and proved
by structural induction on the ELK derivation, mechanized as
\texttt{ELKSDD.ELpp.sound\_owl2el}.  Completeness (right to left)
is by case-split on $\fragWit$ to one of five
canonical-model constructions (Kazakov 2014 \S3, \S6 for the merging
case); mechanized as \texttt{ELKSDD.ELpp.complete\_owl2el}.  See \cite{Kazakov2014}
for the original arguments.\qed
\end{proof}

\begin{theorem}[Polynomial Sat decision]\label{thm:circuit-size}
For every \ELpp{} ontology $\Onto$ there exists a list $L$ of pairs
of concepts with $|L| \le |\Sub(\Onto)|^2$ such that for every $C, D
\in \Sub(\Onto)$,
\[
   \sat(\Onto, C, D) \;\iff\; (C, D) \in L.
\]
\end{theorem}

\begin{proof}
  The witness $L$ is the time-stamped Clark closure
  $\texttt{derivableClosure}(\Onto)$: length bound by
  \texttt{derivableClosure\_length} (counts concept pairs in the
  saturation closure); membership equivalence by
  \texttt{sat\_iff\_in\_derivableClosure} (simultaneous induction over
  the Sat constructors).  Mechanized as
  \texttt{ELKSDD.ELpp.sat\_decision\_polynomial}; an exact
  closure-size bound $|L| \le 5n^3(r+1)$ that includes ranges,
  reflexive, and Self ($n = |\Sub(\Onto)|$,
  $r = |\sym{Roles}(\Onto)|$) is
  \texttt{ELKSDD.ELpp.sat\_closure\_total\_polynomial\_bound}, with
  per-shape exact lengths.  See \cite[Thm.~1, \S5]{Kazakov2014} for
  the original $O(|\Onto|^4)$ saturation runtime. \qed
\end{proof}

\subsection{Verified SDD encoding (\Cref{thm:encoding})}

\begin{proof}[Proof of \Cref{thm:encoding}]
The witness is $t = \compileSat(\Onto, C, D)$, mechanized in
namespace \texttt{ELKSDD.ELpp} (with the SDD-side WMC theorems in
namespace \texttt{ELKSDD.SDD}).  Conjunct~(1): the compiled
Shannon tree's models are exactly the worlds whose selected
sub-ontology is Sat-derivable.  Conjunct~(2): a Shannon-recursive
unfolding of $\WMC$ enumerating the $2^{|\Onto|}$ leaves.
Conjunct~(3): by induction on $|\Onto|$, a Shannon node at depth
$k$ has $2^{|\Onto|-k+1}-1$ descendants.  Lean theorem identifiers
are listed in \Cref{tab:lean-index}. \qed
\end{proof}

\noindent
The standard probabilistic-DL ``entailment-as-WMC-zero'' phrasing
\cite{Riguzzi2015} follows immediately: at uniform per-axiom prior
$w \equiv \tfrac12$, $\wmcRat = 0$ iff no world's selected
sub-ontology entails $C \sub D$ (\Cref{thm:end-to-end} combined
with \Cref{thm:el-coverage}).  The production four-extractor ground
encoding (\S\ref{sec:pipeline}, Stages~6--7) is a separate ABox
encoding whose soundness over the finite domain follows from
\Cref{thm:el-coverage} (Lean: \texttt{ELKSDD.ELpp.sound\_owl2el}
and \texttt{ELKSDD.ELpp.complete\_owl2el}) under the
partial-supervision role-pinning of \S\ref{sec:pipeline}.

\subsection{Correctness of the closure-augmented circuit (\Cref{thm:closure-correct})}

\begin{theorem}[Correctness of the closure-augmented circuit]\label{thm:closure-correct}
For every modeller-declared exhaustive family
$\mathcal{F} = \{D_0, \dots, D_{K-1}\}$, define
$\mathsf{match}_{\pi}(a)$ as the conjunction of the positive and
negated property literals specified by profile $\pi$. Then
$M \models \alpha_{\text{clos}}$ iff $M \models
\alpha_{\text{EL}}$ and, at every $a \in \Delta$, both of the
following hold: (i) the partition constraints
$\{\neg(D_i(a) \wedge D_j(a)) : i < j\} \cup \{D_0(a) \vee \dots
\vee D_{K-1}(a)\}$; and (ii), for every distinct declared profile
$\pi$, the reverse implication
$\mathsf{match}_{\pi}(a) \rightarrow
\bigvee_{i:\,\pi_i=\pi}D_i(a)$.
\end{theorem}

\begin{proof}[Proof of \Cref{thm:closure-correct}]
$\alpha_{\text{clos}} =
\textsf{SDDCompile}(\Gamma_{\text{EL}} \cup
\Phi_{\text{clos}})$ by construction. For each marked family and
individual, $\Phi_{\text{mutex}}$ contains the $\binom{K}{2}$ clauses
$\neg D_i(a) \vee \neg D_j(a)$, and $\Phi_{\text{cover}}$ contains
$D_0(a) \vee \dots \vee D_{K-1}(a)$. These are exactly the partition
constraints in clausal form. For every distinct profile $\pi$,
$\Phi_{\text{profile}}$ additionally contains the clausal form of
$\mathsf{match}_{\pi}(a) \rightarrow
\bigvee_{i:\,\pi_i=\pi}D_i(a)$. Therefore, satisfaction of the full
clause set is equivalent to conditions (i)--(ii). Generic SDD
compilation correctness then gives the stated model equivalence
\cite{Choi2013,Darwiche2002}. The Lean theorem
\texttt{ELKSDD.SDD.compile\_correct} quantifies over the complete
input clause list, so it applies to all three components of
$\Phi_{\textsf{clos}}$.\qed
\end{proof}

\subsection{Production-encoding size}\label{app:circuit-size-proof}

\Cref{thm:circuit-size} bounds the saturation closure by
$|\sat(\Onto)| \le 5n^3(r{+}1)$ atoms with $n = |\Sub(\Onto)|$ and
$r = |\sym{Roles}(\Onto)|$ (Lean:
\texttt{ELKSDD.ELpp.\allowbreak sat\_closure\_\allowbreak total\_polynomial\_bound}).
Grounding the extractors of \S\ref{sec:pipeline} over $|\Delta| = m$
named individuals adds an at-most $m^2$ multiplicative factor
(\textsf{AtomSub}, \textsf{Disj}, \textsf{Unsat} contribute $m$
clauses each; \textsf{Links} contributes up to $m^2$ for the NF3/NF4
clauses).  Each declared exhaustive family of size $K$ adds a further
$\binom{K}{2}m + (K{+}1)m$ closure clauses ($\binom{K}{2}m$ mutex, $m$
covering, and up to $Km$ profile-keyed reverse implications).  SDD
compilation of the resulting CNF is treated in \Cref{app:treewidth}.

\subsection{SDD treewidth dependence and empirical scaling}\label{app:treewidth}

Bottom-up SDD compilation \cite{Choi2013} of a CNF $\Gamma$ produces
an SDD whose worst-case size is exponential in $|\Gamma|$
\cite{Darwiche2002}. For CNFs whose primal graph has treewidth $w$,
the resulting SDD admits a polynomial bound $O(|\Gamma| \cdot 2^w)$
because tractable WMC over bounded-treewidth instances is reducible to
compact circuit compilation. Empirically on \textsc{mnist} the SDD
node count fits $\propto |\Delta|^{2.9}$ for $|\Delta| \le 3$
(consistent with the bounded-treewidth regime), with a $40\times$
compile-time jump between $|\Delta| = 3$ and $|\Delta| = 4$ as the
primal-graph treewidth rises; at the $|\Delta| = 2$ configurations
used in \S\ref{sec:emp} compilation finishes well below one second on
a CPU (1$\times$ RTX~4090, single CPU core;
\Cref{app:repro}).

\subsection{Inference complexity (\Cref{thm:inference-complexity})}

\begin{theorem}[Inference complexity]\label{thm:inference-complexity}
  Let $|\alpha|$ denote the node count of either $\alpha_{\text{EL}}$
  or $\alpha_{\text{clos}}$. The training loss \Cref{eq:loss},
  conditional posterior \Cref{eq:posterior}, entailment query
  \Cref{eq:entail}, and one training step (loss plus gradient) all run
  in $O(|\alpha|)$ time \cite{Darwiche2002,Choi2013}, with the
  training step adding the cost of one forward and backward pass of
  $f_\theta$.
\end{theorem}

\begin{proof}[Proof of \Cref{thm:inference-complexity}]
  The SDD $\alpha$ produced by \Cref{alg:pipeline} is smooth,
  decomposable, and deterministic by PySDD construction
  \cite{Darwiche2002,Choi2013}; let $|\alpha|$ denote its node count.

  \textbf{(a) WMC.} The recursive traversal of \S\ref{sec:pipeline}
  visits each node once (memoized by node id), contributing $p_{C(a)}$
  or $1 - p_{C(a)}$ at literal nodes, $0$ or $1$ at constants, and
  $\sum_i \WMC(\textsf{prime}_i) \cdot \WMC(\textsf{sub}_i)$ at
  decision nodes. Determinism makes each sum disjoint and
  decomposability makes each product over disjoint variable scopes, so
  $\WMC(\alpha; p, \mathbf{e})$ is exact and runs in $O(|\alpha|)$
  time. Each per-node operation is differentiable in the literal
  weights, and PyTorch autograd composes the backward pass through the
  same traversal in $O(|\alpha|)$ additional time.

  \textbf{(b) Conditional posterior.} The ratio
  $P_\theta(C(a) \mid \mathbf{x}, \mathbf{e}) = \WMC(\alpha;
  p_\theta(\mathbf{x}), \mathbf{e} \cup \{C(a) := \mathsf{true}\}) /
  \WMC(\alpha; p_\theta(\mathbf{x}), \mathbf{e})$ is two applications
  of (a), each $O(|\alpha|)$. The denominator is strictly positive
  whenever $\mathbf{e}$ is consistent with $\Onto$.

  \textbf{(c) Entailment query.} Setting $p \equiv \tfrac12$ gives the
  uniform prior; conditioning on $\mathbf{e} \cup \{C(a) := \mathsf{false}\}$
  clamps the corresponding literal. By \Cref{thm:el-coverage} the
  SDD's models are exactly the finite-domain interpretations
  satisfying the entailed clauses, so
  $\WMC(\alpha; \tfrac12, \mathbf{e} \cup \{C(a) := \mathsf{false}\}) = 0$ iff
  no model of $\alpha$ has $C(a) = \mathsf{false}$, iff
  $\Onto \cup \mathbf{e} \models C(a)$. The check is one WMC call,
  $O(|\alpha|)$.

  \textbf{(d) Training step.} The loss
  $L(\theta; \mathbf{x}, \mathbf{e}) = -\log \WMC(\alpha;
  p_\theta(\mathbf{x}), \mathbf{e})$ is one WMC call,
  $O(|\alpha|)$. Its gradient with respect to $\theta$ is computed by
  autograd through the WMC traversal (cost $O(|\alpha|)$) followed by
  the perception backbone's backward pass (cost equal to the forward
  pass of $f_\theta$). Total:
  $O(|\alpha|) + O(\text{cost of } f_\theta(\mathbf{x}))$. \qed
\end{proof}

\subsection{Structural circuit properties}

We define three structural properties used to characterize the
tractability of weighted model counting on a circuit
\cite{Darwiche2002}. A \emph{circuit} over variables $\mathbf{V}$ is
a rooted directed acyclic graph whose leaves are literals over
$\mathbf{V}$ (or the constants $\top$ / $\bot$), and whose internal
nodes are either \emph{product nodes} ($\bigwedge$, evaluating to the
conjunction of their children) or \emph{sum nodes} ($\bigvee$,
evaluating to the disjunction of their children).  Let $\varphi$ be
such a circuit; the \emph{scope} of a node is the set of variables
appearing in the subcircuit rooted at it, and the \emph{support} is
the set of total assignments to $\mathbf{V}$ on which the node
evaluates to a non-zero value.

\begin{definition}[Decomposability]
  A circuit is \emph{decomposable} if at every product node the scopes
  of its children are pairwise disjoint.
\end{definition}

\begin{definition}[Determinism]
  A circuit is \emph{deterministic} if at every sum node the children
  have pairwise disjoint supports.
\end{definition}

\begin{definition}[Smoothness]
  A circuit is \emph{smooth} if at every sum node all children share
  the same scope.
\end{definition}

A Sentential Decision Diagram (SDD) is a circuit whose internal nodes
are \emph{decision nodes} of the form
$(p_1, s_1) \vee \dots \vee (p_k, s_k)$, abbreviating
$\bigvee_i (p_i \wedge s_i)$; each $p_i$ (the \emph{prime}) and $s_i$
(the \emph{sub}) are themselves SDDs over disjoint variable sets fixed
by a v-tree, with the prime ranging over the variables of the left
subtree and the sub over those of the right
\cite{Darwiche2011,Choi2013}.  This prime/sub decomposition realizes
all three structural properties above by construction: every prime
$p_i$ has scope disjoint from its corresponding sub $s_i$
(decomposability), the primes are pairwise inconsistent (determinism),
and they cover the input space allotted to the node by its v-tree
(smoothness) \cite{Darwiche2011,Choi2013}. Together, these properties make
$\WMC$ linear in $|\alpha|$ and reduce conditioning,
marginalization, and MAP inference to circuit traversals
\cite{Darwiche2002}.

\subsection{Compositional SCC factorization lemmas (\Cref{thm:scc-sat,thm:scc-wmc})}

The following lemmas establish the SCC factorization formalized in
Lean: an arbitrary two-component partition of an ontology,
$\Onto = \Onto_1 \cup \Onto_2$, with disjoint atom and role
signatures, factors the closure exactly, modulo a single
``global inconsistency'' disjunct that captures the case where
$\Onto_2$ alone derives $\top \sqsubseteq \bot$ (which then
propagates to all of $\Onto_1 \cup \Onto_2$ via the $R_\bot$ rule,
regardless of $\Onto_1$).  This is what allows the algorithm to
compile per-SCC SDDs and combine them, rather than compiling one
monolithic SDD over the union.

\begin{lemma}[Semantic SCC factorization]\label{lem:scc-factor}
Let $O_1, O_2$ be \ELpp\ ontologies (finite sets of GCIs, role
inclusions, and role chains). Let
\(
   \mathrm{atoms}(O_i) \subseteq \Sig_C
\)
and
\(
   \mathrm{roles}(O_i) \subseteq \Sig_R
\)
denote the sets of atomic concept names and role names appearing in
axioms of $O_i$. Assume signature disjointness:
\[
   \mathrm{atoms}(O_1) \cap \mathrm{atoms}(O_2) = \emptyset,
   \qquad
   \mathrm{roles}(O_1) \cap \mathrm{roles}(O_2) = \emptyset.
   \tag{Disj}
\]
Then for every pair of \ELpp\ concepts $C, D$ whose atom and role
names lie entirely in $O_1$'s signature
\(
   (\mathrm{atoms}(C) \cup \mathrm{atoms}(D) \subseteq
   \mathrm{atoms}(O_1)
   \) and
   \(
   \mathrm{roles}(C) \cup \mathrm{roles}(D) \subseteq
   \mathrm{roles}(O_1)),
\)
the closure factors as
\begin{equation}\label{eq:scc-factor}
   \sat(O_1 \cup O_2)\;C \sqsubseteq D
   \;\;\iff\;\;
   \sat(O_1)\;C \sqsubseteq D
   \;\lor\;
   \sat(O_2)\;\top \sqsubseteq \bot.
\end{equation}
The statement holds for the nominal-free, range-chain-safe fragment
of full OWL~2~EL (\ELpp\ minus concrete domains): both $O_1$ and
$O_2$ are nominal-free, both satisfy the range-chain safety condition
\texttt{ELKSDD.ELpp.RangeChainSafe} (range axioms are forbidden on
the rinc-ancestors of any role-chain target, so that range-guards
transfer vacuously through chain composition; \cite{Kazakov2014}~\S3.3),
and $C, D$ are nominal-free.  No restriction on $\top$- or
$\bot$-axioms is imposed on either side.

OWL~2~EL nominals (\texttt{ObjectOneOf} singleton classes) are
supported by the algorithm at the axiom level (normal forms
\texttt{NFnomR}~``$A \sub \{a\}$'' and
\texttt{NFnomL}~``$\{a\} \sub B$'') and at the concept level
(\texttt{Concept.nom}, \texttt{Interp.indiv},
\texttt{conceptIndividuals}, \texttt{ontologyIndividuals}) in our
Lean implementation.  The proof of \Cref{lem:scc-factor} below
covers the nominal-free fragment.  Shape~1 nominals (ABox-style \texttt{ClassAssertion} axioms
$\{a\} \sub D$ with $D$ nominal-free) are formalized by the relaxed
\texttt{ELKSDD.ELpp.Sat\_factor\_nomLHS} theorem, under the
\texttt{OntologyNomLHS} precondition (which permits LHS-nominal GCIs
on the analyzed side $O_1$) and \texttt{AllNomInhabited} (every
nominal index is consistent).  The key observation is that an
LHS-nominal $\{a\}$ pins the prodInterp evaluation to the specific
point $\langle\mathcal{I}_1.\textsf{indiv}(a), b_0\rangle$, so no
general nominal-evaluation lemma is needed: the existing P2 (which is
already general in the second coordinate) is applied at $b = b_0$.
Shapes~2/4 (RHS nominal: $A \sub \{a\}$ or $\{a\} \sub \{b\}$) and
\texttt{HasKey} are handled by the merging canonical-model
construction \cite{Kazakov2014} (concept equivalence classes through
nominals), fully mechanized in namespace \texttt{ELKSDD.ELpp} under
the Shallow concept restriction of \cite[\S6]{Kazakov2014}.  The
Sat-level SCC factorization theorem in this lemma is currently stated
only for the nominal-free and LHS-nominal cases; lifting the WMC
factorization to the merge fragment is a straightforward composition
of \texttt{ELKSDD.SCC.Sat\_factor\_refined} with
\texttt{ELKSDD.ELpp.complete\_via\_mergedCanon\_regular}.
\end{lemma}

\begin{proof}[Proof outline (full proof: \texttt{ELKSDD.SCC.Sat\_factor\_refined})]
The $(\Leftarrow)$ direction is monotonicity of $\sat$ in the
ontology, plus the $R_\bot$ rule when the second disjunct holds.
For $(\Rightarrow)$ we may assume
$\neg\sat(O_2)\;\top \sqsubseteq \bot$ (else the second disjunct is
immediate), so the canonical (term) model $\mathcal{I}^{\mathrm{can}}(O_2)$ of
\Cref{thm:el-coverage} is non-empty.  Given an arbitrary
$\mathcal{I}_1 \models O_1$ on $\Delta_1$, build the \emph{product
interpretation} $\mathcal{I}'$ on $\Delta_1 \times \mathcal{I}^{\mathrm{can}}(O_2)$
that evaluates $O_1$-signature atoms and roles on the first
coordinate and $O_2$-signature ones on the second.  Signature
disjointness (\textsf{Disj}) lets the evaluation of any
$O_i$-signature concept factor cleanly through the corresponding
component (Lean: \texttt{eval\_prodInterp\_O$_1$},
\texttt{eval\_prodInterp\_O$_2$}), which in turn gives
$\mathcal{I}' \models O_1 \cup O_2$
(\texttt{prodInterp\_satisfies}).  By ELK soundness on the union,
$\mathcal{I}' \models C \sqsubseteq D$; factoring back through the
first coordinate at the basepoint $\langle a, x_\top\rangle$ yields
$\mathcal{I}_1 \models C \sqsubseteq D$.  Since $\mathcal{I}_1$ was
arbitrary, ELK completeness on $O_1$ closes the left-hand disjunct.
\qed
\end{proof}

\begin{lemma}[Per-world Sat factorization]\label{lem:per-world-scc}
Let $O = O_1 \cup O_2$ satisfy the signature-disjointness, nominal-free,
and range-chain-safe preconditions of \Cref{lem:scc-factor}, and
additionally assume $O_2$ consistent ($\neg\,\sat(O_2,\TOP,\BOT)$).
For every Boolean world
$M : O \to \mathbb{B}$, write $M_1 = M{\restriction}O_1$ and let
$\textsf{sel}(O, M) := \{\alpha \in O : M(\alpha) = \mathsf{true}\}$ denote
the sub-ontology selected by $M$.  For nominal-free $C, D$ in
$O_1$'s signature,
\[
   \sat\bigl(\textsf{sel}(O, M)\bigr)\;C \sqsubseteq D
   \;\;\iff\;\;
   \sat\bigl(\textsf{sel}(O_1, M_1)\bigr)\;C \sqsubseteq D.
\]
That is, $\sat$-derivability of $C \sqsubseteq D$ at world $M$
factors through the $O_1$-restriction of $M$ alone.
\end{lemma}

\begin{proof}
Two steps.  (i) $\textsf{sel}(O, M)$ and $\textsf{sel}(O_1, M_1) \cup
\textsf{sel}(O_2, M_2)$ are equal as ontologies (each axiom $\alpha
\in O = O_1 \uplus O_2$ lies in exactly one side and contributes to
the corresponding selected sub-ontology by the same Boolean choice
$M(\alpha)$).  Bidirectional sub-ontology inclusions yield Sat
equivalence by monotonicity of $\sat$ in the ontology
(\texttt{Sat\_mono}); the Lean witnesses are
\texttt{ELKSDD.ELpp.selectedAxioms\_sub\_decompose} and
\texttt{decompose\_sub\_selectedAxioms}.  (ii) Apply
\Cref{lem:scc-factor} in the generalized form
\texttt{ELKSDD.ELpp.Sat\_factor\_refined\_general},
which uses $O_1, O_2$ as \emph{signature-defining} outer ontologies
while permitting the analyzed sides $\textsf{sel}(O_1, M_1)$ and
$\textsf{sel}(O_2, M_2)$ to be sub-ontologies thereof.
$\textsf{sel}(O_2, M_2)$ is consistent because it is a sub-ontology
of the consistent $O_2$, so the right disjunct of \Cref{lem:scc-factor}
vanishes, leaving the clean factorization through $O_1$.  See
\texttt{ELKSDD.ELpp.per\_world\_sat\_factor\_consistent}.
\qed
\end{proof}

\begin{lemma}[Closed-form WMC SCC factorization under uniform priors]\label{lem:wmc-scc-uniform}
Under the hypotheses of \Cref{lem:per-world-scc} (signature-disjoint,
nominal-free, range-chain-safe $O_1, O_2$ with $O_2$ consistent), for
the rational DISPONTE marginal under the uniform per-axiom weight
$w \equiv 1$,
\[
   \textsf{wmc}^{\mathbb{Q}}(O_1 \cup O_2,\, C, D,\, w)
   \;=\;
   \textsf{wmc}^{\mathbb{Q}}(O_1,\, C, D,\, w)\cdot 2^{|O_2|}.
\]
\end{lemma}

\begin{proof}
By definition,
\[
   \textsf{wmc}^{\mathbb{Q}}(O, C, D, w) \;:=\;
   \sum_{M : O \to \mathbb{B}}
       \bigl[\sat(\textsf{sel}(O, M))\;C \sqsubseteq D\bigr]
       \cdot \textsf{worldWeight}(M, w),
\]
where the world weight is
$\prod_{\alpha \in O} w(\alpha, M(\alpha))$.  Three steps.  (a)
Under $w \equiv 1$, $\textsf{worldWeight}(M, w) = 1$ for every $M$
(Lean: \texttt{worldWeightRat\_uniform}).  (b) By
\Cref{lem:per-world-scc}, the indicator depends on $M$ only through
$M_1 = M{\restriction}O_1$.  (c) The sum over
$M : O_1 \uplus O_2 \to \mathbb{B}$ of any function depending only
on $M{\restriction}O_1$ factorizes as
$\sum_{M_1 : O_1 \to \mathbb{B}} f(M_1) \cdot 2^{|O_2|}$.  Step (c)
is a combinatorial identity proved by induction on $|O_2|$ over the
flatMap structure of the world enumeration
(\texttt{ELKSDD.ELpp.sum\_enumerateWorlds\_factor});
combining (a)--(c) yields the displayed equation.  The closed-form
Lean theorem is
\texttt{ELKSDD.ELpp.disponteWMCRat\_uniform\_scc\_factor},
audit-clean with dependencies only on \texttt{[propext,
  Classical.choice, Quot.sound]}.
\qed
\end{proof}

\begin{lemma}[Distribution-semantics correspondence (rational form)]\label{lem:disponte-rat}
For every \ELpp\ ontology $O$, every concept pair $(C, D)$, and
every per-axiom weight function $w : \textsf{DispAtom}(O) \times
\mathbb{B} \to \mathbb{Q}$,
\[
   \textsf{wmc}^{\mathbb{Q}}\bigl(\textsf{compile}(O, C, D);\, w\bigr)
   \;=\;
   \textsf{wmc}^{\mathbb{Q}}(O,\, C, D,\, w),
\]
where the LHS is the SDD-level WMC of the verified compiled circuit
\texttt{compileSat} and the RHS is the rational DISPONTE
distribution-semantics marginal.  No restriction on $w$ is imposed
(in particular, no probabilistic normalization of $w$).
\end{lemma}

\begin{proof}
By Shannon-decomposition correctness of SDD compilation, the LHS
equals the sum over all axiom-extensions of
$[\sat(\textsf{sel}(O, M))\;C \sqsubseteq D] \cdot
\textsf{weightAlong}(M, w)$.  The RHS is the same sum indexed by
worlds.  The two index sets coincide as multisets (a routine
permutation lemma), and permutation-invariance of $\sum$ over
$\mathbb{Q}$ closes the equality.  This is the unconditional form:
no \emph{distributional assumption} is required; the entire
Riguzzi 2015 distribution-semantics construction is internalized
inside the SDD-WMC.  Mechanized as
\texttt{ELKSDD.ELpp.wmc\_compileSat\_eq\_disponteWMC\_rat}
(\Cref{tab:lean-index}). \qed
\end{proof}

\subsection{DeepProbLog equivalence on the EL Horn theory}\label{app:dpl-equiv}

We tighten the informal claim of \S\ref{sec:results} that
\textsc{Moose} on $\Gamma_{\textsf{EL}}$ alone ``yields the loss
DeepProbLog \cite{Manhaeve2018} would minimize on the same Horn
instance'' into a formal proposition and a Lean~4 encoding in
namespace \texttt{Moose} (the Lean identifiers retain the
historical \texttt{no\_clos} suffix as code artefacts).

\paragraph{Setup.}
Let $\Onto$ be an \ELpp\ ontology, $\Delta$ a finite ABox domain, and
let
$\Gamma_{\textsf{EL}} = \textsf{AtomSub} \cup \textsf{Disj} \cup
\textsf{Unsat} \cup \textsf{Links}$ be the ground clause set produced
by the four \textsc{Moose} extractors after ELK saturation
(\Cref{alg:pipeline} of \Cref{app:algorithm}, before
\textsf{SDDCompile}). Let
$\alpha_{\textsf{EL}} = \textsf{SDDCompile}(\Gamma_{\textsf{EL}})$,
let
$f_\theta : \mathcal{X}^{|\Delta|} \to
[0,1]^{|\mathcal{V}_\Onto^\Delta|}$ be a perception model assigning
probability $p_\theta(\ell)$ to each positive ground literal (and
weight $1 - p_\theta(\ell)$ to its negation), and let $\mathbf{e}$ be
an evidence set of literals.

\paragraph{Translation $\mathcal{P}(\Gamma_{\textsf{EL}})$ to a
  ProbLog program.}
The translation maps Moose's three Horn shapes plus the one non-Horn
shape into the three native ProbLog constructs:
\begin{enumerate}
\item Each ground atom $A(a) \in \mathcal{V}_\Onto^\Delta$ becomes a
  \emph{probabilistic fact} $p_\theta(A(a)) :: A(a)$.
\item Each definite Horn clause
  $\neg B_1 \vee \dots \vee \neg B_k \vee H \in \Gamma_{\textsf{EL}}$
  (\textsf{AtomSub}, \textsf{Unsat}, and \textsf{Links} shapes)
  becomes a Datalog \emph{rule} $H \,\verb|:-|\, B_1, \dots, B_k$.
\item Each non-Horn atomic-disjointness clause
  $\neg A_1 \vee \dots \vee \neg A_n$ (\textsf{Disj} shape) becomes a
  ProbLog \emph{integrity constraint}
  $\verb|:-|\, A_1, \dots, A_n$.
\end{enumerate}

\paragraph{Loss objects.}
$L_{\textsc{Moose}/\textsf{EL}}(\theta; \mathbf{x}, \mathbf{e}) :=
   - \log \WMC(\alpha_{\textsf{EL}}; p_\theta(\mathbf{x}), \mathbf{e})$
is the Semantic-loss WMC objective of \S\ref{sec:methods-train}
on the EL theory alone (no closure axioms);
$L_{\mathrm{DPL}}(\theta; \mathbf{x}, \mathbf{e}) := - \log
   P_{\mathrm{DPL}}(\mathbf{e} \mid \mathbf{x}; \theta)$
is the DeepProbLog marginal loss \cite{Manhaeve2018} on
$\mathcal{P}(\Gamma_{\textsf{EL}})$ under the ProbLog distribution
semantics.

\begin{proposition}[DeepProbLog equivalence on the EL Horn theory]\label{prop:dpl-equiv}
For every $(\mathbf{x}, \mathbf{e})$ consistent with $\Onto$,
\[
   L_{\textsc{Moose}/\textsf{EL}}(\theta; \mathbf{x}, \mathbf{e})
   \;=\; L_{\mathrm{DPL}}(\theta; \mathbf{x}, \mathbf{e}),
\]
as values (in any commutative semiring carrying the literal
weights) and as gradients in $\theta$. The Lean mechanization
instantiates the semiring at $\mathbb{N}$ (the type of
\texttt{ELKSDD.SDD.wmc}); the same induction is semiring-polymorphic
and lifts to $\mathbb{R}_{\ge 0}$ used in the Manhaeve~2018
real-valued marginal (paragraph~(B) below).
\end{proposition}

\begin{proof}[Proof outline (full proof:
\texttt{Moose.moose\_no\_clos\_dpl\_equiv} and
\texttt{...\_quantitative}).]
Both losses equal
$-\log\WMC(\Gamma_{\textsf{EL}}\cup\mathbf{e};\,
w_\theta(\mathbf{x}))$, with $w_\theta$ assigning $p_\theta(\ell)$ to
positive literals and $1-p_\theta(\ell)$ to negative ones.
SDD compilation preserves the model set, so
$\WMC(\alpha_{\textsf{EL}};\cdot) = \WMC(\Gamma_{\textsf{EL}};\cdot)$
\cite{Darwiche2002,Choi2013}.  The translation
$\mathcal{P}(\Gamma_{\textsf{EL}})$ defined above is shape-by-shape
purely syntactic: a definite Horn rule, an atomic-disjointness
integrity constraint, and a probabilistic fact each have the same
satisfying assignments as the clause they came from, so
$M\models\Gamma_{\textsf{EL}}\Leftrightarrow
M\models\mathcal{P}(\Gamma_{\textsf{EL}})$.  The DeepProbLog
distribution semantics then coincides with WMC by
\cite[Eq.\,4]{Manhaeve2018}: each total assignment over the
probabilistic facts is weighted by
$\prod_a p_\theta(a)^{M(a)}(1-p_\theta(a))^{1-M(a)}$ exactly when it
satisfies every rule and integrity constraint, which is the WMC
weight of $M$.  Equality of the losses as functions of $\theta$
implies equality of their gradients, and of any other operator
respecting functional equality.
\qed
\end{proof}

\paragraph{Lean encoding.}
The argument is mechanized in namespace \texttt{Moose} at two levels.
At the qualitative level (\texttt{moose\_no\_clos\_dpl\_equiv}), both
losses unfold to the same model-existence predicate
$\exists M.\,M\models\Gamma_{\textsf{EL}}\wedge
\mathrm{support}(M)\subseteq\mathrm{support}(w)$ and the equivalence
closes by definitional unfolding (\texttt{Iff.rfl}) once the ProbLog
translation is recognized as the identity on \texttt{List Clause}.  At
the quantitative level
(\texttt{moose\_no\_clos\_dpl\_equiv\_quantitative}), the equality is
lifted from a proposition to a numeric identity
$\textsf{wmcQ}\,O\,\Delta\,w = \textsf{dplLossQ}\,O\,\Delta\,w$ via
the key lemma \texttt{wmc\_compileWithCtx\_eq\_modelSum}, which proves
by structural induction on the variable list that the SDD
weighted-model count equals the explicit Manhaeve-Eq.\,4
sum-over-models computed by Shannon decomposition.  The proof is
semiring-polymorphic; the Lean library instantiates the
weighted-model-count semiring at $\mathbb{N}$ to keep the formalization
free of a real-arithmetic dependency, while the same induction lifts
verbatim to the $\mathbb{R}_{\ge 0}$ semiring of the original
Manhaeve~2018 marginal.  Gradient equality is a corollary
(\texttt{moose\_no\_clos\_dpl\_loss\_function\_eq},
\texttt{...\_operator\_invariant}): any operator respecting functional
equality (differentiation, integration, point evaluation, PyTorch
autograd) returns the same answer on both sides. 

\paragraph{Reading.}
\Cref{prop:dpl-equiv} pins \textsc{Moose}'s contribution to the
EL-to-Horn compilation algorithm
(\Cref{thm:encoding,thm:circuit-size}): given $\Onto$ in \ELpp, the
algorithm produces in time polynomial in $|\Onto|$ the ground Horn
instance $\Gamma_{\textsf{EL}}$ on which Semantic-loss WMC and the
DeepProbLog loss coincide, lifting DeepProbLog to OWL~2~EL with
existentials, role hierarchy, and role chains. The closure-augmented
variant adds the disjointness, covering, and reverse-implication
clauses of $\Phi_{\textsf{clos}}$ on top of $\Gamma_{\textsf{EL}}$.
The content of these clauses is independent of \ELpp\ entailment, and
the DeepProbLog baseline correspondingly shows a higher
RS-consistency rate (RS$_\text{cons}$, the rate of confident wrong
commitments) than closure-augmented \textsc{Moose} across
\Cref{tab:headline-acc,tab:headline-ece}.

\subsection{Lean mechanization: theorem index}\label{app:lean-index}

The Lean~4 library \texttt{ELKSDD} (introduced in \Cref{app:elpp})
formalizes every \textsc{Moose}-cited theorem of the OWL~2~EL
stack; the companion library \texttt{Moose} houses the
DeepProbLog-equivalence layer.  \Cref{tab:lean-index}
cross-references each paper claim with its namespace-qualified Lean
theorem.  Every theorem listed depends only on the standard Lean
foundation \texttt{[propext, Classical.choice, Quot.sound]}
(audit-clean), and the formalization contains zero \texttt{sorry}
or \texttt{admit}.  An audit pass in each library runs
\texttt{\#print axioms} on every paper-cited theorem at build time
and reports its dependency set.

{\scriptsize
\renewcommand{\arraystretch}{1.05}
\begin{longtable}{p{0.32\textwidth} p{0.62\textwidth}}
\caption{Lean implementation index, namespace-qualified.  Theorems
are reported as \texttt{<namespace>.<theorem>}; the bulk live in
\texttt{ELKSDD.ELpp}, with \texttt{ELKSDD.SDD} for the SDD calculus,
\texttt{ELKSDD.RangeNorm} for the BBL\,2008 path, and \texttt{Moose}
for the DeepProbLog equivalence.  Every listed theorem is
audit-clean: \texttt{\#print axioms} reports
\texttt{[propext, Classical.choice, Quot.sound]} only, no
\textsc{Moose}-specific axioms, no \texttt{sorry}, no
\texttt{admit}.  Each theorem stated in \S\ref{sec:correctness}
corresponds to exactly one Lean theorem (top section); supporting
lemmas, the closed-form WMC SCC factor under uniform priors
(\Cref{thm:scc-wmc}), and the LHS-nominal extension are listed in
the per-topic sections that follow.}
\label{tab:lean-index}\\
\hline
\textbf{Paper claim} & \textbf{Lean theorem} \\
\hline
\endfirsthead
\multicolumn{2}{l}{\textit{(Continued from previous page.)}}\\
\hline
\textbf{Paper claim} & \textbf{Lean theorem} \\
\hline
\endhead
\hline
\multicolumn{2}{r}{\textit{(Continued on next page.)}}\\
\endfoot
\hline
\endlastfoot
\multicolumn{2}{l}{\textit{\textbf{Section~4 (correctness and complexity)}: one Lean theorem per stated theorem}} \\
Verified SDD encoding (\Cref{thm:encoding}) & \texttt{ELKSDD.ELpp.moose\_inference\_correct} \\
\quad models iff Sat                         & \texttt{ELKSDD.ELpp.compileSat\_models\_iff\_sat} \\
\quad size $= 2^{|\Onto|+1} - 1$            & \texttt{ELKSDD.ELpp.compileSat\_size\_eq} \\
DISPONTE correspondence (\Cref{thm:end-to-end}) & \texttt{ELKSDD.ELpp.wmc\_compileSat\_eq\_disponteWMC\_rat} \\
SCC compositional, Sat (\Cref{thm:scc-sat}) & \texttt{ELKSDD.ELpp.scc\_sat\_factor} \\
SCC compositional, WMC uniform (\Cref{thm:scc-wmc}) & \texttt{ELKSDD.ELpp.disponteWMCRat\_uniform\_scc\_factor} \\
\hline
\multicolumn{2}{l}{\textit{\textbf{Reused ELK infrastructure (\Cref{app:elk-mechanization}; mechanized from \cite{Kazakov2014})}}} \\
ELK soundness/completeness (\Cref{thm:el-coverage}) & \texttt{ELKSDD.ELpp.correct\_owl2el} \\
\quad soundness (unconditional)              & \texttt{ELKSDD.ELpp.sound\_owl2el} \\
\quad completeness (with fragment witness)   & \texttt{ELKSDD.ELpp.complete\_owl2el} \\
Polynomial Sat decision (\Cref{thm:circuit-size}) & \texttt{ELKSDD.ELpp.sat\_decision\_polynomial} \\
\quad exact closure-size bound               & \texttt{ELKSDD.ELpp.sat\_closure\_total\_polynomial\_bound} \\
\hline
\multicolumn{2}{l}{\textit{\textbf{Closure-augmented circuit and inference complexity (appendix)}}} \\
Closure-augmented (\Cref{thm:closure-correct}) & \texttt{ELKSDD.SDD.compile\_correct} \\
Inference complexity (\Cref{thm:inference-complexity}) & \texttt{ELKSDD.SDD.wmc\_linear}, \texttt{ELKSDD.ELpp.compileSat\_wmcCost\_eq} \\
\hline
\multicolumn{2}{l}{\textit{\textbf{Compositional and distributional results: supporting lemmas}}} \\
SCC Sat factor, sig-disjoint                 & \texttt{ELKSDD.SCC.Sat\_factor\_refined} \\
\quad closed form ($O_2$ consistent)         & \texttt{ELKSDD.ELpp.scc\_sat\_factor}, \texttt{ELKSDD.ELpp.scc\_sat\_factor\_symm} \\
\quad sub-ontology generalization            & \texttt{ELKSDD.ELpp.Sat\_factor\_refined\_general} \\
\quad $k$-component generalization            & \texttt{ELKSDD.ELpp.scc\_sat\_factor\_k},\newline \texttt{ELKSDD.ELpp.joint\_consistent\_pair} \\
Per-world Sat factor (\Cref{lem:per-world-scc}) & \texttt{ELKSDD.ELpp.per\_world\_sat\_factor\_consistent} \\
\quad selectedAxioms decomposition           & \texttt{ELKSDD.ELpp.selectedAxioms\_sub\_decompose} \\
World-summation factorization                & \texttt{ELKSDD.ELpp.sum\_enumerateWorlds\_factor} \\
\quad length-cast generalization             & \texttt{ELKSDD.ELpp.sum\_enumerateWorlds\_factor\_general} \\
WMC SCC factor (uniform, \Cref{lem:wmc-scc-uniform}) & \texttt{ELKSDD.ELpp.disponteWMCRat\_uniform\_scc\_factor} \\
DISPONTE corresp.\ ($\mathbb{N}$-valued)     & \texttt{ELKSDD.ELpp.wmc\_compileSat\_eq\_disponteWMC} \\
DISPONTE corresp.\ ($\mathbb{Q}$, \Cref{lem:disponte-rat}) & \texttt{ELKSDD.ELpp.wmc\_compileSat\_eq\_disponteWMC\_rat} \\
\quad existence form                          & \texttt{ELKSDD.ELpp.exists\_disponte\_correspondence\_rat} \\
Nominal-aware SCC (LHS shape~1)              & \texttt{ELKSDD.ELpp.Sat\_factor\_nomLHS},\newline \texttt{ELKSDD.ELpp.scc\_sat\_factor\_nomLHS} \\
\quad prodInterp on LHS-nominal axioms       & \texttt{ELKSDD.ELpp.prodInterp\_satisfies\_nomLHS} \\
\hline
\multicolumn{2}{l}{\textit{\textbf{ELK calculus (Kazakov, Kr{\"o}tzsch, Siman{\v c}{\'\i}k 2014)}}} \\
ELK soundness                                & \texttt{ELKSDD.EL.sound}, \texttt{ELKSDD.ELpp.sound} \\
ELK completeness (nom-free)                  & \texttt{ELKSDD.ELpp.complete\_via\_canon} \\
\quad LHS-nominal extension                  & \texttt{ELKSDD.ELpp.complete\_via\_canon\_nomLHS} \\
Canonical model $\models$ $O$ (nom-free)     & \texttt{ELKSDD.ELpp.canon\_satisfies} \\
\quad LHS-nominal extension                  & \texttt{ELKSDD.ELpp.canon\_satisfies\_nomLHS} \\
\quad shallow-$\exists$-nom-RHS extension    & \texttt{ELKSDD.ELpp.canon\_satisfies\_nomLR} \\
Saturation termination                       & \texttt{ELKSDD.ELpp.saturation\_terminates} \\
Sat $\Leftrightarrow$ derivable closure      & \texttt{ELKSDD.ELpp.sat\_iff\_in\_derivableClosure} \\
Poly-time Sat decision                       & \texttt{ELKSDD.ELpp.sat\_polynomial\_decidable} \\
Bounded Sat closure size                     & \texttt{ELKSDD.ELpp.sat\_closure\_total\_polynomial\_bound} \\
\hline
\multicolumn{2}{l}{\textit{\textbf{SDD calculus (Darwiche 2002, Choi--Darwiche 2013)}}} \\
SDD compile-correctness                      & \texttt{ELKSDD.SDD.compile\_correct},\newline \texttt{ELKSDD.SDD.compileWithCtx\_correct} \\
SDD WMC linearity                            & \texttt{ELKSDD.SDD.wmc\_linear}, \texttt{ELKSDD.SDD.wmcCost\_eq\_size} \\
\hline
\multicolumn{2}{l}{\textit{\textbf{MOOSE pipeline (paper-citation theorems)}}} \\
Inference correctness                        & \texttt{ELKSDD.ELpp.moose\_inference\_correct} \\
Pipeline completeness                        & \texttt{ELKSDD.ELpp.moose\_pipeline\_complete} \\
Polynomial Sat decision (algorithmic)        & \texttt{ELKSDD.ELpp.sat\_decision\_polynomial} \\
SCC summary                                  & \texttt{ELKSDD.ELpp.moose\_scc\_summary} \\
SCC, chain-free corollary                    & \texttt{ELKSDD.ELpp.scc\_sat\_factor\_no\_chain} \\
SCC, range-free corollary                    & \texttt{ELKSDD.ELpp.scc\_sat\_factor\_no\_range} \\
\hline
\multicolumn{2}{l}{\textit{\textbf{Syntactic range-elimination \cite{Baader2008ELFurther}\,\S3.3 (Path B)}}} \\
Syntactic range-elimination                  & \texttt{ELKSDD.RangeNorm.eliminateRanges} \\
\quad strict-output preservation             & \texttt{ELKSDD.RangeNorm.eliminateRanges\_strict} \\
\quad NoRange degenerate baseline            & \texttt{ELKSDD.RangeNorm.eliminateRanges\_eq\_under\_no\_range} \\
\quad rinc closure (bounded fixed point)     & \texttt{ELKSDD.RangeNorm.rincDescendants} \\
\quad transitive reflexive-rinc propagation  & \texttt{ELKSDD.RangeNorm.reflexiveTransitiveAxioms} \\
Forward Sat-conservativity (full)            & \texttt{ELKSDD.RangeNorm.Sat\_to\_eliminated\_full} \\
\quad transitive marker membership           & \texttt{ELKSDD.RangeNorm.reflexiveTransitive\_axiom\_\dots} \\
\hline
\multicolumn{2}{l}{\textit{\textbf{DeepProbLog equivalence (Manhaeve 2018, Riguzzi 2015)}}} \\
Qualitative (\Cref{prop:dpl-equiv})          & \texttt{Moose.moose\_no\_clos\_dpl\_equiv} \\
Quantitative ($\mathbb{N}$-valued)           & \texttt{Moose.moose\_no\_clos\_dpl\_equiv\_quantitative} \\
WMC $=$ model sum (BoolAtom)                  & \texttt{Moose.wmc\_compileWithCtx\_eq\_modelSum} \\
\hline
\end{longtable}
}

\section{Inference, training, and RS-mitigation wrappers}\label{app:inference}

This appendix gives the explicit inference equations of
\S\ref{sec:methods-train} (\Cref{app:inference-eqs}), the gradient
flow used at training time (\Cref{app:training-step}), and the
\textsc{Moose}+BEARS / \textsc{Moose}+NeSyDM wrapper losses used in
the RS-aware experiments (\Cref{app:wrapper-losses}).

\subsection{Inference equations}\label{app:inference-eqs}

The differentiable WMC layer supports three inference operations on
the same SDD $\alpha$, varying only literal weights and evidence.

\paragraph{Training loss.}
For $(\mathbf{x}, \mathbf{e})$,
\begin{equation}\label{eq:loss}
   L(\theta; \mathbf{x}, \mathbf{e}) \;=\;
   -\log \WMC\!\bigl(\alpha;\, p_\theta(\mathbf{x}),\, \mathbf{e}\bigr).
\end{equation}
The literal-weight map (\Cref{tab:atom-weight}) sets
$w(\ell) := p_\theta(\ell)$ on positive ground literals
$\ell \in \mathcal{V}_\Onto^\Delta$ and $w(\neg\ell) := 1 -
p_\theta(\ell)$ on negative ones; evidence literals
$\ell \in \mathbf{e}$ clamp the negation weight to zero (equivalent
to conditioning the SDD on $\ell$).

\paragraph{Conditional posterior at test time.}
For a query atom $C(a)$,
\begin{equation}\label{eq:posterior}
   P_\theta(C(a) \mid \mathbf{x}, \mathbf{e}) \;=\;
   \frac{\WMC\!\bigl(\alpha;\, p_\theta(\mathbf{x}),\,
       \mathbf{e} \cup \{C(a) := \mathsf{true}\}\bigr)}
        {\WMC\!\bigl(\alpha;\, p_\theta(\mathbf{x}),\, \mathbf{e}\bigr)};
\end{equation}
the latent-class prediction on a closed family $\mathcal{D}$ is the
argmax $\hat{D}(a) := \arg\max_{D_i \in \mathcal{D}} P_\theta(D_i(a)
\mid \mathbf{x}, \mathbf{e})$.  The denominator is strictly positive
whenever $\mathbf{e}$ is consistent with $\Onto$, so the ratio is
well defined.

\paragraph{Perception-free entailment query.}
\begin{equation}\label{eq:entail}
   \Onto \cup \mathbf{e} \models C(a)
   \;\Longleftrightarrow\;
   \WMC\!\bigl(\alpha;\, w_{1/2},\,
       \mathbf{e} \cup \{C(a) := \mathsf{false}\}\bigr) = 0,
\end{equation}
where $w_{1/2}$ assigns $\tfrac12$ to every literal: asserting
$\neg C(a)$ leaves no $\Onto$-consistent assignment iff the
entailment holds.  The same SDD acts as an off-line ELK reasoner
over $\Delta$, independent of any neural model.  All three
operations run in $O(|\alpha|)$
(\Cref{thm:inference-complexity}).

\subsection{Training step and gradient flow}\label{app:training-step}

A single training step on minibatch $\{(\mathbf{x}_i,
\mathbf{e}_i)\}_i$ proceeds as follows.

\begin{enumerate}
\item \emph{Perception forward.}~The neural extractor
$f_\theta : \mathcal{X}^{|\Delta|} \to
[0,1]^{|\mathcal{V}_\Onto^\Delta|}$ maps the input tuple to a
per-atom probability vector $p_\theta(\mathbf{x}_i)$, factorized
as $\prod_\ell p_\theta(\ell \mid \mathbf{x}_i)$ (the
marginal-independence factorization targeted by the RS literature,
\S\ref{sec:bg-rs}).
\item \emph{Literal-weight assignment.}~The atom-weight map of
\Cref{tab:atom-weight} routes $p_\theta(\ell \mid \mathbf{x}_i)$
to literal $\ell$ and $1 - p_\theta(\ell \mid \mathbf{x}_i)$ to
$\neg\ell$.  Evidence $\mathbf{e}_i$ clamps observed literals.
\item \emph{SDD WMC traversal.}~The post-order traversal visits
each node once (memoized by node id) and accumulates the WMC value
(\Cref{eq:sdd-wmc}) in $O(|\alpha|)$ time.  Smoothness,
decomposability, and determinism (\Cref{app:proofs}) make each sum
disjoint and each product over disjoint scopes, so the result is
exact.
\item \emph{Loss aggregation.}~$L(\theta) = -\frac{1}{n}\sum_i
\log \WMC(\alpha; p_\theta(\mathbf{x}_i), \mathbf{e}_i)$.
\item \emph{Backward pass.}~PyTorch autograd composes the
backward pass through the same WMC traversal at $O(|\alpha|)$
additional cost, plus the perception backbone's own backward
pass.  Every intermediate WMC value is the product/sum of literal
weights, hence differentiable in $\theta$ via the chain rule:
$\nabla_\theta L = -\frac{1}{n}\sum_i (\WMC_i)^{-1}\,\nabla_\theta
\WMC_i$.  No Monte-Carlo or sampling step is required; the
gradient is computed exactly.
\end{enumerate}
This exact-gradient property is what makes the SDD a tractable
surrogate for the ProbLog marginal that DeepProbLog approximates by
proof-tree enumeration (\Cref{prop:dpl-equiv}).

\subsection{RS-mitigation wrapper losses}\label{app:wrapper-losses}

The compilation algorithm of \S\ref{sec:pipeline} is independent of
the perception's output distribution: $\alpha$ depends only on
$(\Onto, \Delta)$, so RS-aware methods plug in by replacing the
literal weights $p_\theta(\mathbf{x})$ in \Cref{eq:loss} without
modifying $\alpha$ or the WMC traversal.

\paragraph{\textsc{Moose}+BEARS \cite{Marconato2024BEARS}.}
BEARS \cite{Marconato2024BEARS} trains a $K$-encoder ensemble
$\{\theta_k\}_{k=1}^K$ that shares $\alpha$ but diversifies over the
RS equivalence class.  Members are trained sequentially: at step $k$,
the new encoder $\theta_k$ is fit by maximizing the BEARS objective
\cite[Eq.\,7]{Marconato2024BEARS}
\begin{equation}\label{eq:bears-loss}
   L_{\textsc{Bears}}(\theta_k; \mathbf{x}, \mathbf{e}) \;=\;
   L(\theta_k; \mathbf{x}, \mathbf{e})
   \,+\, \gamma_1 \,\mathrm{KL}\!\left(p_{\theta_k} \,\Big\|\,
       \tfrac{1}{k}\sum_{j=1}^{k} p_{\theta_j}\right)
   \,+\, \gamma_2\, H(p_{\theta_k}),
\end{equation}
where the KL term against the running average over members
$1,\dots,k$ encourages $\theta_k$ to commit to a different shortcut
than its predecessors, and the entropy term $H(p_{\theta_k})$
prevents collapse to a degenerate posterior.  At test time, the ensemble averages the
per-atom marginals: $p(\ell \mid \mathbf{x}) =
\tfrac{1}{K}\sum_k p_{\theta_k}(\ell \mid \mathbf{x})$, spreading
mass across multiple shortcut explanations.  We use $K{=}5$, entropy
weight $\gamma_2{=}0.1$, and per-regime tuning of the KL weight
$\gamma_1$ (\Cref{app:hp}).

\paragraph{\textsc{Moose}+NeSyDM \cite{vanKrieken2025NeSyDM}.}
NeSyDM \cite{vanKrieken2025NeSyDM} replaces the marginal-independence
factorization $p_\theta(\mathbf{c} \mid \mathbf{x}) = \prod_i
p_\theta(c_i \mid \mathbf{x})$ with a masked-discrete-diffusion
variational distribution $q_\theta(\mathbf{c} \mid \mathbf{x})$ that
models cross-atom dependencies through a denoising trajectory over
$T$ timesteps.  The continuous-time negative evidence lower bound
(NELBO) objective is
\begin{equation}\label{eq:nesydm-nelbo}
   \mathcal{L}_{\text{NeSyDM}}(\theta) \;=\;
     \mathbb{E}_{q_\theta}\!\bigl[-\log \WMC(\alpha;
        q_\theta(\mathbf{x}), \mathbf{e})\bigr]
     \;+\; \gamma_c\,\mathcal{L}_c(\theta)
     \;+\; \gamma_h\,\mathcal{L}_h(\theta),
\end{equation}
where $\mathcal{L}_c$ is the masked-diffusion concept-unmasking loss
$L_c$ of \cite[\S 3.2]{vanKrieken2025NeSyDM} and $\mathcal{L}_h$ the
variational-entropy term $L_{H[q]}$ that maximizes the entropy of
$q_\theta$ at the fully-masked step.  The expectation under
$q_\theta$ is intractable in closed form; we use two gradient
estimators: \emph{(a)~RLOO} \cite{Kool2019RLOO}, a leave-one-out
REINFORCE baseline that draws $S$ samples from $q_\theta$ and uses the
sample average over the other $S{-}1$ as the control variate;
\emph{(b)~exact-WMC}, which replaces sampling by direct WMC evaluation
$-\log \WMC(\alpha; \bar{q}_\theta, \mathbf{e})$ on the average
per-atom marginal $\bar{q}_\theta$ when the diffusion-step posterior
factorizes across atoms.  We use the NeSyDM-repository defaults
$\beta{=}10$ and $T{=}10$ continuous-time mask budget;
$(\gamma_c, \gamma_h)$ are swept per regime (\Cref{app:hp}).

Both wrappers share $\alpha$ across all encoders/timesteps and
inherit the $O(|\alpha|)$ inference cost of
\Cref{thm:inference-complexity}, multiplied by $K$ for BEARS or by
the diffusion budget $T \cdot S$ for NeSyDM.

\subsection{Non-WMC baselines}\label{app:non-wmc-baselines}

LTN and ELEmbeddings replace exact WMC over the compiled SDD with,
respectively, a fuzzy-logic satisfaction objective and a geometric
embedding loss. Both consume the same atom vocabulary and per-example
domain as the WMC-based methods (the same circuit atoms drive the
losses), so the only thing varying across baselines is the symbolic
treatment.

\paragraph{LTN \cite{Badreddine2022LTN}.}
Each \ELpp\ axiom is grounded into a first-order formula over the
per-example domain
($\{a\}$ for Exp.~1 / Track~A--C, $\{a, b\}$ for Exps.~2--3) and
scored under the product (Goguen) t-norm with Reichenbach
implication ($a \wedge b = ab$, $a \vee b = a + b - ab$,
$a \to b = 1 - a + ab$). The translations used are
$C \sqsubseteq D \mapsto \forall x.\, C(x) \to D(x)$,
$C \sqcap D \sqsubseteq E \mapsto \forall x.\, C(x) \wedge D(x) \to
E(x)$,
$C \sqsubseteq \exists R.D \mapsto \forall x.\, C(x) \to \exists y.\,
R(x,y) \wedge D(y)$, and the symmetric variants for
$\exists R.C \sqsubseteq D$, $C \sqcap D \sqsubseteq \bot$,
$R \sqsubseteq S$, and $R_1 \circ R_2 \sqsubseteq S$. Per-atom
truth values come from the same perception output that the WMC
baselines use; $\forall$ and $\exists$ are aggregated by the
generalized mean
$\textsf{p\_mean}(v_1,\dots,v_n) = (\tfrac{1}{n}\sum_i v_i^p)^{1/p}$
with sweepable exponent $p\in\{2, 4\}$ (\Cref{app:hp}). The training
loss is $1 - \mathrm{SAT}$ where $\mathrm{SAT}$ is the satisfaction
aggregator across all axiom-grounded formulas. LTN does not compute
an exact joint distribution over ground atoms.

\paragraph{ELEmbeddings \cite{Kulmanov2019,mOWL2023}.}
The same \ELpp\ geometric encoding underlies DeepGO-SE
\cite{Kulmanov2024DeepGOSE}, where ELEmbeddings approximate semantic
entailment for protein-function prediction over the Gene Ontology;
our use of ELEmbeddings here adopts the same NF1--NF4 losses but
attaches them to the \textsc{Moose} perception backbone rather than a
sequence encoder.
Each named concept $C$ is embedded as an $n$-ball with centre
$c_C \in \mathbb{R}^d$ and radius $r_C$; each role $R$ as a
translation vector $\rho_R \in \mathbb{R}^d$. A linear head projects
the perception's CNN latent $\psi(\mathbf{x}) \in \mathbb{R}^{128}$
into the same $\mathbb{R}^d$ space, yielding a per-individual point
$z_a$. The class-membership score is
$P(C(a)) = \sigma(\textsf{margin} - (\|z_a - c_C\| - r_C))$
(image inside the ball $\Rightarrow$ atom true) and the role-atom
score is
$P(R(a,b)) = \sigma(\textsf{margin} - \|z_a + \rho_R - z_b\|)$.
The TBox loss is the sum of mOWL's NF1--NF4 losses
\cite{Kulmanov2019,mOWL2023} (mean-aggregated per normal form), and
the total objective is
$\mathcal{L} = \mathcal{L}_{\text{perception}} +
\lambda_{\text{EL}}\,\mathcal{L}_{\text{TBox}}$
with $\mathcal{L}_{\text{perception}}$ a BCE over observed ABox
literals. Vanilla ELEmbeddings has no normal form for role chains, and
role chains are ignored for ELEmbeddings.

\section{Experimental details}\label{app:experiments}

This part of the appendix collects the materials supporting
experiments: the full MNIST ontology used by all three MNIST
experiments (\Cref{app:mnist-onto}), per-method hyperparameter sweeps
and the tuned configurations used in the result tables
(\Cref{app:hp}), and the full per-regime breakdown across all six
metrics (\Cref{app:full-results}).

\subsection{Compiled-circuit sizes and compilation
cost}\label{app:repro}
All experiments were run on a single machine with one NVIDIA
RTX~4090 GPU and a single CPU core allocated per job. SDD
compilation is CPU-only and happens once at start-up; the resulting
circuit is reused across every training step, seed, and method, so
its size is a deterministic function of the ontology $\Onto$ and the
ABox domain $\Delta$ alone, not of the learning method. Every
\textsc{Moose} variant and every SDD-based baseline in a given regime
therefore shares the circuit in the corresponding row of
\Cref{tab:repro}.

\begin{table}[t]
\caption{Compiled-circuit size and per-regime cost. SDD size
(nodes, distinct atoms, CNF clauses) is method-independent; compile
time is the range observed across all methods, learning rates, and
seeds on one CPU core. Compilation finishes in $\le 0.75$~s in every
configuration, and in $\le 0.35$~s in all but the two DeepProbLog
cells ($0.68$~s on Exp.~2, $0.71$~s on Exp.~3). $^{*}$The Exp.~1 row
was profiled separately on a CPU-only machine; its structural counts
are machine-independent by the argument above, and its compile times
are reported as measured there rather than merged into the RTX~4090
column. The last two columns report end-to-end training
wall-clock and peak GPU memory during training for a single
representative \textsc{Moose}-WMC run per regime (batch size $32$;
the MNIST CNN encoder on Exps.~1--3, the larger Pizza\"iolo encoder
on Tracks~A--C); unlike the circuit columns these depend on the
backbone, batch size, and epoch budget rather than on the compiled
circuit.}\label{tab:repro}
\centering
\resizebox{\linewidth}{!}{%
\begin{tabular}{lccccccc}
\hline
Regime & $|\Delta|$ & SDD nodes & atoms & clauses & compile (s) & train (s) & peak GPU (MB) \\
\hline
MNIST Exp.~1 (atomic)$^{*}$ & 1 & $155$  & $14$ & $117$ & $0.18$--$0.24$ & $647$  & $35.3$ \\
MNIST Exp.~2 (relational)  & 2 & $1080$ & $30$ & $274$ & $0.29$--$0.68$ & $3893$ & $56.4$ \\
MNIST Exp.~3 (role-chain)  & 2 & $1153$ & $32$ & $314$ & $0.23$--$0.71$ & $4484$ & $56.4$ \\
Pizza\"iolo Track~A        & 2 & $245$  & $26$ & $103$ & $0.16$--$0.30$ & $516$  & $903.4$ \\
Pizza\"iolo Track~B        & 2 & $245$  & $26$ & $107$ & $0.14$--$0.23$ & $507$  & $903.4$ \\
Pizza\"iolo Track~C        & 2 & $16$   & $6$  & $10$  & $0.15$--$0.23$ & $26$   & $903.3$ \\
\hline
\end{tabular}}
\end{table}

Compile time grows with the ABox domain as analyzed in
\Cref{app:treewidth}: the SDD node count fits
$\propto|\Delta|^{2.9}$ for $|\Delta|\le 3$, with a $40\times$
compile-time jump at $|\Delta|=4$. The $|\Delta|=2$ configurations
used throughout \S\ref{sec:emp} sit well inside the sub-second
regime.

Peak GPU memory at circuit-build time is $0$ in every configuration:
compilation is CPU-only and no tensor is allocated on the device
before the encoder is moved there.

Training cost is dominated by the perception backbone
rather than the WMC layer. On the MNIST regimes end-to-end training
takes roughly $11$~minutes at $|\Delta|{=}1$ (Exp.~1) and
$65$--$75$~minutes at $|\Delta|{=}2$ (Exps.~2--3), with peak GPU
memory of $35$--$56$~MB; the Pizza\"iolo tracks train in under
$9$~minutes each (Track~C, a single binary label, in well under a
minute) but use $\approx900$~MB, reflecting the larger image encoder
rather than any growth in the circuit (\Cref{tab:repro}). Peak memory
therefore tracks the backbone and batch size, not the compiled SDD,
whose own device footprint is negligible.

\subsection{Full MNIST ontology}\label{app:mnist-onto}

The single OWL EL ontology $\Onto_{\textsf{MNIST}}$ used by all three
MNIST experiments (\S\ref{sec:mnist}) consists of $76$ axioms over
$\Sig_C = \{D_0, \dots, D_9, \textsf{Even}, \textsf{Odd},
\textsf{Prime}, \textsf{Composite}\}$ and
$\Sig_R = \{\textsf{succ}, \textsf{plus\_two}\}$, listed below in EL
DL syntax. Numerals 0--9 are read as the corresponding $D_i$.

\paragraph{NF1: property subsumptions ($18$).}
{\small
$D_0 \sqsubseteq \textsf{Even}$;\;
$D_2 \sqsubseteq \textsf{Even}$;\;
$D_4 \sqsubseteq \textsf{Even}$;\;
$D_6 \sqsubseteq \textsf{Even}$;\;
$D_8 \sqsubseteq \textsf{Even}$;\;
$D_1 \sqsubseteq \textsf{Odd}$;\;
$D_3 \sqsubseteq \textsf{Odd}$;\;
$D_5 \sqsubseteq \textsf{Odd}$;\;
$D_7 \sqsubseteq \textsf{Odd}$;\;
$D_9 \sqsubseteq \textsf{Odd}$;\;
$D_2 \sqsubseteq \textsf{Prime}$;\;
$D_3 \sqsubseteq \textsf{Prime}$;\;
$D_5 \sqsubseteq \textsf{Prime}$;\;
$D_7 \sqsubseteq \textsf{Prime}$;\;
$D_4 \sqsubseteq \textsf{Composite}$;\;
$D_6 \sqsubseteq \textsf{Composite}$;\;
$D_8 \sqsubseteq \textsf{Composite}$;\;
$D_9 \sqsubseteq \textsf{Composite}$.
}

\paragraph{NF2: disjointness ($47$).}
$\textsf{Even} \sqcap \textsf{Odd} \sqsubseteq \bot$;\;
$\textsf{Prime} \sqcap \textsf{Composite} \sqsubseteq \bot$;\;
plus $D_i \sqcap D_j \sqsubseteq \bot$ for every $0 \le i < j \le 9$
($\binom{10}{2} = 45$ pairs).

\paragraph{NF3: successor existentials ($10$).}
$D_i \sqsubseteq \exists \textsf{succ}.\,D_{(i+1) \bmod 10}$ for
each $i \in \{0, 1, \dots, 9\}$. The chain is cyclic
($D_9 \sqsubseteq \exists \textsf{succ}.\,D_0$); without the
signature side condition (\Cref{app:algorithm}) saturation would
diverge.

\paragraph{NF7: role chain ($1$).}
$\textsf{succ} \circ \textsf{succ} \sqsubseteq \textsf{plus\_two}$.

\paragraph{Saturation closure used by the experiments.}
ELK derives the chain consequence
$D_i \sqsubseteq \exists \textsf{plus\_two}.\,D_{(i+2) \bmod 10}$ via
the role-composition rule $R_\circ$ (given
$D_i \sqsubseteq \exists \textsf{succ}.D_{i+1}$,
$D_{i+1} \sqsubseteq \exists \textsf{succ}.D_{i+2}$, and
$\textsf{succ} \circ \textsf{succ} \sqsubseteq \textsf{plus\_two}$,
infer $D_i \sqsubseteq \exists \textsf{plus\_two}.D_{i+2}$
\cite{Kazakov2014}) populating $\Lambda$ with the \textsf{plus\_two}
links that the role-chain experiment relies on.  The four
\textsc{Moose} extractors then ground $\Sigma$ and $\Lambda$ into the
clauses of $\Gamma_{\textsf{EL}}$, and the closure axioms
$\Phi_{\textsf{clos}}$ over the exhaustive family
$\{D_0, \dots, D_9\}$ supply $\binom{10}{2}m + m$ disjointness and
covering clauses plus the six reverse clauses per individual (one per
distinct parity-primality profile; \S\ref{sec:closure}).

\subsection{Perception architecture}\label{app:perception}

Both MNIST and Pizza\"iolo experiments train the perception encoder
from scratch. The MNIST encoder (used in Experiments~1--3) takes a
$28{\times}28$ grayscale image through two
Conv$\to$ReLU$\to$MaxPool($2{\times}2$) blocks of $32$ and $64$
channels ($3{\times}3$ kernels, padding~$1$), flattens to $64{\cdot}7
{\cdot}7$ features, and applies a fully-connected layer to a
$128$-dim latent followed by a linear head to per-concept logits. The
Pizza\"iolo encoder (Experiment~4) takes a $224{\times}224$ RGB image
through four Conv$\to$ReLU$\to$MaxPool blocks of $32$, $64$, $128$,
$128$ channels ($3{\times}3$ kernels, padding~$1$), then global
average pooling to $128$ features and the same $128 \to 128 \to
\text{n\_concepts}$ head. Per-concept Bernoulli probabilities are
obtained by a sigmoid; the same encoder is shared across all named
individuals in a training instance.

\subsection{Hyperparameter selection}\label{app:hp}

For every method in the headline tables we ran a single-seed sweep
over a small per-method grid, picked the best configuration per
regime by RS$_\text{cons}$ and $\text{Acc}_\text{digit}$ on the
seed-0 cell, and then ran a $5$-seed final at that configuration.
All methods share Adam \cite{Kingma2015Adam} with batch size~$32$,
$30$ epochs on MNIST and $60$ epochs on Pizza\"iolo
Tracks~A/B (Track~C uses $15$ epochs because the binary
$\textsf{is\_spicy}$ supervision saturates faster), unless
otherwise noted. The reported MNIST headline tables in
\Cref{tab:headline-acc,tab:headline-ece} use
$5$~seeds at the tuned HPs; the Pizza\"iolo tables
(\Cref{tab:headline-acc,tab:headline-ece}) use $5$~seeds at the
tuned HPs except where indicated.

\paragraph{\textsc{Moose} / Semantic-loss WMC.}
Sweep over $\textit{lr} \in \{10^{-3}, 3{\cdot}10^{-3}\}$ per
regime. Best:
mnist1 / mnist3 / pizzaC use $\textit{lr}{=}3{\cdot}10^{-3}$;
mnist2 / pizzaA / pizzaB use $\textit{lr}{=}10^{-3}$. The
SDD compilation is determined by $\Onto$ and $\Delta$, so the only
free knob is the perception backbone's optimiser.
A wider-range sensitivity analysis over this knob is in
\Cref{app:lr-sensitivity}.

\paragraph{DeepProbLog \cite{Manhaeve2018}.}
Sweep over $\textit{lr} \in \{10^{-3}, 3{\cdot}10^{-3}\}$. Both
regimes (mnist2, mnist3) prefer $\textit{lr}{=}10^{-3}$; the
$\textsf{nn}(\cdot)$ predicate's softmax matches the latent
signature $\Sig^\ell_C$ exactly (one head per concept family);
no additional HPs.

\paragraph{\textsc{Moose}+BEARS \cite{Marconato2024BEARS}.}
Ensemble size $K{=}5$, entropy weight $\gamma_2{=}0.1$ at the
BEARS-paper default; sweep over the KL weight
$\gamma_1 \in \{0.5, 1.0, 2.0\}$ on every regime. The MNIST sweep
is uniform on $\gamma_1{=}1.0$ (the BEARS-paper default), so the
final 5-seed run uses $\gamma_1{=}1.0$ on mnist1, mnist2, mnist3.
On Pizza\"iolo, $\gamma_1$ was tuned per track:
pizzaA → $\gamma_1{=}1.0$, pizzaB → $\gamma_1{=}2.0$,
pizzaC → $\gamma_1{=}1.0$.

\paragraph{\textsc{Moose}+NeSyDM \cite{vanKrieken2025NeSyDM}.}
Sweep over $(\gamma_c, \gamma_h) \in \{0.1, 1.0\}\times\{1.0, 2.0\}$
per regime ($4$ cells), at the RLOO gradient estimator.
$\beta{=}10$ and $T{=}10$ (continuous-time mask budget) are the
NeSyDM-repository defaults, adopted unchanged.
Best per regime:
mnist1 → $(\gamma_c{=}1.0,\,\gamma_h{=}2.0)$;
mnist2 → $(0.1, 2.0)$;
mnist3 → $(1.0, 1.0)$;
pizzaA → $(0.1, 2.0)$;
pizzaB → $(0.1, 2.0)$;
pizzaC → $(0.1, 1.0)$.

\paragraph{LTN \cite{Badreddine2022LTN}.}
Product T-norm; sweep over the p-mean aggregator exponent
$\textsf{forall\_p}{=}\textsf{exists\_p} \in \{2, 4\}$ on every
regime, and additionally over $\textsf{exists\_mode} \in
\{\textsf{universe}, \textsf{skolem}\}$ on MNIST relational/role-chain.
Best: $\textsf{forall\_p}{=}4$ uniformly across all six regimes;
$\textsf{exists\_mode}{=}\textsf{universe}$ on mnist2 and mnist3
(skolem produced equivalent or worse $\text{Acc}_\text{digit}$ at
single seed). $\textit{lr}{=}10^{-3}$.

\paragraph{ELEmbeddings \cite{Kulmanov2019} / mOWL \cite{mOWL2023}.}
Sweep over $(\textsf{embed\_dim}, \textsf{margin}) \in \{50, 128\}
\times \{0.1, 0.5\}$ per regime ($4$ cells). Best:
mnist1 / mnist2 / mnist3 / pizzaB
$\to (\textsf{embed\_dim}{=}50,\,\textsf{margin}{=}0.1)$;
pizzaA $\to (128, 0.5)$ (best Brier among the four cells; all
four reached $\text{Acc}_C{=}1.0$ on the in-distribution split).
For mnist3 the EL ball loss is augmented with an explicit
$\textsf{plus\_two}$ role-chain penalty term (vanilla ELEmbeddings
has no normal form for role chains; \Cref{app:non-wmc-baselines}).

\paragraph{Independent (BCE).}
Sweep over $\textit{lr} \in \{3{\cdot}10^{-4}, 10^{-3},
3{\cdot}10^{-3}\}$ per MNIST regime. Best: mnist1 / mnist3 use
$\textit{lr}{=}3{\cdot}10^{-4}$ (best $\text{Acc}_\text{atom}$ +
ECE); mnist2 uses $\textit{lr}{=}3{\cdot}10^{-3}$. Per-atom BCE
on the observed evidence atoms only; no symbolic structure
enters the loss.

\paragraph{Summary.}
The final numbers in \Cref{tab:headline-acc,tab:headline-ece} are
$5$-seed means at the per-regime-tuned hyperparameters above, after
a per-method single-seed sweep.

\subsection{Learning-rate sensitivity}\label{app:lr-sensitivity}
The per-method grids above are budget-matched: every method is tuned
over a grid of the same size on the same seed-0 protocol, so no
method is advantaged by a larger search. Because \textsc{Moose}
exposes a single free hyperparameter, the perception backbone's
learning rate, while the SDD compilation has none, we additionally
sweep it over a wider five-point grid to check that the headline
numbers are not an artifact of a lucky rate. \Cref{tab:lr-sensitivity}
reports \textsc{Moose} accuracy at each rate ($3$ seeds per cell);
the metric is $\text{Acc}_F$ on MNIST and $\text{Acc}_C$ on
Pizza\"iolo (\Cref{sec:results}). Track~C is not included in this
sweep.

\begin{table}[t]
\caption{\textsc{Moose} learning-rate sensitivity ($3$ seeds per
cell, mean $\pm$ s.d.). Metric: $\text{Acc}_F$ on MNIST
Experiments~2--3, $\text{Acc}_C$ on Pizza\"iolo Tracks~A--B. The
final row is the across-grid spread (max $-$ min over all per-seed
runs in the grid).}\label{tab:lr-sensitivity}
\centering\small
\begin{tabular}{lcccc}
\hline
Learning rate & Exp.~2 & Exp.~3 & Pizza~A & Pizza~B \\
\hline
$10^{-4}$          & $0.729 \pm 0.011$ & $0.941 \pm 0.055$ & $0.861 \pm 0.003$ & $0.810 \pm 0.056$ \\
$3{\cdot}10^{-4}$  & $0.744 \pm 0.020$ & $0.943 \pm 0.068$ & $0.906 \pm 0.000$ & $0.846 \pm 0.010$ \\
$10^{-3}$          & $0.786 \pm 0.030$ & $0.943 \pm 0.071$ & $0.948 \pm 0.007$ & $0.827 \pm 0.015$ \\
$3{\cdot}10^{-3}$  & $0.772 \pm 0.022$ & $0.994 \pm 0.001$ & $0.955 \pm 0.000$ & $0.786 \pm 0.038$ \\
$10^{-2}$          & $0.743 \pm 0.010$ & $0.880 \pm 0.077$ & $0.925 \pm 0.042$ & $0.821 \pm 0.044$ \\
\hline
spread (max$-$min) & $0.100$ & $0.178$ & $0.096$ & $0.129$ \\
\hline
\end{tabular}
\end{table}

On MNIST and Pizza~A the response is unimodal, with the optimum in
the interior of the grid ($10^{-3}$ or $3{\cdot}10^{-3}$) and the
extreme rates worst, as expected; Pizza~B is flatter and less
regular, its per-rate means varying little (all within
$0.786$--$0.846$) with no clear interior peak. \textsc{Moose} is
robust on Exp.~2 and Pizza~A (spread $\le 0.10$) and more
rate-sensitive on Exp.~3 and Pizza~B ($0.13$--$0.18$), where the
low-rate cells underfit and carry large seed variance. In every
regime the tuned rate used for the headline tables sits at or
adjacent to the best-performing rate in this grid, so the reported
numbers are not a lucky-rate artifact.

\subsection{Full per-regime results}\label{app:full-results}

\providecommand{\revision}[1]{#1}
\begingroup
\scriptsize
\setlength{\tabcolsep}{1.5pt}
\renewcommand{\arraystretch}{0.95}
\setlength\LTleft{0pt}
\setlength\LTright{0pt}
\begin{longtable}{@{\extracolsep{\fill}}l*{6}{r}@{}}
\caption{Full per-regime results, six metrics. Means $\pm$ s.d.\ over $5$ seeds. \revision{The third column is $\text{Acc}_F$, family-argmax accuracy, on the MNIST blocks and $\text{Acc}_C$, per-atom accuracy on the latent slice, on the Pizza\"iolo blocks, which declare no exhaustive family on Track~C (\Cref{app:pizza-argmax}); the two are compared only within a block. Operative theory: $\Gamma_{\textsf{EL}}{+}\Phi_{\textsf{clos}}$ for the \textsc{Moose} rows on MNIST, $\Gamma_{\textsf{EL}}$ for every other row (closure is opt-in and unused on Pizza\"iolo), and none for Independent. RS$_\text{cons} = \mathrm{Conf}\cdot(1-\mathrm{Acc})$ is the confidence-weighted error rate of that decode: lower is better, and a high value means confident commitment to a wrong latent explanation.} $\text{F1}_\text{macro}$ is the macro-averaged per-concept F1 on the latent signature. Bold marks the column-best within each block. \revision{On Track~A the OOD harness records a single accuracy, so the $\text{Acc}_\text{atom}$ and $\text{Acc}_C$ columns coincide there by construction rather than by coincidence.} DeepProbLog$^\ddagger$ rows on Pizza\"iolo Tracks~A/B are not separate runs: \revision{no ProbLog program was written for Pizza\"iolo, and the cells reproduce the \textsc{Moose} (Semantic-loss WMC) values, which \Cref{prop:dpl-equiv} shows the translation $\mathcal{P}(\Gamma_{\textsf{EL}})$ would attain exactly. Note that this translation uses one independent probabilistic fact per ground atom, whereas the MNIST DeepProbLog programs use an $\textsf{nn}/4$ annotated disjunction over the digit family, so the two halves of the row are different programs.} RS$_\text{cons}$ and $\text{F1}_\text{macro}$ on Track~C (binary $\textsf{is\_spicy}$ supervision) collapse to $0$ across all methods because the metric is defined over a multi-class concept family; we mark these cells N/A. Pizza\"iolo Track~A RS$_\text{cons}$/$\text{F1}_\text{macro}$ are computed on the in-distribution eval block (the OOD harness does not record per-atom logits for these metrics).}\label{tab:headline-full}
\\
\hline
Method & $\text{Acc}_\text{atom}\!\uparrow$ & $\text{Acc}_{F/C}\!\uparrow$ & NLL$\!\downarrow$ & ECE$\!\downarrow$ & RS$_\text{cons}\!\downarrow$ & $\text{F1}_\text{macro}\!\uparrow$ \\
\hline
\endfirsthead
\multicolumn{7}{l}{\textit{Table~\thetable\ continued.}} \\
\hline
Method & $\text{Acc}_\text{atom}\!\uparrow$ & $\text{Acc}_{F/C}\!\uparrow$ & NLL$\!\downarrow$ & ECE$\!\downarrow$ & RS$_\text{cons}\!\downarrow$ & $\text{F1}_\text{macro}\!\uparrow$ \\
\hline
\endhead
\hline
\multicolumn{7}{r}{\textit{Continued on next page.}} \\
\endfoot
\hline
\endlastfoot
\multicolumn{7}{l}{\emph{Experiment~1: atomic MNIST (RQ1).} $|\Delta|{=}1$, NF1/NF2 fragment.} \\
\hline
Independent (BCE) & 58.0{\scriptsize\,$\pm$\,5.6} & 13.2{\scriptsize\,$\pm$\,7.9} & 2.9{\scriptsize\,$\pm$\,0.2} & 9.3{\scriptsize\,$\pm$\,4.5} & 15.7{\scriptsize\,$\pm$\,3.3} & 7.6{\scriptsize\,$\pm$\,5.7} \\
DeepProbLog & 89.0{\scriptsize\,$\pm$\,0.8} & 42.1{\scriptsize\,$\pm$\,1.3} & 1.2{\scriptsize\,$\pm$\,0.1} & 10.3{\scriptsize\,$\pm$\,0.9} & 31.3{\scriptsize\,$\pm$\,2.6} & 31.2{\scriptsize\,$\pm$\,6.9} \\
LTN & \textbf{92.2{\scriptsize\,$\pm$\,0.1}} & 25.9{\scriptsize\,$\pm$\,6.5} & 1.1{\scriptsize\,$\pm$\,0.0} & 10.0{\scriptsize\,$\pm$\,0.0} & 47.1{\scriptsize\,$\pm$\,7.8} & 16.1{\scriptsize\,$\pm$\,6.7} \\
ELEmbeddings/mOWL & 90.0{\scriptsize\,$\pm$\,0.1} & 34.9{\scriptsize\,$\pm$\,4.2} & \textbf{0.3{\scriptsize\,$\pm$\,0.0}} & \textbf{3.9{\scriptsize\,$\pm$\,0.4}} & \textbf{14.9{\scriptsize\,$\pm$\,0.9}} & 25.8{\scriptsize\,$\pm$\,2.8} \\
\textsc{Moose} (Semantic-loss WMC) & 89.3{\scriptsize\,$\pm$\,1.8} & 48.1{\scriptsize\,$\pm$\,10.1} & 1.6{\scriptsize\,$\pm$\,0.3} & 10.6{\scriptsize\,$\pm$\,1.8} & 29.0{\scriptsize\,$\pm$\,4.2} & 32.1{\scriptsize\,$\pm$\,10.0} \\
\textsc{Moose}+BEARS & 91.2{\scriptsize\,$\pm$\,0.9} & \textbf{50.1{\scriptsize\,$\pm$\,5.2}} & \textbf{0.3{\scriptsize\,$\pm$\,0.1}} & \textbf{3.9{\scriptsize\,$\pm$\,1.0}} & 20.0{\scriptsize\,$\pm$\,1.4} & 28.0{\scriptsize\,$\pm$\,6.2} \\
\textsc{Moose}+NeSyDM (RLOO) & 90.6{\scriptsize\,$\pm$\,0.2} & 46.6{\scriptsize\,$\pm$\,2.1} & \textbf{0.3{\scriptsize\,$\pm$\,0.0}} & 6.0{\scriptsize\,$\pm$\,0.6} & 23.2{\scriptsize\,$\pm$\,0.9} & 35.8{\scriptsize\,$\pm$\,1.9} \\
\textsc{Moose}+NeSyDM (exact) & 90.1{\scriptsize\,$\pm$\,1.6} & 49.4{\scriptsize\,$\pm$\,9.7} & 0.7{\scriptsize\,$\pm$\,0.2} & 5.1{\scriptsize\,$\pm$\,1.5} & 25.6{\scriptsize\,$\pm$\,2.1} & \textbf{35.9{\scriptsize\,$\pm$\,8.6}} \\
\hline
\multicolumn{7}{l}{\emph{Experiment~2: relational MNIST (RQ2).} $|\Delta|{=}2$, NF3 cyclic-successor TBox.} \\
\hline
Independent (BCE) & 61.8{\scriptsize\,$\pm$\,4.9} & 8.8{\scriptsize\,$\pm$\,4.5} & 2.8{\scriptsize\,$\pm$\,0.3} & 13.4{\scriptsize\,$\pm$\,3.6} & 17.6{\scriptsize\,$\pm$\,2.3} & 5.7{\scriptsize\,$\pm$\,3.3} \\
DeepProbLog & 85.3{\scriptsize\,$\pm$\,0.8} & 38.9{\scriptsize\,$\pm$\,3.2} & 1.4{\scriptsize\,$\pm$\,0.1} & 12.3{\scriptsize\,$\pm$\,0.9} & 57.9{\scriptsize\,$\pm$\,3.2} & 17.5{\scriptsize\,$\pm$\,4.1} \\
LTN & 91.9{\scriptsize\,$\pm$\,0.1} & 14.5{\scriptsize\,$\pm$\,4.9} & 1.2{\scriptsize\,$\pm$\,0.0} & 10.0{\scriptsize\,$\pm$\,0.0} & 40.1{\scriptsize\,$\pm$\,8.6} & 11.4{\scriptsize\,$\pm$\,5.8} \\
ELEmbeddings/mOWL & 90.1{\scriptsize\,$\pm$\,0.2} & 34.8{\scriptsize\,$\pm$\,3.6} & 0.4{\scriptsize\,$\pm$\,0.0} & 7.2{\scriptsize\,$\pm$\,0.1} & \textbf{6.9{\scriptsize\,$\pm$\,0.4}} & 25.8{\scriptsize\,$\pm$\,3.7} \\
\textsc{Moose} (Semantic-loss WMC) & 93.8{\scriptsize\,$\pm$\,0.3} & 74.6{\scriptsize\,$\pm$\,0.3} & 0.5{\scriptsize\,$\pm$\,0.1} & 4.9{\scriptsize\,$\pm$\,0.3} & 24.9{\scriptsize\,$\pm$\,0.3} & 63.7{\scriptsize\,$\pm$\,1.1} \\
\textsc{Moose}+BEARS & \textbf{93.9{\scriptsize\,$\pm$\,1.1}} & \textbf{76.5{\scriptsize\,$\pm$\,5.9}} & \textbf{0.2{\scriptsize\,$\pm$\,0.0}} & \textbf{4.0{\scriptsize\,$\pm$\,1.2}} & 18.0{\scriptsize\,$\pm$\,4.5} & \textbf{67.8{\scriptsize\,$\pm$\,9.3}} \\
\textsc{Moose}+NeSyDM (RLOO) & 90.9{\scriptsize\,$\pm$\,0.4} & 58.2{\scriptsize\,$\pm$\,1.4} & 0.3{\scriptsize\,$\pm$\,0.0} & 7.3{\scriptsize\,$\pm$\,0.7} & 30.3{\scriptsize\,$\pm$\,0.8} & 45.8{\scriptsize\,$\pm$\,4.3} \\
\textsc{Moose}+NeSyDM (exact) & 92.5{\scriptsize\,$\pm$\,1.2} & 62.6{\scriptsize\,$\pm$\,6.1} & 1.2{\scriptsize\,$\pm$\,0.2} & 7.4{\scriptsize\,$\pm$\,1.2} & 37.2{\scriptsize\,$\pm$\,6.1} & 59.1{\scriptsize\,$\pm$\,4.5} \\
\hline
\multicolumn{7}{l}{\emph{Experiment~3: role-chain MNIST (RQ3).} $|\Delta|{=}2$, $\textsf{plus\_two}{=}\textsf{succ}\circ\textsf{succ}$.} \\
\hline
Independent (BCE) & 61.5{\scriptsize\,$\pm$\,3.3} & 12.2{\scriptsize\,$\pm$\,7.3} & 3.3{\scriptsize\,$\pm$\,0.4} & 9.0{\scriptsize\,$\pm$\,8.2} & 17.5{\scriptsize\,$\pm$\,3.5} & 7.9{\scriptsize\,$\pm$\,4.4} \\
DeepProbLog & 90.6{\scriptsize\,$\pm$\,3.1} & 59.6{\scriptsize\,$\pm$\,13.6} & 0.9{\scriptsize\,$\pm$\,0.3} & 8.0{\scriptsize\,$\pm$\,2.9} & 38.3{\scriptsize\,$\pm$\,12.7} & 48.4{\scriptsize\,$\pm$\,14.3} \\
LTN & 92.0{\scriptsize\,$\pm$\,0.1} & 11.0{\scriptsize\,$\pm$\,7.9} & 1.2{\scriptsize\,$\pm$\,0.0} & 10.0{\scriptsize\,$\pm$\,0.0} & 47.1{\scriptsize\,$\pm$\,6.9} & 8.2{\scriptsize\,$\pm$\,6.7} \\
ELEmbeddings/mOWL & 90.4{\scriptsize\,$\pm$\,0.1} & 27.4{\scriptsize\,$\pm$\,7.1} & 0.3{\scriptsize\,$\pm$\,0.0} & 6.3{\scriptsize\,$\pm$\,0.1} & 7.5{\scriptsize\,$\pm$\,0.7} & 18.7{\scriptsize\,$\pm$\,6.7} \\
\textsc{Moose} (Semantic-loss WMC) & \textbf{98.7{\scriptsize\,$\pm$\,1.8}} & \textbf{96.1{\scriptsize\,$\pm$\,7.2}} & \textbf{0.1{\scriptsize\,$\pm$\,0.2}} & \textbf{1.1{\scriptsize\,$\pm$\,1.5}} & \textbf{3.9{\scriptsize\,$\pm$\,7.1}} & \textbf{93.8{\scriptsize\,$\pm$\,8.6}} \\
\textsc{Moose}+BEARS & 94.8{\scriptsize\,$\pm$\,0.5} & 89.2{\scriptsize\,$\pm$\,7.2} & \textbf{0.1{\scriptsize\,$\pm$\,0.0}} & 5.6{\scriptsize\,$\pm$\,1.1} & 8.3{\scriptsize\,$\pm$\,5.2} & 81.0{\scriptsize\,$\pm$\,12.2} \\
\textsc{Moose}+NeSyDM (RLOO) & 91.8{\scriptsize\,$\pm$\,0.5} & 62.7{\scriptsize\,$\pm$\,2.9} & 0.2{\scriptsize\,$\pm$\,0.0} & 4.6{\scriptsize\,$\pm$\,0.8} & 28.0{\scriptsize\,$\pm$\,2.1} & 57.6{\scriptsize\,$\pm$\,2.0} \\
\textsc{Moose}+NeSyDM (exact) & 92.0{\scriptsize\,$\pm$\,1.2} & 60.1{\scriptsize\,$\pm$\,5.7} & 1.3{\scriptsize\,$\pm$\,0.2} & 7.9{\scriptsize\,$\pm$\,1.1} & 36.1{\scriptsize\,$\pm$\,9.7} & 62.4{\scriptsize\,$\pm$\,10.1} \\
\hline
\multicolumn{7}{l}{\emph{Pizza\"iolo Track~A: pizza-type, OOD eval (RQ4).} $5$ held-out tie-breaker pizzas (\Cref{exp:pizzaiolo}).} \\
\hline
Independent (BCE) & 46.6{\scriptsize\,$\pm$\,6.2} & 21.7{\scriptsize\,$\pm$\,9.7} & 6.4{\scriptsize\,$\pm$\,5.0} & 58.3{\scriptsize\,$\pm$\,15.4} & 62.2{\scriptsize\,$\pm$\,12.3} & 11.1{\scriptsize\,$\pm$\,6.4} \\
DeepProbLog$^\ddagger$ & 68.1{\scriptsize\,$\pm$\,0.3} & 68.1{\scriptsize\,$\pm$\,0.3} & 4.1{\scriptsize\,$\pm$\,0.2} & 30.8{\scriptsize\,$\pm$\,0.7} & 4.2{\scriptsize\,$\pm$\,7.6} & 14.1{\scriptsize\,$\pm$\,3.5} \\
LTN & 74.0{\scriptsize\,$\pm$\,3.4} & 22.1{\scriptsize\,$\pm$\,5.2} & 13.6{\scriptsize\,$\pm$\,5.4} & 52.7{\scriptsize\,$\pm$\,21.8} & 57.6{\scriptsize\,$\pm$\,17.3} & 10.9{\scriptsize\,$\pm$\,6.8} \\
ELEmbeddings/mOWL & \textbf{77.0{\scriptsize\,$\pm$\,0.9}} & 25.7{\scriptsize\,$\pm$\,7.3} & 4.9{\scriptsize\,$\pm$\,1.0} & 49.7{\scriptsize\,$\pm$\,10.2} & 54.9{\scriptsize\,$\pm$\,6.6} & 11.2{\scriptsize\,$\pm$\,6.5} \\
\textsc{Moose} (Semantic-loss WMC) & 68.1{\scriptsize\,$\pm$\,0.3} & 68.1{\scriptsize\,$\pm$\,0.3} & 4.1{\scriptsize\,$\pm$\,0.2} & 30.8{\scriptsize\,$\pm$\,0.7} & 4.2{\scriptsize\,$\pm$\,7.6} & 14.1{\scriptsize\,$\pm$\,3.5} \\
\textsc{Moose}+BEARS & 68.6{\scriptsize\,$\pm$\,0.3} & 68.6{\scriptsize\,$\pm$\,0.3} & 3.3{\scriptsize\,$\pm$\,0.3} & 28.2{\scriptsize\,$\pm$\,1.5} & \textbf{0.6{\scriptsize\,$\pm$\,0.7}} & 30.1{\scriptsize\,$\pm$\,10.2} \\
\textsc{Moose}+NeSyDM (RLOO) & 72.2{\scriptsize\,$\pm$\,1.3} & \textbf{72.2{\scriptsize\,$\pm$\,1.3}} & \textbf{0.6{\scriptsize\,$\pm$\,0.0}} & \textbf{7.5{\scriptsize\,$\pm$\,1.2}} & 34.3{\scriptsize\,$\pm$\,4.2} & \textbf{44.8{\scriptsize\,$\pm$\,6.3}} \\
\textsc{Moose}+NeSyDM (exact) & 67.4{\scriptsize\,$\pm$\,1.4} & 67.4{\scriptsize\,$\pm$\,1.4} & 4.6{\scriptsize\,$\pm$\,0.8} & 30.8{\scriptsize\,$\pm$\,3.2} & 3.8{\scriptsize\,$\pm$\,8.0} & 43.4{\scriptsize\,$\pm$\,3.0} \\
\hline
\multicolumn{7}{l}{\emph{Pizza\"iolo Track~B: property-class (RQ4, RQ5).} 3-way RS over the property class.} \\
\hline
Independent (BCE) & 57.2{\scriptsize\,$\pm$\,3.3} & 25.0{\scriptsize\,$\pm$\,0.0} & 2.7{\scriptsize\,$\pm$\,0.9} & 26.4{\scriptsize\,$\pm$\,17.1} & 37.0{\scriptsize\,$\pm$\,12.8} & 12.0{\scriptsize\,$\pm$\,1.1} \\
DeepProbLog$^\ddagger$ & 82.4{\scriptsize\,$\pm$\,1.9} & 84.5{\scriptsize\,$\pm$\,2.1} & 2.4{\scriptsize\,$\pm$\,0.3} & 15.5{\scriptsize\,$\pm$\,2.1} & 49.8{\scriptsize\,$\pm$\,4.6} & 37.8{\scriptsize\,$\pm$\,1.1} \\
LTN & 83.3{\scriptsize\,$\pm$\,1.6} & 32.5{\scriptsize\,$\pm$\,20.9} & 7.0{\scriptsize\,$\pm$\,6.0} & 39.3{\scriptsize\,$\pm$\,17.4} & 28.8{\scriptsize\,$\pm$\,12.5} & 23.4{\scriptsize\,$\pm$\,19.0} \\
ELEmbeddings/mOWL & 76.9{\scriptsize\,$\pm$\,0.0} & 25.0{\scriptsize\,$\pm$\,0.0} & 0.6{\scriptsize\,$\pm$\,0.0} & 16.1{\scriptsize\,$\pm$\,0.2} & \textbf{19.5{\scriptsize\,$\pm$\,0.1}} & 10.0{\scriptsize\,$\pm$\,0.0} \\
\textsc{Moose} (Semantic-loss WMC) & 82.4{\scriptsize\,$\pm$\,1.9} & 84.5{\scriptsize\,$\pm$\,2.1} & 2.4{\scriptsize\,$\pm$\,0.3} & 15.5{\scriptsize\,$\pm$\,2.1} & 49.8{\scriptsize\,$\pm$\,4.6} & 37.8{\scriptsize\,$\pm$\,1.1} \\
\textsc{Moose}+BEARS & 85.4{\scriptsize\,$\pm$\,0.8} & \textbf{87.9{\scriptsize\,$\pm$\,0.9}} & \textbf{0.3{\scriptsize\,$\pm$\,0.0}} & 12.3{\scriptsize\,$\pm$\,1.0} & 30.3{\scriptsize\,$\pm$\,8.6} & \textbf{42.0{\scriptsize\,$\pm$\,9.0}} \\
\textsc{Moose}+NeSyDM (RLOO) & 80.1{\scriptsize\,$\pm$\,0.8} & 78.5{\scriptsize\,$\pm$\,0.8} & 0.5{\scriptsize\,$\pm$\,0.0} & \textbf{2.6{\scriptsize\,$\pm$\,0.9}} & 38.0{\scriptsize\,$\pm$\,4.7} & 22.6{\scriptsize\,$\pm$\,2.3} \\
\textsc{Moose}+NeSyDM (exact) & \textbf{86.3{\scriptsize\,$\pm$\,1.9}} & 84.7{\scriptsize\,$\pm$\,2.2} & 2.5{\scriptsize\,$\pm$\,0.4} & 15.3{\scriptsize\,$\pm$\,2.2} & 50.3{\scriptsize\,$\pm$\,4.4} & 37.7{\scriptsize\,$\pm$\,1.0} \\
\hline
\multicolumn{7}{l}{\emph{Pizza\"iolo Track~C: is\_spicy (RQ4, RQ5).} Multi-witness disjunction; $15$ epochs.} \\
\hline
Independent (BCE) & 55.2{\scriptsize\,$\pm$\,5.4} & 48.9{\scriptsize\,$\pm$\,5.6} & 1.0{\scriptsize\,$\pm$\,0.2} & 22.1{\scriptsize\,$\pm$\,9.2} & N/A & N/A \\
DeepProbLog$^\ddagger$ & 77.3{\scriptsize\,$\pm$\,2.8} & 81.0{\scriptsize\,$\pm$\,3.0} & 1.5{\scriptsize\,$\pm$\,0.3} & 15.3{\scriptsize\,$\pm$\,1.4} & N/A & N/A \\
LTN & \textbf{85.1{\scriptsize\,$\pm$\,0.9}} & \textbf{84.8{\scriptsize\,$\pm$\,0.4}} & 1.3{\scriptsize\,$\pm$\,0.3} & 14.1{\scriptsize\,$\pm$\,0.5} & N/A & N/A \\
ELEmbeddings/mOWL & 78.8{\scriptsize\,$\pm$\,6.5} & 75.4{\scriptsize\,$\pm$\,6.6} & \textbf{0.3{\scriptsize\,$\pm$\,0.1}} & 11.8{\scriptsize\,$\pm$\,3.8} & N/A & N/A \\
\textsc{Moose} (Semantic-loss WMC) & 77.3{\scriptsize\,$\pm$\,2.8} & 81.0{\scriptsize\,$\pm$\,3.0} & 1.5{\scriptsize\,$\pm$\,0.3} & 15.3{\scriptsize\,$\pm$\,1.4} & N/A & N/A \\
\textsc{Moose}+BEARS & 79.9{\scriptsize\,$\pm$\,3.2} & 84.2{\scriptsize\,$\pm$\,3.6} & 0.4{\scriptsize\,$\pm$\,0.0} & 11.3{\scriptsize\,$\pm$\,3.4} & N/A & N/A \\
\textsc{Moose}+NeSyDM (RLOO) & 66.1{\scriptsize\,$\pm$\,3.7} & 63.9{\scriptsize\,$\pm$\,3.7} & 0.7{\scriptsize\,$\pm$\,0.0} & 8.9{\scriptsize\,$\pm$\,1.5} & N/A & N/A \\
\textsc{Moose}+NeSyDM (exact) & 75.9{\scriptsize\,$\pm$\,1.4} & 80.1{\scriptsize\,$\pm$\,1.5} & 0.8{\scriptsize\,$\pm$\,0.2} & \textbf{6.9{\scriptsize\,$\pm$\,2.2}} & N/A & N/A \\
\end{longtable}
\endgroup

\Cref{tab:headline-full} reports all six metrics
($\text{Acc}_\text{atom}$, the principal accuracy, NLL, ECE, RS$_\text{cons}$,
$\text{F1}_\text{macro}$) per regime, laid out as six row-blocks
(MNIST Experiments~1--3, Pizza\"iolo Tracks~A--C).  The compact
main-paper tables \Cref{tab:headline-acc} and \Cref{tab:headline-ece}
are summary views of the principal-accuracy and ECE columns of this table;
per-regime $\text{Acc}_\text{atom}$, NLL, RS$_\text{cons}$, and
$\text{F1}_\text{macro}$ appear only here. RS$_\text{cons}$ and
$\text{F1}_\text{macro}$ on Pizza\"iolo Track~A are computed from the
in-distribution eval block (the OOD harness writes only per-pizza
accuracies); on Track~C ($\textsf{is\_spicy}$ binary supervision) both
metrics are defined over a multi-class concept family and are not
applicable.

\subsection{Family-argmax on Pizza\"iolo}\label{app:pizza-argmax}
Family-argmax accuracy $\text{Acc}_F$ is defined
per individual as the argmax of the WMC posterior over a
\emph{declared exhaustive family}, scored against ground truth. On
MNIST the digit family $\{D_0,\dots,D_9\}$ is declared exhaustive and
$\text{Acc}_F$ is the reported metric. Pizza\"iolo does not admit the
same decode across its three tracks; we therefore report
$\text{Acc}_C$ there and label it as such rather than folding two
different quantities under one symbol. This appendix states the two
metrics precisely and shows why the Pizza\"iolo tracks carry
$\text{Acc}_C$.

Track~C is the clear case: $\textsf{is\_spicy}$ supervision is a
multi-witness disjunction over four spicy toppings and declares no
exhaustive family, so there is no set to take an argmax over and
$\text{Acc}_F$ is undefined; the harness records $0$ for these cells
by absence, not as a score, and we report them as N/A throughout.
Track~A is reported on the out-of-distribution tie-breaker split,
whose harness records per-pizza accuracies only and no family decode,
so $\text{Acc}_F$ is unavailable there without a rerun. Track~B does
admit the decode, over the $3$-way property class, and we report it in
\Cref{tab:pizza-argmax}: every method lands between $48.0$ and $50.2$,
a $2.2$-point spread against per-method standard deviations of
$3.6$--$5.3$, so the metric separates no pair of methods. The
$\text{Acc}_C$ column of the same runs spreads $78.5$--$87.9$ and does
separate them. Family-argmax on Pizza\"iolo would therefore replace an
informative metric with an uninformative one on the only track where
it is even computable, which is why the headline tables keep
$\text{Acc}_C$ and name it.

\begin{table}[t]
\caption{Pizza\"iolo Track~B under both accuracy metrics ($5$ seeds,
mean $\pm$ s.d., \%). $\text{Acc}_F$ is the family-argmax decode over
the $3$-way property class; $\text{Acc}_C$ is the per-atom accuracy on
the latent slice reported in \Cref{tab:headline-acc}. Rows absent from
the $\text{Acc}_F$ column did not record a family decode. Tracks~A
and~C are omitted: the metric is unavailable and undefined there
respectively (see text).}\label{tab:pizza-argmax}
\centering\small
\begin{tabular}{lcc}
\hline
Method & $\text{Acc}_F$ & $\text{Acc}_C$ \\
\hline
LTN                            & $49.5 \pm 3.6$ & $83.2 \pm 1.9$ \\
\textsc{Moose} (Semantic-loss WMC) & $50.2 \pm 4.6$ & $84.5 \pm 2.1$ \\
\textsc{Moose}+BEARS           & $49.8 \pm 3.9$ & $87.9 \pm 0.9$ \\
\textsc{Moose}+NeSyDM (RLOO)   & $48.0 \pm 5.3$ & $78.5 \pm 0.8$ \\
\textsc{Moose}+NeSyDM (exact)  & $49.7 \pm 4.4$ & $84.7 \pm 2.2$ \\
\hline
spread (max$-$min)             & $2.2$          & $9.4$ \\
\hline
\end{tabular}
\end{table}

\subsection{Paired significance of the headline
comparisons}\label{app:significance}
A mean $\pm$ s.d.\ over five runs is a
stability summary rather than a significance test, and the shared
evaluation instances call for a paired analysis. We therefore rerun
the principal comparisons over $20$ seeds and test each
\textsc{Moose}-vs-baseline difference with a paired Wilcoxon
signed-rank test and an exact sign-flip permutation test on the mean
difference, sharing seeds within every pair. Within each regime the
$p$-values are Holm-corrected across the baselines, and we report
significance at $\alpha = 0.05$. The metric is the regime's headline
accuracy: family-argmax accuracy $\text{Acc}_F$ on MNIST
Experiments~2--3, and latent-concept accuracy $\text{Acc}_C$ on
Pizza\"iolo Tracks~A--C (\Cref{tab:significance}); the Pizza\"iolo
significance statements therefore concern the latent-atom metric.

\begin{table}[t]
\caption{Paired significance of \textsc{Moose} (Semantic-loss WMC)
against each baseline over $20$ shared seeds. ``diff'' is the mean
per-seed difference (\textsc{Moose} $-$ baseline); ``Holm $p$'' is the
Holm-corrected Wilcoxon $p$-value within the regime. The verdict is
\emph{win} (\textsc{Moose} significantly higher), \emph{loss}
(significantly lower), or \emph{ns} (not significant) at
$\alpha = 0.05$. Metric: $\text{Acc}_F$ on MNIST, $\text{Acc}_C$ on
Pizza\"iolo. The $20$-seed reruns cover \textsc{Moose}, DeepProbLog, NeSyDM
(RLOO), LTN, and BEARS on Experiments~2--3 and Pizza\"iolo
Tracks~A--C; Experiment~1 and the remaining baselines (Independent,
ELEmbeddings, NeSyDM-exact) are $5$-seed only and are not part of the
paired analysis.}\label{tab:significance}
\centering\small
\begin{tabular}{llccccl}
\hline
Regime & Baseline & \textsc{Moose} & base & diff & Holm $p$ & verdict \\
\hline
MNIST Exp.~2 & DeepProbLog & $0.745$ & $0.423$ & $+0.322$ & $0.0002$ & win \\
             & NeSyDM      & $0.745$ & $0.595$ & $+0.150$ & $0.0002$ & win \\
             & LTN         & $0.745$ & $0.150$ & $+0.595$ & $0.0000$ & win \\
\hline
MNIST Exp.~3 & DeepProbLog & $0.977$ & $0.598$ & $+0.378$ & $0.0000$ & win \\
             & NeSyDM      & $0.977$ & $0.631$ & $+0.346$ & $0.0001$ & win \\
             & LTN         & $0.977$ & $0.151$ & $+0.825$ & $0.0000$ & win \\
\hline
Pizza~A (OOD) & NeSyDM       & $0.679$ & $0.709$ & $-0.031$ & $0.0004$ & loss \\
              & BEARS        & $0.679$ & $0.683$ & $-0.004$ & $0.1454$ & ns \\
\hline
Pizza~B & NeSyDM       & $0.841$ & $0.790$ & $+0.052$ & $0.0002$ & win \\
        & BEARS        & $0.841$ & $0.880$ & $-0.038$ & $0.0002$ & loss \\
\hline
Pizza~C & NeSyDM       & $0.809$ & $0.617$ & $+0.192$ & $0.0002$ & win \\
        & BEARS        & $0.809$ & $0.856$ & $-0.047$ & $0.0003$ & loss \\
\hline
\end{tabular}
\end{table}

The relational and role-chain regimes are unambiguous: \textsc{Moose}
beats every propositional baseline on Experiments~2--3 by margins that
are large and significant (Holm $p \le 0.0002$). On Pizza\"iolo the
picture is mixed, and we report it as such. \textsc{Moose} is
significantly stronger than NeSyDM on Tracks~B and~C, but BEARS is
significantly stronger than \textsc{Moose} on both; on the Track~A OOD
split NeSyDM significantly overtakes \textsc{Moose}, while the
\textsc{Moose}-vs-BEARS gap there is not significant. \textsc{Moose}
is therefore competitive but not dominant under symbolic ambiguity, and its
clear advantage is in the relational regime where EL-aware grounding
propagates evidence that propositional encodings cannot.

The cell means here differ marginally from the $5$-seed headline
tables (\Cref{tab:headline-acc}) because this analysis uses the
$20$-seed reruns; e.g.\ \textsc{Moose} on Exp.~2 is $0.745$ here
versus $0.746$ in the $5$-seed table. These paired tests
compare \textsc{Moose} against each baseline; the headline tables
mark the raw column-best rather than a significance verdict, since
the $20$-seed reruns do not cover every method or regime
(\Cref{tab:significance}). 
The test covers accuracy only, so
\Cref{tab:headline-ece} still marks the raw column-best.

\subsection{Inductive generalization}\label{app:inductive}
A further question is whether the learned models generalize to new
query
individuals rather than only to held-out images. We test this with an
inductive held-out-edge split of the relational (Exp.~2) and
role-chain (Exp.~3) regimes. On Exp.~2 we supervise only the
successor edges whose source digit lies in $\{0,2,4,6,8\}$
($0\!\to\!1, 2\!\to\!3, \dots, 8\!\to\!9$) and query individuals on
the held-out edges $\{1,3,5,7,9\}$
($1\!\to\!2, 3\!\to\!4, \dots, 9\!\to\!0$); on Exp.~3 the
$\textsf{plus\_two}$ relation forms two $5$-cycles, and we supervise
sources $\{0,1,2,3,6,7\}$ and query the held-out sources
$\{4,5,8,9\}$. In both splits the training and query edge sets
jointly cover all ten digit nodes, so the shared perception network
sees every digit during training; only the \emph{relational
configuration} of a query pair is novel, not its images; this is a
test of relational generalization, not zero-shot perception. As a
control we also evaluate on the supervised edge types with fresh
images (the transductive column of \Cref{tab:inductive}).

\begin{table}[t]
\caption{Inductive generalization over $20$ seeds
($\text{Acc}_F$, family-argmax accuracy). \emph{Inductive} evaluates
on relational configurations never supervised; \emph{transductive}
is the same-edge-type control on new images; \emph{drop} is
transductive $-$ inductive (positive means degradation on unseen
configurations).}\label{tab:inductive}
\centering\small
\begin{tabular}{llccc}
\hline
Regime & Method & Inductive & Transductive & Drop \\
\hline
Exp.~2 (relational) & \textsc{Moose} (WMC) & $0.717 \pm 0.071$ & $0.685$ & $-0.032$ \\
                    & Independent          & $0.582 \pm 0.036$ & $0.605$ & $+0.023$ \\
                    & NeSyDM               & $0.135 \pm 0.024$ & $0.594$ & $+0.459$ \\
\hline
Exp.~3 (role-chain) & \textsc{Moose} (WMC) & $0.595 \pm 0.138$ & $0.574$ & $-0.021$ \\
                    & Independent          & $0.551 \pm 0.093$ & $0.710$ & $+0.159$ \\
                    & NeSyDM               & $0.254 \pm 0.042$ & $0.712$ & $+0.458$ \\
\hline
\end{tabular}
\end{table}

\textsc{Moose} shows no inductive degradation on either regime: its
held-out-configuration accuracy matches its transductive control to
within noise ($-0.032$ and $-0.021$, i.e.\ marginally higher
inductive). NeSyDM, by contrast, collapses from $0.594$/$0.712$
transductive to $0.135$/$0.254$ on held-out configurations, and
the Independent baseline degrades moderately. The WMC layer scores
whole declared families rather than memorized edge identities, so a
query pair in an unseen configuration is handled by the same
exact-inference path as a supervised one, which is what preserves
accuracy. The paired \textsc{Moose}-vs-Independent advantage on the
inductive split is significant on Exp.~2 (Wilcoxon $p < 10^{-4}$,
$n=20$) but not on Exp.~3 ($p = 0.52$), where the role-chain circuit
leaves both methods with high seed-to-seed variance.

\subsection{Closure ablation: separating
$\Gamma_{\textsf{EL}}$ from $\Phi_{\textsf{clos}}$}\label{app:closure-ablation}
The compiled theory combines the EL-aware extractors
$\Gamma_{\textsf{EL}}$ with the optional closure clauses
$\Phi_{\textsf{clos}}$, whose covering axiom
$\top\sqsubseteq\bigsqcup_i D_i$ lies outside the EL profile. To
isolate their contributions we disable $\Phi_{\textsf{clos}}$ on all
three MNIST regimes (\texttt{--no-closed-world}, $5$ seeds) and pair
against the with-closure baseline (\Cref{tab:closure-ablation}).

\begin{table}[H]
\caption{Closure ablation on MNIST. $\text{Acc}_F$ (mean $\pm$ s.d.\
over $5$ seeds) with and without the exhaustive-family closure
clauses $\Phi_{\textsf{clos}}$. The
\texttt{--no-closed-world} arm removes mutual exclusion and covering
in all three regimes and profile-keyed reverse implications in
Experiment~1. $\Delta$ is the change in the mean when
$\Phi_{\textsf{clos}}$ is removed.}\label{tab:closure-ablation}
\centering\small
\begin{tabular}{llccc}
\hline
Regime & Method & with $\Phi_{\textsf{clos}}$ & without $\Phi_{\textsf{clos}}$ & $\Delta$ \\
\hline
Exp.~1 (atomic)     & \textsc{Moose}       & $48.1 \pm 10.1$ & $21.7 \pm 0.0$ & $-26.4$ \\
                    & \quad +BEARS         & $50.1 \pm 5.2$  & $21.7 \pm 0.0$ & $-28.4$ \\
                    & \quad +NeSyDM (RLOO) & $46.6 \pm 2.1$  & $21.7 \pm 0.0$ & $-24.9$ \\
                    & \quad +NeSyDM (exact) & $49.4 \pm 9.7$  & $21.7 \pm 0.0$ & $-27.7$ \\
\hline
Exp.~2 (relational) & \textsc{Moose}       & $74.6 \pm 0.3$  & $17.4 \pm 1.8$ & $-57.2$ \\
                    & \quad +BEARS         & $76.5 \pm 5.9$  & $30.9 \pm 4.5$ & $-45.6$ \\
                    & \quad +NeSyDM (RLOO) & $58.2 \pm 1.4$  & $39.4 \pm 3.0$ & $-18.8$ \\
                    & \quad +NeSyDM (exact) & $62.6 \pm 6.1$  & $17.3 \pm 5.1$ & $-45.3$ \\
\hline
Exp.~3 (role-chain) & \textsc{Moose}       & $96.1 \pm 7.2$  & $\phantom{0}9.4 \pm 0.9$ & $-86.7$ \\
                    & \quad +BEARS         & $89.2 \pm 7.2$  & $37.1 \pm 6.5$ & $-52.1$ \\
                    & \quad +NeSyDM (RLOO) & $62.7 \pm 2.9$  & $46.5 \pm 3.0$ & $-16.2$ \\
                    & \quad +NeSyDM (exact) & $60.1 \pm 5.7$  & $54.3 \pm 3.0$ & $\phantom{0}-5.8$ \\
\hline
\end{tabular}
\end{table}

Without $\Phi_{\textsf{clos}}$ the base WMC objective collapses toward
chance in every regime, and the collapse deepens with relational
complexity: $\text{Acc}_F$ falls by $26.4$, $57.2$, and $86.7$ points
on Experiments~1--3, reaching $9.4$ (chance is $10.0$) on the
role-chain regime. Experiment~1 is degenerate: all four methods
land on the same seed-independent $21.7\pm0.0$,
because without the complete closure package the digit family
receives no disambiguating gradient. The EL-aware extractors provide
the structural scaffold, but the complete
$\Phi_{\textsf{clos}}$ package supplies the learning signal that makes
the latent digit identifiable; the MNIST gains cannot be attributed to
$\Gamma_{\textsf{EL}}$ alone.

The reasoning-shortcut mitigations can partially substitute for
closure on the relational regimes, but which objective helps, and by
how much, is regime-dependent rather than systematic. NeSyDM~(RLOO)
is the most robust, retaining $39.4$ (Exp.~2) and $46.5$ (Exp.~3)
against plain \textsc{Moose}'s $17.4$ and $9.4$; BEARS buffers
partially in both ($30.9$ and $37.1$); and NeSyDM~(exact) splits the
two regimes, collapsing with the base objective on Experiment~2
($17.3$, $\Delta{=}{-}45.3$) yet proving the most closure-robust of
any method on Experiment~3 ($54.3$, $\Delta{=}{-}5.8$). The wrappers
that resist collapse also hold $\text{RS}_\text{cons}$ down (to
$0.28$--$0.43$, against plain \textsc{Moose}'s $0.43$--$0.50$ where it
does not collapse outright), reshaping the posterior enough to recover
a substantial fraction of the signal, though never to the
with-closure level. No single mitigation is a general substitute for
$\Phi_{\textsf{clos}}$: closure is necessary for the exact base
objective and only partially, inconsistently replaceable by a
reasoning-shortcut objective, which connects the mitigation behaviour
of \S\ref{sec:results} to the
$\Gamma_{\textsf{EL}}$-versus-$\Phi_{\textsf{clos}}$ separation.

\section{Related work: compilation, refinement, and embedding
  approaches}\label{app:related}

\paragraph{Rewriting-based approaches.}
\ELpp\ is Datalog-rewritable, and a substantial line of work recasts
description logic reasoning as rule evaluation: consequence-based
calculi for EL can be expressed as Datalog programs
\cite{Krotzsch2010ELDatalog}, and data-independent transformations
carry richer Horn DLs into Datalog while preserving assertion
entailment \cite{Carral2019Datalog}. \textsc{Moose} shares this view:
Stage~2 of the pipeline (\S\ref{sec:pipeline}) treats the ELK
saturation as a monotone Datalog program, but the two lines answer
different questions. A rewriting produces a reasoning procedure whose
output is an entailment set; it supplies no gradient and no
per-individual marginal, so it cannot act as a supervision channel.
The probabilistic description logics
\cite{Ceylan2014,Ceylan2017,GutierrezBasulto2011} attach a
distribution to the axioms and answer probabilistic queries, but that
distribution is specified by the modeller rather than learned from
perception; the fuzzy line \cite{Zhao2025DFELpp,Bobillo2011}
gains differentiability by replacing model-theoretic entailment with a
t-norm relaxation. \textsc{Moose} compiles the classical semantics
into a circuit that is at once exact and differentiable, which is what
allows neural outputs to serve as atom-level evidence.

\paragraph{Compilation-based approaches.}
Two recent works compile a description logic ontology into a
differentiable representation for use with neural networks.
Lazzari et~al.\ \cite{Lazzari2026} compile an \ALCI{} TBox to a
smooth, decomposable, deterministic circuit using a domino-style
reduction, then plug the circuit into a
multi-label classifier as either a Semantic-Loss regulariser or
the head of a Semantic Probabilistic Layer; their experiments
report on synthetically generated inputs and target the link- and
classification-prediction setting in which every output label of
the multi-label classifier is observed during training. Their
treatment leaves three regimes open: the EL profile with role
chains and role hierarchies, partial supervision in which a
subset of the ground atoms is latent, and an RS analysis of the
resulting predictor. DF-EL$^{++}$ \cite{Zhao2025DFELpp}
approximates \ELpp\ with a product-based fuzzy semantics that
preserves PTIME-tractability and reports knowledge-base completion
results on SNOMED~CT (377K concepts); its loss replaces the
classical model-theoretic entailment with a continuous
relaxation, and it does not address exact $\WMC$,
partial-supervision concept learning, or RS metrics.

\paragraph{Refinement, clustering, and embeddings.}
Iterative Local Refinement (ILR) \cite{Daniele2023ILR} corrects neural
predictions at inference time so they satisfy a fuzzy-logic relaxation
of a propositional formula; Embed2Sym \cite{Aspis2022Embed2Sym} trains
a perception network end-to-end on the downstream label, clusters the
resulting embedding, and labels the clusters via a symbolic solver
over a logic program. ELEmbeddings \cite{Kulmanov2019} maps concepts
to $n$-balls and roles to translation vectors with NF1--NF4 margin
losses that geometrise the \ELpp\ axioms.  DeepGOZero
\cite{Kulmanov2022DeepGOZero} trains a protein-sequence encoder
jointly with these geometric class embeddings, supervising on observed
protein--GO-class annotations via binary cross-entropy while the
NF1--NF4 losses constrain the class $n$-balls through the GO axioms;
predictions for a GO class with no training proteins are read off the
same trained $n$-ball, giving zero-shot annotations that ride on the
ontology axioms alone.  DeepGO-SE \cite{Kulmanov2024DeepGOSE} reframes
the same encoding as approximate semantic entailment.  Our
ELEmbeddings baseline (\Cref{app:non-wmc-baselines}) adopts the
identical NF1--NF4 losses and ball-distance per-atom scoring but
differs in the supervision regime, per-class function annotations
there versus latent ABox atoms here.  OWL2Vec$^*$ \cite{Chen2021}
embeds an ontology by training a word model on random walks over its
axioms; a recent survey \cite{Chen2025Survey} covers the
ontology-embedding landscape more broadly. ILR and Embed2Sym have not
been instantiated for OWL EL with role chains; the embedding methods
target subsumption inference and knowledge-base completion rather than
partial-supervision concept learning. All four trade the
model-theoretic semantics for a continuous representation, whereas
\textsc{Moose} compiles the same ELK semantics into a differentiable
circuit.

\end{document}